\documentclass[10pt]{article}

\usepackage[utf8]{inputenc}  % Input encoding
\usepackage[T1]{fontenc}     % Font encoding
\usepackage[english]{babel}  % Naming of figures and such
\usepackage{csquotes}        % Required for BiBLaTeX
\usepackage{import}          % Importing subdocuments
\usepackage{standalone}      % Compilable subdocuments
\usepackage{hyperref}        % Clickable references
\usepackage{microtype}       % Nice typography
\usepackage{silence}         % Silence warnings and errors
\usepackage[table]{xcolor}   % Colours
\usepackage[
    separate-uncertainty=true,
    per-mode=symbol
    ]{siunitx}               % Display units
\usepackage[
    inline
    ]{enumitem}              % Better enumeration
\usepackage{xifthen}         % If statements
\usepackage{lipsum}          % Dummy text
\usepackage{blindtext}       % More dummy text
\usepackage{listings}        % Listings
\usepackage{xparse}          % Better arguments for \newcommand
\usepackage{xfrac}           % Better fractions
\usepackage{etoolbox}        % Patch stuff
\usepackage[
    notquote
    ]{hanging}               % Hanging paragraphs
\usepackage{scalerel}        % Scale objects
\usepackage{setspace}        % Set spacing
\usepackage{ragged2e}        % Better line endings
\usepackage{footmisc}
\usepackage{multirow}

\usepackage{pgfplots}        % Plots
\pgfplotsset{compat=newest}
\usepackage{tikz}            % TikZ figures
\usepackage{float}           % Control floating of figures
\usepackage{multirow}        % Cells spanning multiple rows
\usepackage{graphicx}        % Graphics
\usepackage[
    font=small
    ]{caption}               % Subfigures and captions
\usepackage{subcaption}      % Subfigures and captions
\usepackage{booktabs}        % Nice looking tables
\usepackage{tabularx}        % Extended functionality for tables

\usepackage{amsmath}         % Math
\usepackage{amssymb}         % Math symbols
\usepackage{mathtools}       % Math tools
\usepackage{amsthm}          % Theorems
\usepackage{thmtools}        % More math theorems
\usepackage{bbm}             % Bold math symbols
\usepackage{cancel}          % Cancel equations

\usepackage[
    capitalize,
    nameinlink]{cleveref}    % Automatic referencing

\renewcommand{\cite}[1]{\PackageError{thesis}{use either parencite or authorcite}{}}
\newcommand{\parencite}[1]{\citep{#1}}
\newcommand{\textcite}[1]{\citet{#1}}

\newcommand{\ifempty}[3]{\ifthenelse{\equal{#1}{}}{#2}{#3}}

\definecolor{darkblue} {rgb} {0.0 , 0.0 , 0.65}
\definecolor{darkred}  {rgb} {0.80, 0.0 , 0.0 }
\definecolor{darkgreen}{rgb} {0.0 , 0.50, 0.0 }
\definecolor{gray75}   {gray}{0.75}
\definecolor{mplorange}{rgb}{1.0, 0.4980392156862745, 0.054901960784313725}
\definecolor{mplblue}  {rgb}{0.12156862745098039, 0.4666666666666667, 0.7058823529411765}

\newcommand{\mplorange}[1]{{\color{mplorange} #1}}
\newcommand{\mplblue}[1]{{\color{mplblue} #1}}

\usetikzlibrary{
    calc,
    positioning,
    fit,
    tikzmark,
    arrows.meta,
    shapes,
    decorations.pathreplacing
}
\tikzset{
    line/.style = {
        thick,
        ->,
        > = {
            Triangle[length=2.0mm, width=2.0mm]
        }
    },
    arrow/.style = {
        line
    },
    hidden node/.style = {
        circle,
        minimum size = 0.9cm,
        draw = white,
        thick
    },
    latent node/.style = {
        hidden node,
        draw = black,
    },
    factor node/.style = {
        hidden node,
        rectangle,
        draw = black,
    },
    observed node/.style = {
        latent node,
        fill = gray!15
    },
    plate/.style = {
        draw,
        label={[anchor=north west]south west:#1},
        rounded corners=2pt,
        shape=rectangle,
        inner sep=10pt,
        thick
    }
}
\setlist[enumerate]{label=(\arabic*),itemsep=-0.25\baselineskip}
\setlist[itemize]{label=\textbullet,itemsep=-0.25\baselineskip}

\newcolumntype{L}{>{\RaggedRight\arraybackslash}X}

\hypersetup{
    colorlinks,
    citecolor = black,
    filecolor = black,
    linkcolor = black,
    urlcolor  = black
}

\tikzset{
    every picture/.append style = {
        xscale = 2,
        yscale = 2
    }
}

\newcommand{\code}[1]{\texttt{#1}}

\newcommand{\eg}{\textit{e.g.}}

\input{math}
\makeatletter
\patchcmd\thmt@mklistcmd
    {\thmt@thmname}
    {\check@optarg{\thmt@thmname}}
    {}{}
\patchcmd\thmt@mklistcmd
    {\thmt@thmname\ifx}
    {\check@optarg{\thmt@thmname}\ifx}
    {}{}
\protected\def\check@optarg#1{%
    \@ifnextchar\thmtformatoptarg\@secondoftwo{#1}%
}
\makeatother

\let\oldlistoftheorems\listoftheorems

\renewcommand{\listoftheorems}{
    \renewcommand{\listtheoremname}{List of Theorems}
    \oldlistoftheorems[ignoreall, show={theorem}]
}

\newlength{\thmtopsep}
\newlength{\thmbotsep}
\newtheoremstyle{theoremstyle}
    {\thmtopsep}{\thmbotsep}
    {}           % Body font
    {}           % Indent amount
    {\bfseries}  % Theorem head font
    {.}          % Punctuation after theorem head
    {.5em}       % Space after theorem head
    {}           % Theorem head spec
\theoremstyle{theoremstyle}

\newcommand{\period}{$\text{.}$}  % Surpress spacing that comes afterwards.
\newtheorem{corollary}{Cor\period}%[section]
\newtheorem{lemma}{Lem\period}%[section]
\newtheorem{model}{Mod\period}%[section]
\newtheorem{proposition}{Prop\period}%[section]
\newtheorem{remark}{Rem\period}%[section]
\crefname{assumption}{Assumption}{Assumptions}
\Crefname{assumption}{Assumption}{Assumptions}
\crefname{corollary}{Corollary}{Corollaries}
\Crefname{corollary}{Corollary}{Corollaries}
\crefname{definition}{Definition}{Definitions}
\Crefname{definition}{Definition}{Definitions}
\crefname{example}{Example}{Examples}
\Crefname{example}{Example}{Examples}
\crefname{fact}{Fact}{Facts}
\Crefname{fact}{Fact}{Facts}
\crefname{lemma}{Lemma}{Lemmas}
\Crefname{lemma}{Lemma}{Lemmas}
\crefname{model}{Model}{Models}
\Crefname{model}{Model}{Models}
\crefname{proposition}{Proposition}{Propositions}
\Crefname{proposition}{Proposition}{Propositions}
\crefname{question}{Question}{Questions}
\Crefname{question}{Question}{Questions}
\crefname{remark}{Remark}{Remarks}
\Crefname{remark}{Remark}{Remarks}
\crefname{theorem}{Theorem}{Theorems}
\Crefname{theorem}{Theorem}{Theorems}

\newlist{asslist}{enumerate}{1}
\setlist[asslist]{
    ref=\theassumption.(\arabic*),
    label=(\arabic*),
    itemsep=-0.25\baselineskip
}
\crefname{asslisti}{Assumption}{Assumptions}
\Crefname{asslisti}{Assumption}{Assumptions}
\newlist{corlist}{enumerate}{1}
\setlist[corlist]{
    ref=\thecorollary.(\arabic*),
    label=(\arabic*),
    itemsep=-0.25\baselineskip
}
\crefname{corlisti}{Corollary}{Corollaries}
\Crefname{corlisti}{Corollary}{Corollaries}
\newlist{deflist}{enumerate}{1}
\setlist[deflist]{
    ref=\thedefinition.(\arabic*),
    label=(\arabic*),
    itemsep=-0.25\baselineskip
}
\crefname{deflisti}{Definition}{Definitions}
\Crefname{deflisti}{Definition}{Definitions}
\newlist{exlist}{enumerate}{1}
\setlist[exlist]{
    ref=\theexample.(\arabic*),
    label=(\arabic*),
    itemsep=-0.25\baselineskip
}
\crefname{exlisti}{Example}{Examples}
\Crefname{exlisti}{Example}{Examples}
\newlist{factlist}{enumerate}{1}
\setlist[factlist]{
    ref=\thefact.(\arabic*),
    label=(\arabic*),
    itemsep=-0.25\baselineskip
}
\crefname{factlisti}{Fact}{Facts}
\Crefname{factlisti}{Fact}{Facts}
\newlist{lemlist}{enumerate}{1}
\setlist[lemlist]{
    ref=\thelemma.(\arabic*),
    label=(\arabic*),
    itemsep=-0.25\baselineskip
}
\crefname{lemlisti}{Lemma}{Lemmas}
\Crefname{lemlisti}{Lemma}{Lemmas}
\newlist{modlist}{enumerate}{1}
\setlist[modlist]{
    ref=\themodel.(\arabic*),
    label=(\arabic*),
    itemsep=-0.25\baselineskip
}
\crefname{modlisti}{Model}{Models}
\Crefname{modlisti}{Model}{Models}
\newlist{proplist}{enumerate}{1}
\setlist[proplist]{
    ref=\theproposition.(\arabic*),
    label=(\arabic*),
    itemsep=-0.25\baselineskip
}
\crefname{proplisti}{Proposition}{Propositions}
\Crefname{proplisti}{Proposition}{Propositions}
\newlist{qlist}{enumerate}{1}
\setlist[qlist]{
    ref=\theremark.(\arabic*),
    label=(\arabic*),
    itemsep=-0.25\baselineskip
}
\crefname{qlisti}{Question}{Questions}
\Crefname{qlisti}{Question}{Questions}
\newlist{remlist}{enumerate}{1}
\setlist[remlist]{
    ref=\theremark.(\arabic*),
    label=(\arabic*),
    itemsep=-0.25\baselineskip
}
\crefname{remlisti}{Remark}{Remarks}
\Crefname{remlisti}{Remark}{Remarks}
\newlist{thmlist}{enumerate}{1}
\setlist[thmlist]{
    ref=\thetheorem.(\arabic*),
    label=(\arabic*),
    itemsep=-0.25\baselineskip
}
\crefname{thmlisti}{Theorem}{Theorems}
\Crefname{thmlisti}{Theorem}{Theorems}

\usepackage[accepted]{icml2020}

\renewcommand{\paragraph}[1]{\textbf{#1}}
\newcommand{\Matern}{Mat\'ern}

\crefname{section}{Sec\period}{Sects\period}
\crefname{proposition}{Prop\period}{Props\period}
\crefname{lemma}{Lem\period}{Lems\period}
\crefname{remark}{Rem\period}{Rems\period}
\crefname{model}{Mod\period}{Mods\period}
\crefname{appendix}{App\period}{Apps\period}
\crefname{table}{Tab\period}{Tabs\period}

\newlength\figureheight
\newlength\figurewidth

\begin{document}
\twocolumn[
    \icmltitle{
        Scalable Exact Inference in Multi-Output Gaussian Processes
    }
    
    \begin{icmlauthorlist}
    \icmlauthor{Wessel P.\ Bruinsma}{cam,invenia}
    \icmlauthor{Eric Perim}{invenia}
    \icmlauthor{Will Tebbutt}{cam}
    \icmlauthor{J.\ Scott Hosking}{bas,ati}
    \icmlauthor{Arno Solin}{aalto}
    \icmlauthor{Richard E.\ Turner}{cam,msr}
    \end{icmlauthorlist}
    
    \icmlaffiliation{cam}{University of Cambridge}
    \icmlaffiliation{invenia}{Invenia Labs}
    \icmlaffiliation{bas}{British Antarctic Survey}
    \icmlaffiliation{ati}{Alan Turing Institute}
    \icmlaffiliation{aalto}{Aalto University}
    \icmlaffiliation{msr}{Microsoft Research}
    
    \icmlcorrespondingauthor{Wessel P.\ Bruinsma}{wpb23@cam.ac.uk}
    
    \icmlkeywords{Machine Learning, ICML}
    
    \vskip 0.3in
]

\printAffiliationsAndNotice{}

\begin{abstract}
    Multi-output Gaussian processes (MOGPs) leverage the flexibility and interpretability of GPs while capturing structure across outputs, which is desirable, for example, in spatio-temporal modelling. The key problem with MOGPs is their computational scaling $\O(n^3 p^3)$, which is cubic in the number of both inputs $n$ (\eg, time points or locations) and outputs $p$. For this reason, a popular class of MOGPs assumes that the data live around a low-dimensional linear subspace, reducing the complexity to $\O(n^3m^3)$. However, this cost is still cubic in the dimensionality of the subspace $m$, which is still prohibitively expensive for many applications.
    We propose the use of a sufficient statistic of the data to accelerate inference and learning in MOGPs with orthogonal bases.
    The method achieves \emph{linear} scaling in $m$ in practice, allowing these models to scale to large $m$ without sacrificing significant expressivity or requiring approximation.
    This advance opens up a wide range of real-world tasks and can be combined with existing GP approximations in a plug-and-play way.
    We demonstrate the efficacy of the method on various synthetic and real-world data sets.
\end{abstract}

\section{Introduction}
Gaussian processes \citep[GPs,][]{Rasmussen:2006:Gaussian_Processes} form an interpretable, modular, and tractable probabilistic framework for modelling nonlinear functions.
They are successfully applied in a wide variety of single-output problems:
they can automatically discover structure in signals \parencite{Duvenaud:2014:Automatic_Construction}, achieve state-of-the-art performance in regression tasks \parencite{Bui:2016:Deep_Gaussian_Processes_for_Regression}, enable data-efficient models in reinforcement learning \parencite{Deisenroth:2011:PILCO_A_Model-Based_and_Data-Efficient}, and support many applications in probabilistic numerics \parencite{Hennig:2015:Probabilistic_Numerics_and_Uncertainty_in}, such as in optimisation \parencite{Brochu:2010:A_Tutorial_on_Bayesian_Optimization} and quadrature \parencite{Minka:2000:Quadrature_GP}.

Multi-output Gaussian processes (MOGPs) leverage the flexibility and interpretability of GPs while capturing structure across outputs.
One of the first applications of GPs with multiple outputs was in geostatistics \parencite{Matheron:1969:Le_Krigeage_Universel}.
Today, MOGPs models can be found in various areas, including geostatistics \parencite{Wackernagel:2003:Multivariate_Geostatistics}, factor analysis \parencite{Teh:2005:Semiparametric_Latent_Factor,Yu:2009:Gaussian-Process_Factor_Analysis_for_Low-Dimensional}, dependent or multi-task learning \parencite{Boyle:2005:Dependent_Gaussian_Processes,Bonilla:2007:Kernel_Multi-Task_Learning_Using_Task-Specific,Bonilla:2008:Multi-Task_Gaussian_Process,Osborne:2008:Towards_Real-Time_Information_Processing_of}, latent force models \parencite{Alvarez:2009:Latent_Force_Models,Alvarez:2009:Sparse_Convolved_Gaussian_Processes_for,Alvarez:2010:Efficient_Multioutput_Gaussian_Processes_Through,Alvarez:2011:Computationally_Efficient_Convolved}, state space modelling \parencite{Sarkka:2013:Spatiotemporal_Learning_via}, regression networks \parencite{Wilson:2012:GP_Regression_Networks,Nguyen:2014:Collaborative_Multi-Output, Dezfouli:2017:Semi-Parametric_Network_Structure_Discovery_Models}, and mixture models \parencite{Ulrich:2015:Cross_Spectrum,Bruinsma:2016:GGPCM,Parra:2017:Spectral_Mixture_Kernels_for_Multi-Output,Requeima:2019:The_Gaussian_Process_Autoregressive_Regression}.

A key practical problem with existing MOGPs is their computational complexity.
For $n$ input points, each having $p$ outputs, inference and learning in general MOGPs take $\O(n^3p^3)$ time and $\O(n^2p^2)$ memory, although these may be alleviated by a wide range of approximations \parencite{Quinonero:2005:Unifying_View,Titsias:2009:Variational_Learning,Lazaro-Gredilla:2010:Sparse_Spectrum_Gaussian_Process_Regression,Hensman:2013:Gaussian_Processes_for_Big_Data,Wilson:2015:Kernel_Interpolation_for_Scalable_Structured,Bui:2016:A_Unifying_Framework_for_Gaussian,Cheng:2017:Variational_Inference_for_Gaussian_Process,Hensman:2018:Variational_Fourier_Features_for_Gaussian}.
To mitigate these unfavourable scalings, a particular class of MOGPs, which we call the Instantaneous Linear Mixing Model (ILMM, \cref{sec:ILMM}), assumes that the data live around an $m$-dimensional linear subspace, where $m < p$.
This class exploits the low-rank structure of its covariance to reduce the complexity of inference and learning to roughly $\O(n^3 m^3)$ time and $\O(n^2 m^2)$ memory.
Although $m$ is typically much smaller than $p$, the runtime complexity is again cubic in $m$.
Consequently, the ILMM is prohibitively expensive in applications where moderate $m$ is required.
Consider, for example, hourly real-time electricity prices at 2313 different locations across 15 U.S.\ states and the Canadian province of Manitoba during the year 2019 \parencite{MISO}.
Forecasting electricity prices is crucial in the planning of energy transmission, which happens 24 hours in advance.
The ILMM is particularly well suited to this task:
the prices derive from optimal power flow, which tends to exhibit low-rank structure.
However, it still takes roughly $m=40$ to explain $95\%$ of the variance of this data.
For $n=$ \SI{1}{k} time points, this requires the inversion of a \SI{40}{k} $\times$ \SI{40}{k} matrix.
Even worse, to explain $99\%$ of the variance, it requires the inversion of a \SI{120}{k} $\times$ \SI{120}{k} matrix, clearly far beyond what is feasible.

In this paper, we develop a new perspective on MOGPs in the Instantaneous Linear Mixing Model class through the use of a sufficient statistic of the data.
We use this sufficient statistic to identify a class of MOGPs, which we call the Orthogonal Instantaneous Linear Mixing Model (OILMM, \cref{sec:OILMM}), in which inference and learning take $\O(n^3m + nmp + m^2p)$ time and $\O(n^2m + np + mp)$ memory, without sacrificing significant expressivity nor requiring any approximations.
The dominant (first) terms in these expressions are \emph{linear} in $m$, rather than cubic.
It is this feature that allows the OILMM to scale to large $m$.
The linear scaling is achieved by breaking down the high-dimensional multi-output problem into independent single-output problems, whilst retaining exact inference.
Consequently, the proposed methodology is interpretable---\eg, it can be seen as a natural generalisation of probabilistic principal component analysis \citep[PPCA,][]{Tipping:1999:Probabilistic_Principal_Component_Analysis}---simple to implement, and trivially compatible with single-output scaling techniques in a plug-and-play way.
For example, it can be combined with the variational inducing point approximation by \citep{Titsias:2009:Variational_Learning} or with state-space approximation methods (\cref{sec:spatio-temporal}); these approximations reduce the time complexity of the dominant term to linear in \emph{both} the number of data points $n$ and $m$.
We demonstrate the efficacy of the OILMM in experiments on various synthetic and real-world data sets.
Simple algorithms to perform inference and learning in the OILMM are presented in \cref{app:implementation}.

\section{Multi-Output Gaussian Process Models}
\label{sec:MOGPs}
For tasks with $p$ outputs, multi-output Gaussian processes induce a prior distribution over \emph{vector-valued} functions $f\colon\Tc\to\R^p$ by requiring that any finite collection of function values $f_{p_1}(t_1), \ldots, f_{p_n}(t_n)$ with $(p_i)_{i=1}^n \sub \set{1, \ldots, p}$ are multivariate Gaussian distributed.
We consider the input space time, where $\Tc=\R$, but the analysis trivially applies to more general feature spaces, \eg\ $\Tc=\R^d$.
A MOGP $f \sim \GP(m, K)$ is described by a \emph{vector-valued} mean function
$
    m(t) = \E[f(t)]
$
and a \emph{matrix-valued} covariance function
$
     K(t, t') = \E[f(t)f^\T(t')] - \E[f(t)]\E[f^\T(t')].
$
For $n$ observations $y(t_1),\ldots,y(t_n) \in \R^p$, inference and learning take $\O(n^3p^3)$ time and $\O(n^2p^2)$ memory.

\subsection{The Instantaneous Linear Mixing Model}
\label{sec:ILMM}

A simple, but general class of MOGPs decomposes a signal $f(t)$ comprising $p$ outputs into a fixed basis ${h_1, \ldots, h_m \in \R^p}$ with coefficients $x_1(t), \ldots, x_m(t) \in \R$:
\[
    f(t)
    = h_1 x_1(t) + \ldots + h_m x_m(t)
    = Hx(t)
\]
where $h_i$ is the $i$\textsuperscript{th} column of $H$.
The coefficients $x_1(t)$, \ldots,\newpage $x_m(t)$ are time varying and modelled independently with unit-variance Gaussian processes.
The noisy signal $y(t)$ is then generated by adding $\Normal(0, \Sigma)$-distributed noise to $f(t)$.
Intuitively, this means that the $p$-dimensional data live in a ``pancake'' \parencite{Roweis:1999:A_Unifying_Review_of_Linear,MacKay:2002:Information_Theory_Learning} around the $m$-dimensional column space of $H$, where typically $m \ll p$.

\begin{model}[Instantaneous Linear Mixing Model] \label{mod:ILMM}
    Let $K$ be an $m \times m$ diagonal multi-output kernel with $K(t, t)=I_m$, $H$ a $p \times m$ matrix, and $\Sigma$ a $p \times p$ observation noise covariance.
    Then the ILMM is given by the following generative model:
    \[\begin{aligned}
        x &\sim \GP(0, K(t, t')), &\text{(latent processes)} \\
        f(t) \cond H, x(t) &= Hx(t), &\text{\hspace{-1em}(mixing mechanism)} \\
        y \cond f &\sim \GP(f(t), \delta[t-t']\Sigma). &\text{(noise model)}
    \end{aligned}\]
\end{model}

We call $x$ the \textit{latent processes} and $H$ the \textit{mixing matrix} or \textit{basis}.
Throughout the paper, we assume that $H$ has linearly independent columns.
If we marginalise out $f$ and $x$, we find that
$y \sim \GP(0, HK(t,t')H^\T + \delta[t-t']\Sigma)$,
which reveals that the ILMM exhibits low-rank covariance structure.
It also shows that the ILMM is a time-varying generalisation of factor analysis (FA): choosing $K(t,t')\!=\!\delta[t\!-\!t']I_m$ and $\Sigma$ diagonal recovers FA exactly.

The ILMM is definitely not novel;
the specific formulation in \cref{mod:ILMM} is for convenience of the exposition in this paper.
In particular, the ILMM is very similar to the Linear Model of Coregionalisation (LMC) \parencite{Goovaerts:1997:Geostatistics_for_Natural_Resources_Evaluation}.
In the LMC, every latent process has multiple independent copies and the observation noise $\Sigma$ is typically diagonal.
More generally, the ILMM is a special case of the more general formulation with mixing mechanism
$
    f(t) = \int \tilde H(t, \tau) x(\tau) \isd \tau
$
where $\tilde H\colon \Tc\times\Tc\to\R^{p\times m}$ is a matrix-valued time-varying filter.
In particular, it is the case $\tilde H(t, \tau) = \delta(t-\tau)H$;
here the mixing is \emph{instantaneous} and \emph{time-invariant}.
Many other MOGPs in the machine learning and geostatistics literature can be seen as specialisations of this more general formulation by imposing structure on $\tilde H$ and $K$.
An organisation of the literature from this point of view, which we call the Mixing Model Hierarchy (MMH), is presented in \cref{app:mmh}.

\subsection{Inference and Learning in the ILMM}
The complexities of inference and learning in MOGPs can often be alleviated by exploiting structure in the kernel.
This is the case for the ILMM, which we have seen exhibits low-rank covariance structure.
In this section, we develop a new perspective on the ILMM by showing that the covariance structure can be exploited by devising a low-dimensional ``summary'' or ``projection'' of the $p$-dimensional observations.
This reduces the complexities from $\O(n^3 p^3)$ time and $\O(n^2 p^2)$ memory to $\O(n^3 m^3 + nmp + m^2 p)$ time and $\O(n^2 m^2 + np + mp)$ memory, where $\O(nmp + m^2 p)$ is the cost of projecting the data and computing the projection, and $\O(np + mp)$ is the cost of storing the data and projection.

The low-dimensional projection of the observations $y$ will be given by a sufficient statistic for the model, which is therefore ``without loss of information'' and can be used to accelerate inference.
Concretely, the projection of $y$ is given by the maximum likelihood estimate (MLE) of $x$ under the likelihood $p(y \cond x)$ of the ILMM.
As \cref{prop:mle} in \cref{app:mle} shows, this MLE is given by $Ty$ where $T$ is the $m\times p$ matrix $(H^\T \Sigma^{-1} H)^{-1} H^\T \Sigma^{-1}$;
$Ty$ is an unbiased estimator of $x$.
Because $Ty$ is an MLE for $x$, it is a function of a sufficient statistic for $x$, if one exists.
\cref{prop:sufficiency} in \cref{app:sufficiency} shows that $T y$ is actually minimally sufficient itself.
For any prior $p(x)$ over $x$, sufficiency of $Ty$ gives that $p(x \cond y) = p(x \cond Ty)$;
that is, conditioning on $y$ is equivalent to conditioning on $Ty$, where $Ty$ can be interpreted as a ``summary'' or ``projection'' of $y$.
This idea is formalised in the following proposition, which is proven in \cref{app:proof}:

\begin{proposition} \label{prop:general_sufficiency}
    Let $p(x)$ be a model for $x\colon\Tc\to\R^m$, not necessarily Gaussian, $H$ a $p \times m$ matrix, and $\Sigma$ a $p \times p$ observation noise covariance.
    Then consider the following generative model:
    \vspace{-1em}\begin{align*}
        x &\sim p(x), &\text{(latent processes)} \\
        f(t) \cond H, x(t) &= Hx(t), &\text{\hspace{-1em}(mixing mechanism)} \\
        y \cond f &\sim \GP(f(t), \delta[t-t']\Sigma). &\text{(noise model)}
    \end{align*}
    Consider a $p \times n$ matrix $Y$ of observations of $y$.
    Then $p(f\cond Y) = p(f \cond TY)$, where the distribution of the projected observed signal $Ty$ is
    \[
        Ty \cond x \sim \GP(x(t), \delta[t-t']\Sigma_T) \,\text{ with }\, \Sigma_T = (H^\T \Sigma^{-1} H)^{-1}.
    \]
    Moreover, the probability of the data $Y$ is given by
    \[
        p(Y) \!=\! \sbrac*{
            \prod_{i=1}^n \!\frac{\Normal(y_i\cond 0, \Sigma)}{\Normal(T y_i\cond 0, \Sigma_T)}\!
        }
        \!\!\int\! p(x) \!\prod_{i=1}^n\! \Normal(T y_i\cond x_i,\Sigma_T) \isd x
    \]
    where the $i$\textsuperscript{th} observation $y_i$ is the $i$\textsuperscript{th} column of $Y$.
\end{proposition}

Crucially, $Y$ are $p$-dimensional observations, $TY$ are $m$-dimensional summaries, and typically $m \ll p$, so conditioning on $T Y$ is often much cheaper;
note that computing $TY$ takes $\O(nmp)$ time and $\O(mp)$ memory.
In particular, if we apply \cref{prop:general_sufficiency} to the ILMM by letting $x \sim \GP(0, K(t, t'))$, we immediately get the claimed reduction in complexities:
whereas conditioning on $Y$ takes $\O(n^3 p^3)$ time and $\O(n^2 p^2)$ memory, we may equivalently condition on $TY$, which takes $\O(n^3 m^3)$ time and $\O(n^2 m^2)$ memory instead.
This important observation is depicted in \cref{fig:commutative_diagram_ILMM}.

The case of \cref{prop:general_sufficiency} where $x$ is Gaussian can be found as Results 1 and 2 by \citet{Higdon:2008:Computer_Model_Calibration_Using_High-Dimensional}, and was also used by the authors to accelerate inference.
Although the reduction in computational complexities allows \citeauthor{Higdon:2008:Computer_Model_Calibration_Using_High-Dimensional} to scale to significantly larger data, they are still limited by the cubic dependency on $m$.

If the observations can be naturally represented as multi-index arrays in $\R^{p_1\times\cdots \times p_q}$, a natural choice is to correspondingly decompose $H = H_1 \otimes \cdots \otimes H_q$ where $\otimes$ is the Kronecker product.
In this parametrisation, the projection and projected noise also become the Kronecker products:
$T = T_1 \otimes \cdots \otimes T_q$ and $\Sigma_T = \Sigma_{T_1} \otimes \cdots \otimes \Sigma_{T_q}$.
See \cref{app:tensor}.
The model by \citet{Zhe:2019:Scalable_High-Order_Gaussian_Process_Regression} can be seen as an ILMM of this form with $K(t,t') = k(t,t')I_m$ where $k$ is a scalar-valued kernel and $\Sigma = \sigma^2 I_p$.

In \cref{prop:general_sufficiency}, we call
$\Sigma_T = T \Sigma T^\T = (H^\T \Sigma^{-1} H)^{-1}$
the \emph{projected observation noise}.
The projected noise $\Sigma_T$ is important, because it couples the latent processes upon observing data.
In particular, if the latent processes are independent under the prior and $\Sigma_T$ is diagonal, then the latent processes remain independent when data is observed.
This observation forms the basis of the computational gains achieved by the Orthogonal Instantaneous Linear Mixing Model.

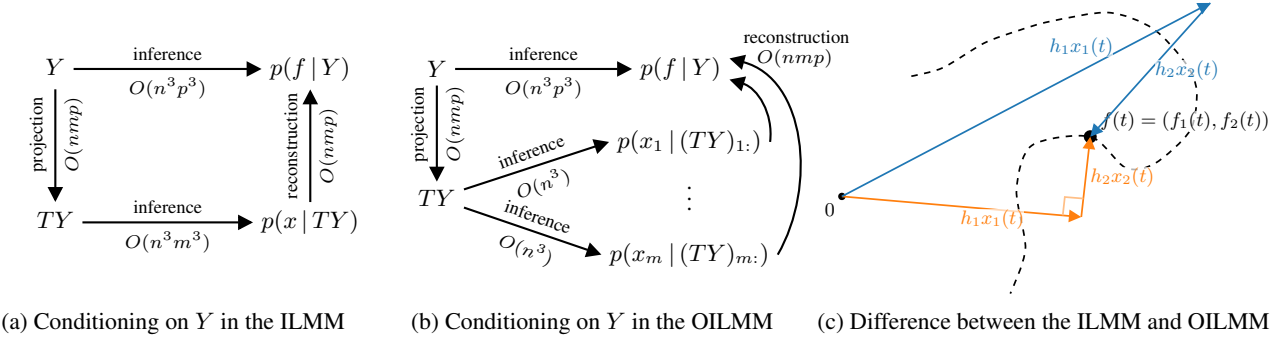
\begin{figure*}[t]
  \centering\footnotesize
  \begin{subfigure}[t]{.30\textwidth}
    \centering
    \begin{tikzpicture}
        \node [anchor=center] (Y) {$Y\vphantom{p(}$};
        \node [below=1.5cm of Y] (TY) {$TY\vphantom{p(}$};
        \node [right=2.5cm of Y] (pyY) {$p(f\cond Y)$};
        \node [below=1.5cm of pyY] (pTY) {$p(x\cond TY)$};
        \draw [arrow]
            (Y) -- 
                node [pos=0.4, anchor=south, rotate=90] {\scriptsize projection}
                node [pos=0.4, anchor=north, rotate=90] {\scriptsize $\O(nmp)$}
            (TY);
        \draw [arrow]
            (Y) --
                node [pos=0.5, anchor=south] {\scriptsize inference}
                node [pos=0.5, anchor=north] {\scriptsize $\O(n^3 p^3)$}
            (pyY);
        \draw [arrow]
            (TY) --
                node [pos=0.5, anchor=south] {\scriptsize inference}
                node [pos=0.5, anchor=north] {\scriptsize $\O(n^3 m^3)$}
            (pTY);
        \draw [arrow]
            (pTY) -- 
                node [pos=0.5, anchor=south, rotate=90] {\scriptsize reconstruction\vphantom{j}}
                node [pos=0.5, anchor=north, rotate=90] {\scriptsize $\O(nmp)$}
            (pyY);
        
        % Specify bounding box to have the pictures align nicely.
        \useasboundingbox (-.5cm,.5cm) rectangle (2.25cm,-1.5cm);
        
    \end{tikzpicture}
    \caption{Conditioning on $Y$ in the ILMM}
    \label{fig:commutative_diagram_ILMM}
  \end{subfigure}
  \hspace*{\fill}
  \begin{subfigure}[t]{.30\textwidth}
    \centering
    \begin{tikzpicture}
        \node [anchor=center] (Y)  {$Y$};
        \node [below=1.5cm of Y, anchor=center] (TY) {$TY$};
        \node [right=3cm of Y, anchor=center] (pyY) {$p(f\cond Y)$};
        \node [right=3cm of TY, yshift=.75cm, anchor=center] (pu1TY) {$p(x_1\cond (TY)_{1:})$};
        \node [right=3cm of TY, yshift=-.75cm, anchor=center] (pumTY) {$p(x_m\cond (TY)_{m:})$};
        \draw [arrow]
            (Y) -- 
                node [pos=0.4, anchor=south, rotate=90] {\scriptsize projection}
                node [pos=0.4, anchor=north, rotate=90] {\scriptsize $\O(nmp)$}
            (TY);
        \draw [arrow]
            (Y) --
                node [pos=0.5, anchor=south] {\scriptsize inference}
                node [pos=0.5, anchor=north] {\scriptsize $\O(n^3 p^3)$}
            (pyY);
        \draw [arrow]
            (TY) --
                node [pos=0.5, anchor=south, rotate=20] {\scriptsize inference}
                node [pos=0.5, anchor=north, rotate=20] {\scriptsize $\O(n^3)$}
            (pu1TY.west);
        \node [right=3cm of TY, anchor=center, yshift=.1cm] {$\vdots$};
        \draw [arrow]
            (TY) --
                node [pos=0.5, anchor=south, rotate=-20] {\scriptsize inference}
                node [pos=0.5, anchor=north, rotate=-20] {\scriptsize $\O(n^3)$}
            (pumTY.west);
        \draw [arrow, ->] (pu1TY.east) to[bend right=40] ([yshift=-2pt]pyY.east);
        \draw [arrow, ->]
            (pumTY.east)
                to[out=60, in=-15]
                    node [pos=1, anchor=south west, xshift=2pt, yshift=1pt] {\scriptsize reconstruction\vphantom{j}}
                    node [pos=1, anchor=south west, xshift=7pt, yshift=-7pt] {\scriptsize $\O(nmp)$}
            ([yshift=2pt]pyY.east);
        
        % Specify bounding box to have the pictures align nicely.
        \useasboundingbox (-.25cm,.5cm) rectangle (2.5cm,-1.5cm);

    \end{tikzpicture}
    \caption{Conditioning on $Y$ in the OILMM}
    \label{fig:commutative_diagram_OILMM}
  \end{subfigure} 
  \hspace*{\fill}
  \begin{subfigure}[t]{.35\textwidth}
    \centering
    \resizebox{\textwidth}{!}{%
    \begin{tikzpicture}[
        scale=1.5,
        arrowlabel/.style={
            fill=white,
            opacity=0.7,
            text opacity=1,
            inner sep=0
        },
        rotate=-30
    ]
        % \draw [thick]
        %     (2, 0)
        %     -- node [pos=.5, anchor=north] {$f_1(t)$} (0, 0)
        %     -- node [pos=.5, anchor=east] {$f_2(t)$} (0, 2);
        \fill [black]
            (0, 0) circle (0.02);
        \node at (0, 0) [anchor=north east] {$0$};
        \draw [dashed, thick] plot [smooth, tension=0.5] coordinates
            {
                (0, 0.75)
                (0.125, 0.875)
                (0.25, 1.1)
                (0.375, 1.25)
                (0.5, 1.38)
                (0.75, 1.58)
                (1, 1.625)
                (1.25, 1.55)
                (1.375, 1.45)
                (1.45, 1.25)
                (1.375, 1)
                (1.25, 0.9)
                (1, 0.95)
                (0.75, 0.75)
                (0.8, 0.5)
                (1, 0.25)
                (1.05, 0)
            };
        \fill [black]
            (1.0, 0.95)
            circle (0.035)
            node [arrowlabel,
                  anchor=south west,
                  yshift=0.3em,
                  xshift=0.5em] {$f(t)=(f_1(t), f_2(t))$};
        \draw [mplorange, arrow]
            (0, 0)
            -- node [midway,
                     anchor=north west,
                     arrowlabel,
                     yshift=-0.1em,
                     xshift=-0.1em] {$h_1 x_1(t)$} (1.175, 0.55);
        \draw [mplorange, arrow]
            (1.175, 0.55)
            -- node [pos=.5,
                     anchor=west,
                     arrowlabel,
                     xshift=0.2em] {$h_2 x_2(t)$} (1.0, 0.95);
        \draw [thick, mplorange!50]
            (1.084, 0.508)
            -- (1.044, 0.599)
            -- (1.135, 0.642);
        \draw [mplblue, arrow]
            (0, 0)
            -- node [pos=0.75,
                     anchor=east,
                     arrowlabel,
                     xshift=-0.2em,
                     yshift=0.2em] {$h_1 x_1(t)$} (1.2, 1.9);
        \draw [mplblue, arrow]
            (1.2, 1.9)
            -- node [pos=.5,
                     anchor=west,
                     arrowlabel,
                     xshift=0.2em] {$h_2 x_2(t)$} (1.0, 0.95);
    \end{tikzpicture}}
    \caption{
      Difference between the ILMM and OILMM
    }
    \label{fig:particle}      
    \end{subfigure}       
    \caption{(a--b)~Commutative diagrams depicting that conditioning on $Y$ in the ILMM and OILMM is equivalent to conditioning respectively on $TY$ and independently every $x_i$ on $(TY)_{i:}$, but yield different computational complexities. The reconstruction costs assume computation of the marginals. (c)~Illustration of the difference between the ILMM and OILMM. The trajectory of a particle (dashed line) in two dimensions is modelled by the ILMM (\mplblue{blue}) and OILMM (\mplorange{orange}). The noise-free position $f(t)$ is modelled as a linear combination of basis vectors $h_1$ and $h_2$ with coefficients $x_1(t)$ and $x_2(t)$ (two independent GPs). In the OILMM, the basis vectors $h_1$ and $h_2$ are constrained to be orthogonal; in the ILMM, $h_1$ and $h_2$ are unconstrained.}
    \vspace*{-0.5em}
\end{figure*}

\subsection{Interpretation of the Likelihood}
\cref{prop:general_sufficiency} shows that the log-probability of the data $Y$ is equal to the log-probability of the projected data $TY$ plus, for every observation $y_i$, a correction term of the form
$
    \log \Normal(y_i\cond 0, \Sigma)/\Normal(T y_i\cond 0, \Sigma_T).
$
\cref{prop:regularisation-term} in \cref{app:interpretation-likelihood} shows that this correction term can be written as
\begin{equation*}
    -\tfrac12 (p - m)\log 2\pi
    - \underbracket{
        \tfrac12\log |\Sigma|/|\Sigma_T|
    }_{\mathclap{
        \text{noise ``lost by projection''}
    }}
    - \underbracket{
        \tfrac12\norm{y_i - HTy_i}_{\Sigma}^2,
    }_{\mathclap{
        \text{data ``lost by projection''}
    }}
\end{equation*}
where $\norm{\vardot}_{\Sigma} = \norm{\Sigma^{-\frac12}\vardot}$.
When the likelihood is optimised with respect to $H$, the correction terms will prevent the projection $T$ from discarding a component of the data $Y$ and the noise $\Sigma$ that is ``too large''.
For example, for the ILMM, if these correction terms were ignored, then after optimising we would find that $TY = 0$ and $\Sigma_T = 0$, because the density of a zero-mean Gaussian is highest at the origin, and becomes higher as the variance becomes smaller;
it is exactly $TY = 0$ and $\Sigma_T = 0$ that the two penalties prevent from happening.

\section{The Orthogonal Instantaneous Linear Mixing Model}
\label{sec:OILMM}
Inspired by \cref{prop:general_sufficiency}, we will now identify a subclass of the ILMM for which, in practice, inference and learning scale \emph{linearly} in the number of latent processes $m$ rather than cubically.
As we will see, this happens when the projected observation noise is diagonal, which is the case for the
Orthogonal Instantaneous Linear Mixing Model (OILMM):
the subclass of ILMMs where the basis $H$ is \emph{orthogonal}.
In particular, $H = U S^{\frac12}$ where $U$ is a matrix with orthonormal columns and $S > 0$ a diagonal. We define this model as follows:

\begin{model}[Orthogonal Instantaneous Linear Mixing Model] \label{mod:OILMM}
    The OILMM is an ILMM (\cref{mod:ILMM}) where the basis $H$ is a $p \times m$ matrix of the form $H=US^{\frac{1}{2}}$ with $U$ a matrix with orthonormal columns and $S > 0$ diagonal, and $\Sigma = \sigma^2 I_p + H D H^\T$ a $p \times p$ matrix with $D \geq 0$ diagonal.
\end{model}
\vspace*{5pt}

The difference between the ILMM and the OILMM is illustrated in \cref{fig:particle}.
In the OILMM, we require that $m \le p$, since the number of $p$-dimensional vectors that can be mutually orthogonal is at most $p$.
Also, $D$ in $\Sigma$ can be interpreted as heterogeneous noise deriving from the latent processes.
Moreover, although $H$ and $\Sigma$ do not depend on time, our analysis and results trivially carry over to the case where $H_t$ and $\Sigma_t$ do vary with time.
Finally, for the OILMM, \cref{prop:form} in \cref{app:oilmm-proj-and-noise} shows that
$
    T = S^{-\frac12}U^\T
$ and $
    \Sigma_T = \sigma^{2}S^{-1} + D.
$

Whereas the ILMM is a time-varying generalisation of FA, the OILMM can be seen as a time-varying generalisation of probabilistic principal component analysis \citep[PPCA,][]{Tipping:1999:Probabilistic_Principal_Component_Analysis}: $D=0$ and $K(t,t') = \delta[t-t'] I_m$ recovers the orthogonal solution of PPCA exactly;
recall that PPCA admits infinitely many solutions, with only one corresponding to orthogonal axes, whereas the modelling assumptions of the OILMM recover this solution automatically.
See \cref{fig:relationships} for a visualisation of the relationship between FA, PPCA, the ILMM, and the OILMM.
The OILMM is also related to Gaussian Process Factor Analysis \citep[GPFA,][]{Yu:2009:Gaussian-Process_Factor_Analysis_for_Low-Dimensional}, with the crucial difference being that in GPFA orthogonalisation of the columns of $H$ is done as a post-processing step, whereas in the OILMM orthogonality of the columns of $H$ is built into the model.
In this respect, the OILMM is more similar to the model by \citet{Higdon:2008:Computer_Model_Calibration_Using_High-Dimensional}, who also consider a MOGP with an orthogonal basis built in.

\begin{figure}
    \centering
    \begin{tikzpicture}
        \node [anchor=center] (FA) {FA};
        \node [above=1cm of FA] (ILMM) {ILMM};
        \node [right=3cm of FA] (PPCA) {PPCA};
        \node [above=1cm of PPCA] (OILMM) {OILMM};
        \draw [arrow]
            (FA) -- 
                node [pos=0.525, anchor=south] {\scriptsize orthogonality constraint}
            (PPCA);
        \draw [arrow]
            (FA) --
                node [pos=0.5, anchor=center, rotate=90, align=center] {\scriptsize time \\ \scriptsize varying}
            (ILMM);
        \draw [arrow]
            (ILMM) --
                node [pos=0.5, anchor=south] {\scriptsize orthogonality constraint}
            (OILMM);
        \draw [arrow]
            (PPCA) -- 
                node [pos=0.5, anchor=center, rotate=90, align=center] {\scriptsize time \\ \scriptsize varying}
            (OILMM);
    \end{tikzpicture}
    \caption{Relationship between factor analysis (FA), probabilistic principal component analysis \citep[PPCA,][]{Tipping:1999:Probabilistic_Principal_Component_Analysis}, the ILMM (\cref{mod:ILMM}), and the OILMM (\cref{mod:OILMM})}
    \vspace*{-0.5em}
    \label{fig:relationships}
\end{figure}
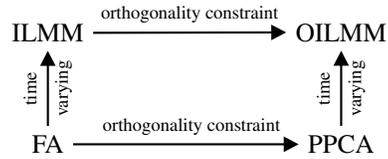

\subsection{Generality of the OILMM}
A central theme of the experiments will be to assess how restrictive the orthogonality assumption is for the OILMM.
In this section, we theoretically investigate this question from various perspectives.
In the separable case, where $K(t,t') = k(t,t') I_m$ for a scalar-valued kernel $k$, for every ILMM with homogeneous observation noise ($\Sigma= \sigma^2 I_p$), there exists an OILMM with $D=0$ that is equal in distribution to $y$.
To see this, note
\[\begin{aligned}
    y^{\text{(ILMM)}} &\sim \GP(0, k(t, t') H H^\T + \sigma^2 \delta[t-t'] I_p), \\
    y^{\text{(OILMM)}} &\sim \GP(0, k(t, t') U S U^\T + \sigma^2 \delta[t-t'] I_p).
\end{aligned}\]
Hence, letting $USU^\T$ be the eigendecomposition of $H H^\T$ gives an OILMM equal in distribution to $y$.
In the nonseparable case, where diagonal elements of $K$ are linearly independent, in general only the distribution of $y(t)$ at every $t$ can be recovered by an OILMM, but the correlation between $y(t)$ and $y(t')$ for $t' \neq t$ may be different.
In terms of the joint distribution over $x$ and $y$, which is important for interpretability of the latent processes, \cref{prop:kl} in \cref{app:kl} shows that the Kullback--Leibler (KL) divergence between two ILMMs with bases $H$ and $\hat H$ is proportional to $\norm{H - \hat H}_F^2$, hence symmetric, where $\norm{\vardot}_F$ denotes the Frobenius norm.
As a consequence (\cref{prop:kl}), the KL between an ILMM with basis $H$ and the OILMM closest in KL is upper bounded by $\norm{I_m - V}_F^2$ where $V$ are the right singular vectors of $H$.
This makes sense: $V = I_m$ implies that $H$ is of the form $US^{\frac12}$ with $U$ a matrix with orthonormal columns and $S > 0$ diagonal.
It also shows that an ILMM is close to an OILMM if $V$ is close to $I_m$ in the sense of the Frobenius norm.

\subsection{Choice of Basis}
The basis $H$ is a parameter of the model that can be learned through gradient-based optimisation of the likelihood.
Parametrising the orthogonal part $U$ of the basis $H$ takes $\O(m^2 p)$ time and $\O(m p)$ memory (see \cref{app:complexities}).
This complexity is \emph{quadratic} in $m$, rather than linear.
However, the cost of parametrising $U$ is typically far from dominant, which means that this cost is typically negligible.
See \cref{app:cost_u} for a more detailed discussion.

Observing that $\E[f(t)f^\T(t)]=HH^\T$, a sensible initialisation of the basis $H$ is (a truncation of) $\hat U\hat S^{\frac12}$ where $\hat \Sigma = \hat U \hat S \hat U^\T$ is the eigendecomposition of an estimate $\hat \Sigma$ of the spatial covariance.
In the case that there is a kernel over the outputs, \eg{} in separable spatio-temporal GP models, $H$ can be set to (a truncation of) $US^{\frac12}$ where $USU^\T$ is an eigendecomposition of the kernel matrix over the outputs.
The hyperparameters of the kernel over the outputs can then be learned with gradient-based optimisation by differentiating through the eigendecomposition.
See \cref{sec:spatio-temporal}.

\subsection{Diagonal Projected Noise}
As alluded to in \cref{sec:ILMM}, under the OILMM, the projected noise $\Sigma_T$ from \cref{prop:general_sufficiency} is diagonal:
$
    \Sigma_T = \sigma^2 S^{-1} + D
$;
\cref{prop:decoupling} in the \cref{app:diagonal-noise} characterises exactly when this is the case.
This property is crucial, because, as we explain in the next paragraph, it allows the model to break down the high-dimensional multi-output problem into independent single-output problems, which brings significant computational advantages.

\subsection{Inference}
Since the projected noise is diagonal, the latent processes remain independent when data is observed.
We may hence treat the latent processes independently, conditioning the $i$\textsuperscript{th} latent process $x_i$ on $(T Y)_{i:} = Y^\T U_{i:} / \sqrt{S_{ii}}$ under noise $(\Sigma_{T})_{ii} = \sigma^2 /S_{ii} + D_{ii}$,
which means that the high-dimensional prediction problem breaks down into independent single-output problems.
Therefore, inference takes $\O(n^3 m + nmp)$ time and $\O(n^2 m + np)$ memory (see \cref{app:complexities}), which are \emph{linear} in $m$.
This decoupled inference procedure is depicted in \cref{fig:commutative_diagram_OILMM} and outlined in more detail in \cref{app:subsec:inference,app:subsec:sampling}.
Note that the decoupled problems can be treated in parallel to achieve sublinear wall time, and that in the separable case further speedups are possible.

\subsection{Learning}
For computing the marginal likelihood, the OILMM also offers computational benefits.
\cref{prop:likelihood} in \cref{app:oilmm-likelihood} shows that $\log p(Y)$ from \cref{prop:sufficiency} simplifies to:
\[\begin{aligned}
    &\log p(Y)  \\
    &=
        \!-\! \frac{n}{2} \! \log |S|
        \!-\! \frac{n (p\!-\!m)}{2} \! \log 2 \pi \sigma^2
        \!-\! \frac{1}{2\sigma^2}\norm{(I_p \!-\! UU^\T)Y}_F \\
    &\qquad + \textstyle\sum_{i=1}^m \log \Normal((TY)_{i:}\cond 0, K_i + (\sigma^2/S_{ii} + D_{ii}) I_n)
\end{aligned}\]
where $\norm{\vardot}_F$ denotes the Frobenius norm and
$K_i$ is the $n\times n$ kernel matrix for the $i$\textsuperscript{th} latent process $x_i$.
We conclude that learning also takes $\O(n^3m + nmp)$ time and $\O(n^2m + np)$ memory (see \cref{app:complexities}), again \emph{linear} in the number of latent processes.
Computation of the marginal likelihood is outlined in more detail in \cref{app:subsec:likelihood}.

\subsection{Interpretability}
Besides computational benefits, the fact that the OILMM breaks down into independent problems for the latent processes also promotes interpretability.\footnote{In the OILMM, the latent processes retain independence in the posterior distribution, which is not generally true for the ILMM.}
Namely, the independent problems can be separately inspected to interpret, diagnose, and improve the model.
This is \emph{much} easier than directly working with predictions for the data, which are high dimensional and often strongly correlated between outputs.
For example, the OILMM allows a simple and interpretable decomposition of the mean squared error:
\vspace{-.15em}
\begin{equation*}
    \underbracket{\norm{y - H x}^2 \vphantom{\sum}}_{
        \text{MSE}
    }
    = \underbracket{\norm{P_{H^\perp} y}^2 \vphantom{\sum}}_{
        \mathclap{\substack{\text{data not} \\ \text{captured by basis}}}
    } + \sum_{i=1}^m
    S_{ii}
    \underbracket{((T y)_i - x_i)^2, \vphantom{\sum}}_{
        \mathclap{\substack{\text{MSE of} \\ \text{$i$\textsuperscript{th} latent process}}}
    }
\end{equation*}
\vspace{-.1em}%
where $P_{H^\perp}$ is the orthogonal projection onto the orthogonal complement of $\col(H)$.
See \cref{prop:mse} in \cref{app:mse} for a proof.

\subsection{Scaling}
For both learning and inference, the problem decouples into $m$ independent single-output problems.
Therefore, to scale to a large number of data points $n$, off-the-shelf single-output GP scaling techniques can be trivially applied to these independent problems.
For example, if the variational inducing point method by \textcite{Titsias:2009:Variational_Learning} is used with $r \ll n$ inducing points, then inference and learning are further reduced to $\O(nm r^2)$ time and $\O(nm r)$ memory, ignoring the cost of the projection (see \cref{app:complexities}).
Most importantly, if $k(t,t')$ is Markovian (\eg\ of the Mat\'ern class), then one can leverage state-space methods to efficiently solve the $m$ independent problems exactly \parencite{Hartikainen:2010:Kalman_filtering_and_smoothing_solutions,Sarkka:2019:Applied_Stochastic_Differential_Equations}.
This brings down the scaling to $\O(nmd^3)$ time and $\O(nmd^2)$ memory, where $d$ is the state dimension, typically $d \ll m, n$ (see \cref{app:complexities}).
We further discuss this approach in \cref{sec:spatio-temporal}.

\newcommand{\obs}{_{\text{o}}}
\newcommand{\miss}{_{\text{m}}}

\subsection{Missing Data}
Missing data is troublesome for the OILMM, because it is not possible to take away a subset of the rows of $H$ and retain orthogonality of the columns.
In this section, we develop an approximation for the OILMM to deal with missing data in a simple and effective way.
For a matrix or vector $A$,
let $A\obs$ and $A\miss$  denote the rows of $A$ corresponding to respectively observed and missing values.
Also, for a matrix $A$, let $d[A]$ denote the diagonal matrix resulting from setting the off-diagonal entries of $A$ to zero.
In the case of missing data, \cref{prop:form-missing}
in \cref{app:missing-data-projection-noise-approximation} shows that the projection and projected noise are given by $T\obs = S^{-\frac12} {U\obs}^\dagger$ and $\Sigma_{T\obs} = \sigma^{2} S^{-\frac12}({U\obs}^\T U\obs)^{-1} S^{-\frac12} + D$.
Observe that $\Sigma_{T\obs}$ is dense, because, unlike $U$, the columns of $U\obs$ are not orthogonal.
However, they may be approximately orthogonal, which motivates the approximation $\Sigma_{T\obs} \approx d[\Sigma_{T\obs}]$.
\cref{prop:diagonal-approximation} in \cref{app:missing-data-projection-noise-approximation} shows that this approximation will be accurate if missing observations cannot decrease the norm of a vector in $\col(H)$ too much:
\[
    \e\ss{rel}
    =\frac
        {\norm{\Sigma_{T\obs} - d[\Sigma_{T\obs}]}\ss{op}}
        {\norm{d[\Sigma_{T\obs}]}\ss{op}}
    \lessapprox
        \max_{y \in \col(H) : \norm{y}=1} \norm{y\miss}^2
\]
where $\norm{\vardot}\ss{op}$ denotes the operator norm and $\lessapprox$ denotes inequality up to a proportionality constant.
For example, if the $i$\textsuperscript{th} column of $H$ is a unit vector, say $e_k$, then the bound does not guarantee anything.
Indeed, if the $k$\textsuperscript{th} output is missing, then potentially all information about the $i$\textsuperscript{th} latent process is lost.
On the other hand, if, for example, $\norm{U}^2_\infty \lessapprox 1/p$, then \cref{prop:bound-U-inequality} in \cref{app:missing-data-projection-noise-approximation} shows that $\e\ss{rel} \lessapprox s/p$ if $s$ outputs are missing, which means that the approximation will be accurate if $s \ll p$.
With this approximation, two things change in the log-likelihood (\cref{rem:using-form-missing} in \cref{app:missing-data-projection-noise-approximation}):
for every time point with missing data
{\it (i)} $UU^\T$ becomes $U\obs U\obs^\dagger$ and
{\it (ii)} an extra term $-\tfrac12\log |U\obs^\T U\obs|$ appears.

It is also easy to use variational inference to handle missing data (\cref{app:missing-data-variational-approximation}) and to support heterogeneous observation noise (\cref{app:oilmm-het-obs-noise}), but we leave experimental tests of these approaches to future work.

\subsection{Application to Separable Spatio--Temporal GPs}
\label{sec:spatio-temporal}

Separable spatio--temporal GPs, which are of the form
$
    f \sim \GP(0,k_t(t,t') k_r(r,r'))
$,
form a vector-valued process $f(t) = (f(t, r_i))_{r=1}^p \sim \GP(0, k_t(t, t') K_r)$ when observed at a fixed number of locations in space, where $K_r$ is the $p \times p$ matrix with $(K_r)_{ij} = k_r(r_i, r_j)$. 
Letting $USU^\T$ be the eigendecomposition of $K_r$, $f(t)$ is an OILMM with $H=US^{\frac12}$ and $K(t,t')=k_t(t,t') I_p$. Note that $m = p$, so the projection takes $O(np^2)$ time (see \cref{app:complexities}).
\newpage

Combining the OILMM framework with efficient state-space scaling techniques \parencite{Hartikainen:2010:Kalman_filtering_and_smoothing_solutions, Sarkka:2019:Applied_Stochastic_Differential_Equations, Solin:2018:Infinite-Horizon_Gaussian_Processes, nickisch2018state}, which are either exact or arbitrarily good approximations, the complexities are reduced to $\O(np^2 + p^3)$ time and $\O(np + p^2)$ memory for the entire problem, which are \emph{linear} in $n$ (see \cref{app:complexities}).
This compares favourably with the filtering techniques of \textcite{Sarkka:2013:Spatiotemporal_Learning_via} and \textcite{hartikainen2011sparse}, both of which have $\O(np^3)$ time and $\O(np^2)$ memory, and the Kronecker product decomposition \citep[][Ch.~5]{saatcci2012scalable} approach, which requires $\O(p^3 + n^3)$ time and $\O(p^2 + n^2)$ memory complexity.

By relaxing $K$ to be a general diagonal multi-output kernel with $K(t,t)=I_p$, we obtain a new class of models which are nonseparable relaxations of the above in which exact inference remains efficient.
The orthogonal basis for this OILMM is, as before, the eigenvectors of a kernel matrix whose hyperparameters can be optimised.

\section{Experiments}
\label{sec:experiments}
We test the OILMM in experiments on synthetic and real-world data sets.
A Python implementation and code to reproduce the experiments is available at {\footnotesize \url{https://github.com/wesselb/oilmm}}.
A Julia implementation is available at {\footnotesize \url{https://github.com/willtebbutt/OILMMs.jl}}.

\begin{figure}[t]\small
    \centering
    \vspace{0.5em}
    \includegraphics[width=\linewidth]{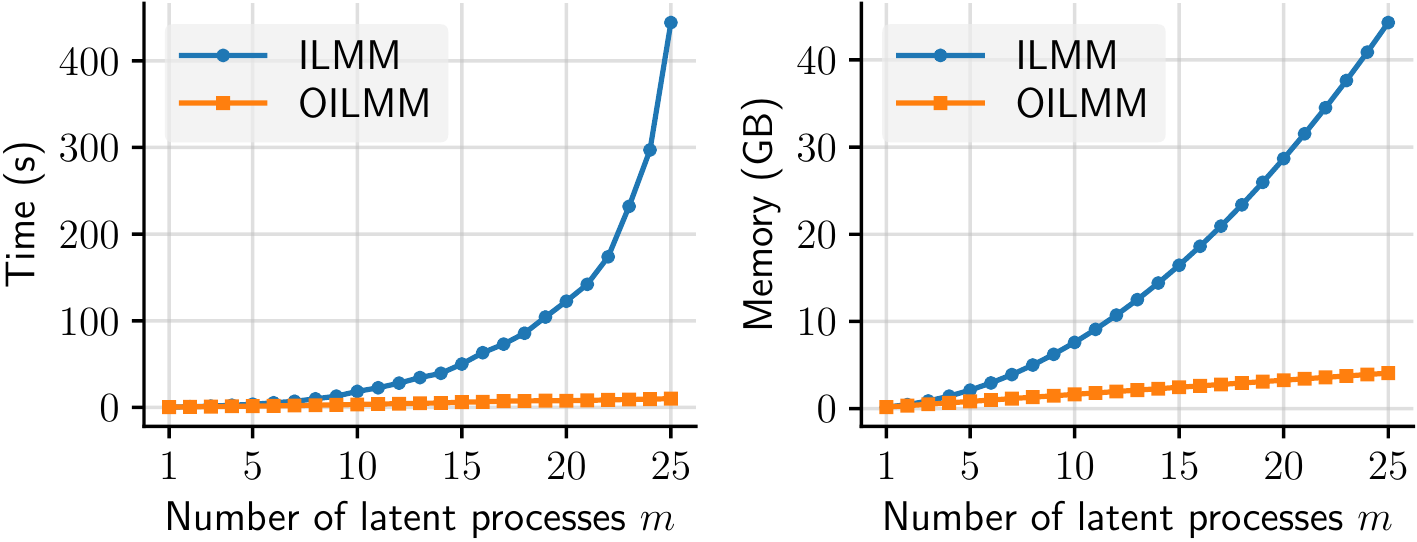}
    \vspace*{-1.5em}
    \caption[Runtime and memory usage of the ILMM and OILMM for various numbers of latent processes]{
        Runtime (left) and memory usage (right) of the ILMM and OILMM for computing the evidence of $n=1500$ observations for $p=200$ outputs.
    }
    \vspace*{-1em}
    \label{fig:scaling}
\end{figure}

\subsection{Computational Scaling}
\label{exp:scaling}
We demonstrate that exact inference scales favourably in $m$ for the OILMM, whereas the ILMM quickly becomes computationally infeasible as $m$ increases.
We use a highly optimised implementation of exact inference for the ILMM, kindly made available by Invenia Labs \footnote{\url{https://invenialabs.co.uk}}.
\cref{fig:scaling} shows the runtime and the memory usage of the ILMM and OILMM.
Observe that the ILMM scales $\O(m^3)$ in time and $\O(m^2)$ in memory, whereas the OILMM scales $\O(m)$ in both time and memory.
For $m=25$, the ILMM takes nearly 10 minutes to compute the evidence, whereas the OILMM only requires a couple seconds.
See \cref{app:scaling} for more details.

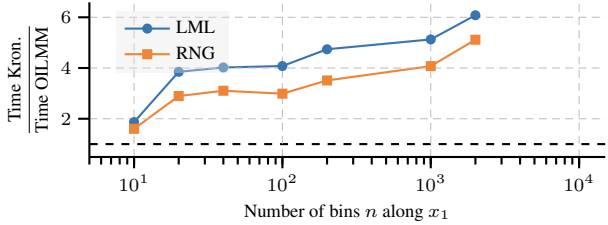
\begin{figure}[t]
    \centering\scriptsize
    % Colours
    \definecolor{mycolor1}{HTML}{3976af}
    \definecolor{mycolor2}{HTML}{ef8536}
    % Adjust figure size below
    \setlength{\figurewidth}{.4\textwidth}
    \setlength{\figureheight}{.6\figurewidth}  
    \tikzset{every picture/.append style = {xscale = .5, yscale = .5}}
    % Customize the figure
    \pgfplotsset{
        width=\figurewidth,
        height=0.5\figureheight,
        xlabel={Number of bins $n$ along $x_1$},
        ylabel={$\displaystyle \frac{\text{Time Kron.}}{\text{Time OILMM}}$},
        grid=major,
        xmin=5,
        xmax=15000,
        major grid style={densely dashed,black!25},
        legend style = {{at={(0.05, 0.95)},anchor=north west}},
        legend style={legend cell align=left, align=left, draw=none, fill=black!10, fill opacity=0.35, text opacity=1, draw opacity=1},
        scale only axis,
        axis on top,
        axis x line*=bottom,
        axis y line*=left,
        every axis plot/.append style={thick},
        axis line style = thick,
        xtick align=outside,
        xtick style={black, thick},
        ytick align=outside,
        ytick style={black, thick},
        clip=false
    }
    \vspace*{1em}
    \begin{tikzpicture}
\begin{axis}[xmode = {log}]
\addplot+[color=mycolor1, mark options={fill=mycolor1, opacity=1}, mark size=1.5pt] coordinates {
(2000.0,6.080886413727074)
(1000.0,5.127248318776301)
(200.0,4.739247269903093)
(100.0,4.077705646480687)
(40.0,4.02122448158909)
(20.0,3.85026607492165)
(10.0,1.8683002927560601)
};
\addlegendentry{LML}
\addplot+[color=mycolor2, mark options={fill=mycolor2, opacity=1}, mark size=1.5pt] coordinates {
(2000.0,5.116240323495079)
(1000.0,4.0743851853265305)
(200.0,3.5106622582877254)
(100.0,2.9887877155369686)
(40.0,3.1036948231517885)
(20.0,2.8949935749493347)
(10.0,1.6033773298694052)
};
\addlegendentry{RNG}
\addplot[color=black, dashed] coordinates {
(15000.0, 1)
(5.0, 1)
};
\end{axis}
\end{tikzpicture}
    \vspace*{-1em}
    \caption{Ratio of timings of the Kronecker approach \citep[][Ch.~5]{saatcci2012scalable} and the OILMM to compute the marginal likelihood of the latent function (LML) and to generate a single prior sample (RNG). See \cref{tab:kronecker-timing-statistics} in \cref{app:point-process} for full results.}
    \label{fig:timing-ratios}
    \vspace*{-0.5em}
\end{figure}

\subsection{Rainforest Tree Point Process Modelling}
\label{exp:rainforest}
We consider a subset of the extensive rain forest data set credited to \citet{rainforest,Condit:2005,Hubbell:1999} in which the locations of $12929$ \emph{trichilia tuberculata} have been recorded.
This data is modelled via an inhomogeneous Poisson process, whose log-intensity is given a GP prior.
Inference is framed in terms of a latent GP with a Poisson likelihood over a discrete collection of bins.
The methodology of \citet{Solin:2018:Infinite-Horizon_Gaussian_Processes} is adapted to accelerate inference of the latent processes, which demonstrates the ability of the OILMM to be combined with existing scaling techniques in a plug-and-play fashion.
Inference in the kernel parameters and log-intensity process utilise a simple blocked Gibbs sampler.

It takes roughly three hours\footnote{
    3.6 GHz Intel Core i7 processor and 48 GB RAM
} to perform $10^5$ iterations of MCMC (\textit{circa} $10^5$ marginal likelihood evaluations and $10^6$ prior samples) with $20000$ bins, demonstrating the feasibility of a computationally demanding choice of approximate inference procedure.
The Kronecker product factorisation technique \citep[][Ch.~5]{saatcci2012scalable} is a competitive method in this setting, as it can also efficiently and exactly compute log marginal likelihoods and generate prior samples efficiently.
\cref{fig:timing-ratios} shows the trade off between the two approaches to inference.
In this experiment we define $p = n / 2$, meaning that the approach described in \cref{sec:spatio-temporal} scales cubically in $n$.
Despite their quite different implementation details, they do obtain similar performance, with the OILMM performing relatively better as $n$ increases.
See \cref{app:point-process} for further experimental details and analysis.

\begin{table}[t]\small\centering%
    \caption{
        Root-mean-square error (RMSE) and
        normalised posterior predictive log-probability (PPLP)
        of held-out test data for the OILMM with varying $m$ and independent GPs (IGP) in the temperature extrapolation experiment.
        The OILMM achieves parity in RMSE with IGP at $m=200$ and surpasses it in PPLP at $m=5$.
    }
    \vspace{1em}
    \begin{tabular}{
        @{}
        c@{\hspace{6pt}}
        l@{\hspace{4pt}}
        c@{\hspace{4pt}}
        c@{\hspace{4pt}}
        c@{\hspace{4pt}}
        c@{\hspace{4pt}}
        c@{\hspace{4pt}}
        c@{}
    }
        \toprule
        & \multicolumn{1}{c}{$m$} & $1$ & $5$ & $50$ & $100$ & $200$ & $247$ \\ \midrule
        \multirow{2}{*}{\rotatebox[origin=c]{90}{\scriptsize RMSE}}
        & OILMM & $2.151$ & $2.072$ & $2.030$ & $2.002$ & $1.992$ & $1.991$ \\ 
        & IGP & $1.993$ & $1.993$ & $1.993$ & $1.993$ & $1.993$ & $1.993$ \\ \midrule
        \multirow{2}{*}{\rotatebox[origin=c]{90}{\scriptsize PPLP \hspace*{-2pt}}}
        & OILMM & $-1.976$ & $-1.457$ & $-0.905$ & $-0.774$ & $-0.600$ & $-0.525$ \\ 
        & IGP & $-1.923$ & $-1.923$ & $-1.923$ & $-1.923$ & $-1.923$ & $-1.923$ \\ \bottomrule
    \end{tabular}
    \label{tab:temperature_extrapolation_results}
    \vspace*{0.5em}
\end{table}

\subsection{Temperature Extrapolation}
\label{exp:temp_extrapolation}
Having demonstrated that the OILMM offers computational benefits, we now show that the method can scale to large numbers of latent processes ($m=p=247$) to capture meaningful dependencies between outputs.
We consider a simple spatio--temporal temperature prediction problem over Europe.
Approximately 30 years worth of the ERA-Interim reanalysis temperature data\footnote{%
    \label{footnote:regridding}%
    All output from CMIP5 and ERA-Interim models was regridded onto the latitude--longitude grid used for the IPSL-CM5A-LR model.
} \parencite{dee2011era} is smoothed in time with a Hamming window of width 31 and sub-sampled once every 31 days to produce a data set comprising $13 \times 19 = 247$ outputs and approximately $350$ months worth of data.
We train the OILMM and IGPs (both models use \Matern--$5/2$ kernels with a periodic component) on the first 250 months of the data and test on the next 100 months.
For the OILMM, we use a range of numbers of latent processes, up to $m = p = 247$, and let the basis $H$ be given by the eigenvectors of the kernel matrix over the points in space (\Matern--$5/2$ with a different length scale for latitude and longitude).

\cref{tab:temperature_extrapolation_results} summarises the performance of the models; more detailed graphs can be found in \cref{app:temp}.
The OILMM achieves parity in RMSE with IGP at $m=200$ latent processes---the data is highly periodic and the predictions are accurate for both models.
Moreover, the OILMM requires only $m=5$ latent processes to achieve a better PPLP than IGP and continues to improve with increasing $m$, demonstrating the need for a large number of latent processes.

\begin{figure*}[t]
    \centering
    \includegraphics[width=\linewidth]{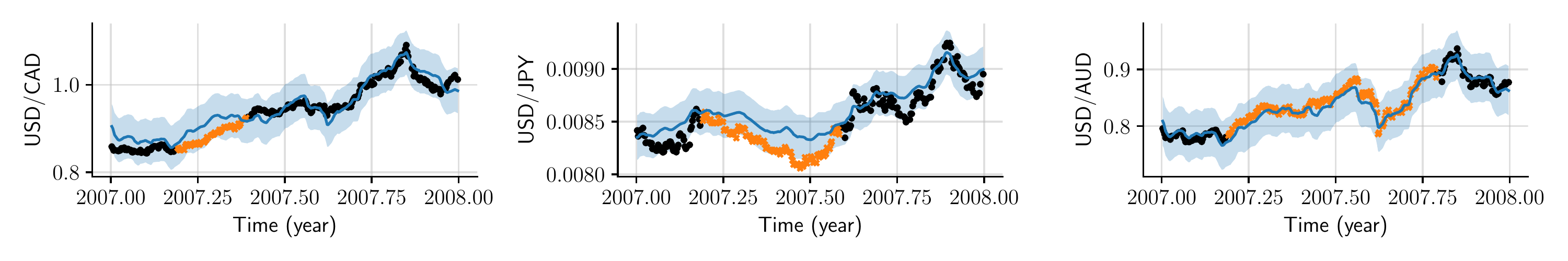}
    \includegraphics[width=\linewidth]{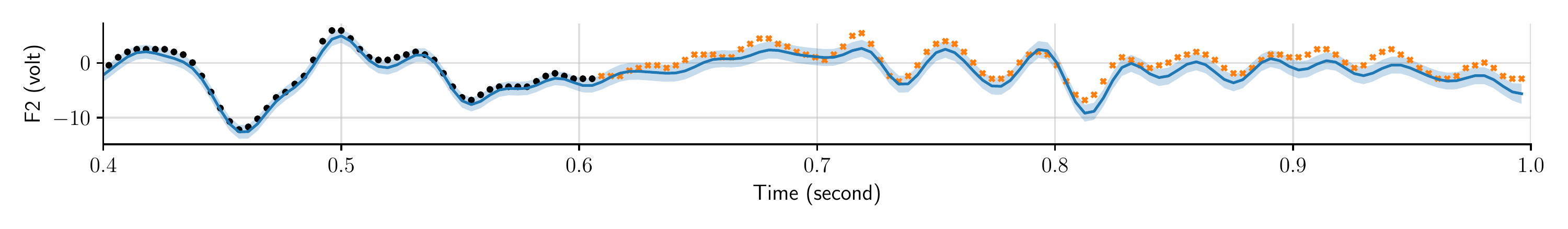}
    \vspace{-2.5em}
    \caption{
        Predictions of the OILMM for the exchange rates experiment (top) and for one of the seven electrodes (F2) in the EEG experiment (bottom).
        Predictions are shown in blue, denoting the mean and central $95\%$ credible region.
        Training data are denoted as black dots ($\bullet$) and held-out test data as orange crosses (\mplorange{$\times$}).
    }
    \vspace*{-1em}
    \label{fig:prediction_ex_and_eeg}
\end{figure*}

\subsection{Exchange Rates Prediction}
\label{exp:exchange}
In this experiment and the next, we test the orthogonality assumption and missing data approximation of the OILMM by comparing its performance to an equivalent ILMM with no restrictions on $H$ and which deals exactly with missing data.
We consider daily exchange rates with respect to USD of the top ten international currencies and three precious materials in the year 2007.
The task is to predict CAD, JPY, and AUD on particular days given that all other currencies are observed throughout the whole year;
we exactly follow \citet{Requeima:2019:The_Gaussian_Process_Autoregressive_Regression} in the data and setup of the experiment.
For the (O)ILMM, we use $m=3$ latent processes with \Matern--$1/2$ kernels and randomly initialise and learn the basis $H$.

\newpage
\cref{tab:results_er_and_eeg} shows that the ILMM and OILMM have identical performance.
This shows that the orthogonality assumption and missing data approximation of the OILMM can work well in practice.

\begin{table}[t]\small\centering%
    \caption{%
        Standardised mean-squared error (SMSE) and normalised posterior predictive log-probability (PPLP) of held-out test data for various models in the exchange rates (ER) and EEG experiment.
        IGP stands for independent GPs.
        The references in square brackets are to models in \cref{fig:mixing_model_hierarchy} in \cref{app:mmh}.
        GPAR (right column) is not a linear MOGP, and thus not comparable to the other methods. However, it is state-of-the-art on both tasks and hence provided as reference.
        The ILMM and OILMM achieve results equal up to two decimal places.
        ${}^*$Numbers are taken from \citet{Nguyen:2014:Collaborative_Multi-Output}.
        ${}^\dagger$Numbers are taken from \citet{Requeima:2019:The_Gaussian_Process_Autoregressive_Regression}.
    }
    \vspace{1em}
    \setlength{\tabcolsep}{3pt}
    \begin{tabular}{
        @{}
        c@{\hspace{5pt}}
        l@{\hspace{3pt}}
        c@{\hspace{3pt}}
        c@{\hspace{3pt}}
        c@{\hspace{3pt}}
        c@{\hspace{3pt}}
        c|
        c@{}
    }
        \toprule
        && IGP           & CMOGP$^{\text{[11]}}\!$ & CGP$^{\text{[14]}}\!$ & ILMM   & OILMM  & GPAR$^{\text{[18]}}\!$ \\ \midrule
        \multirow{2}{*}{\rotatebox[origin=c]{90}{\scriptsize SMSE}}
        & ER  & $0.60^*$       & $0.24^*$   & $0.21^*$ & $0.19$ & $0.19$ & $0.03^\dagger$ \\
        & EEG & $1.75^\dagger$ &            &          & $0.49$ & $0.49$ & $0.26^\dagger$ \\ \midrule
        \multirow{2}{*}{\rotatebox[origin=c]{90}{\scriptsize PPLP\hspace*{-1pt}}}
        & ER  &  $3.57$ &            &          & $3.39$ & $3.39$ & \\
        & EEG & $-1.27$ &            &          & $-2.11$ & $-2.11$ & \\
        \bottomrule
    \end{tabular}
    \label{tab:results_er_and_eeg}
    \vspace*{-0.5em}
\end{table}

\subsection{Electroencephalogram Prediction}
\label{exp:eeg}
We consider 256 voltage measurements from 7 electrodes placed on a subject's scalp while the subject is shown a certain image;
\citet{Zhang:1995:Event_Related_Potentials_During_Object} describes the data collection process in detail.
The task is to predict the last 100 samples for three electrodes given that the remainder of the data is observed;
we exactly follow \citet{Requeima:2019:The_Gaussian_Process_Autoregressive_Regression} in the data and setup of the experiment.
For the (O)ILMM, we use $m=3$ latent processes with exponentiated quadratic kernels and randomly initialise and learn $H$.

\cref{tab:results_er_and_eeg} shows that the ILMM and OILMM again have identical performance.
This again shows that the orthogonality assumption does not harm the model's predictive power and that the missing data approximation can work well.

\begin{figure*}[t]
    \centering\footnotesize
    \begin{subfigure}[b]{.45\textwidth}
        \centering
        \includegraphics[width=\linewidth]{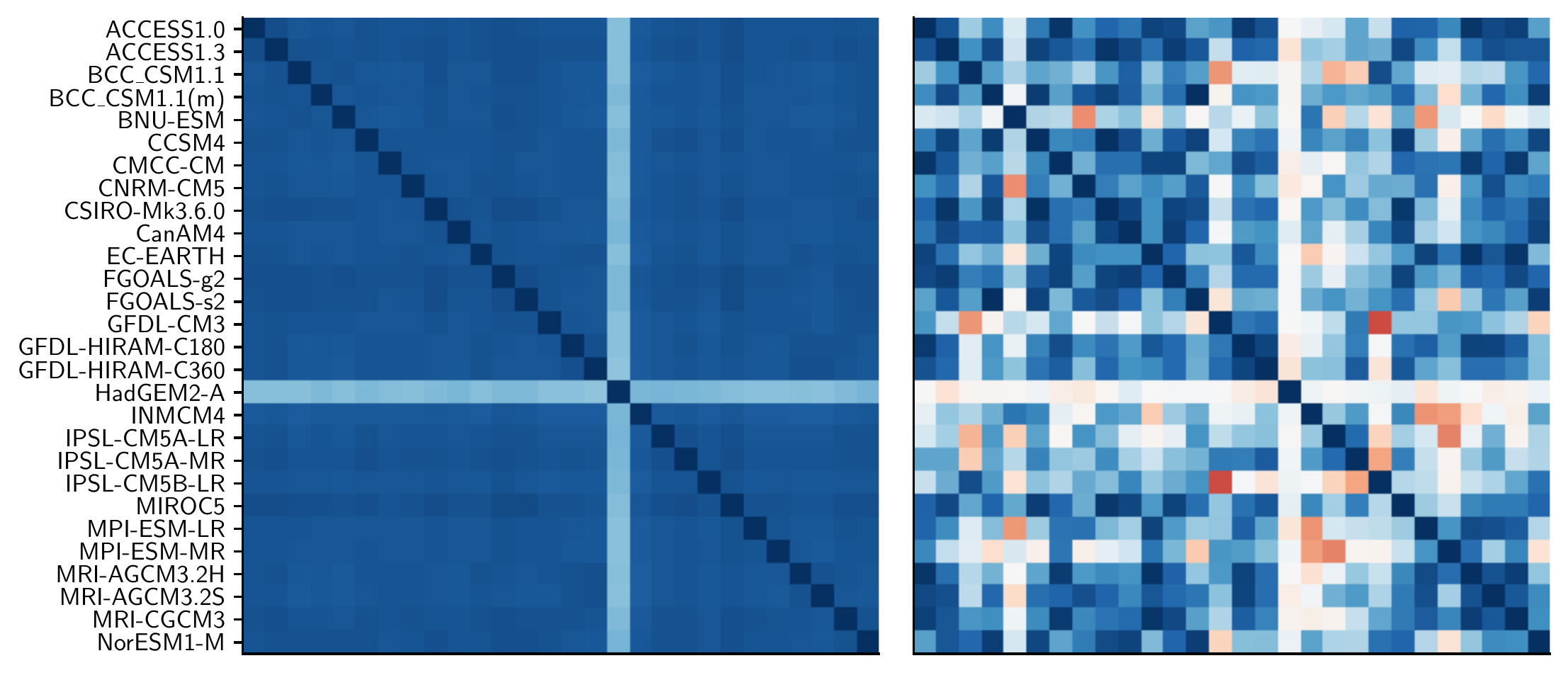}
        \caption{Empirical (left) and learned (right) correlations}
        \label{fig:correlations}    
    \end{subfigure}
    \hspace*{\fill}
    \begin{subfigure}[b]{.20\textwidth}
        \centering
        \includegraphics[width=\linewidth]{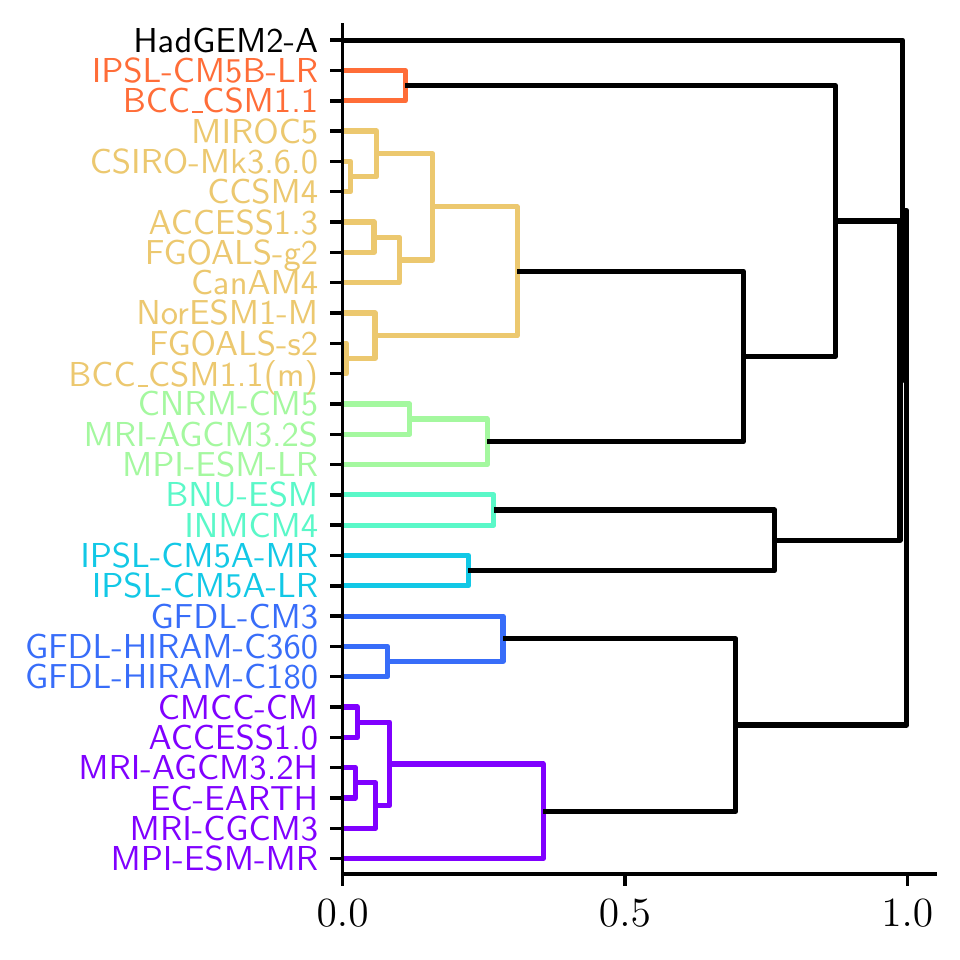}
        \caption{Clustering}
        \label{fig:dendrogram}    
    \end{subfigure}
    \hspace*{\fill}
    \begin{subfigure}[b]{.32\textwidth}
        \centering
        \includegraphics[width=\linewidth]{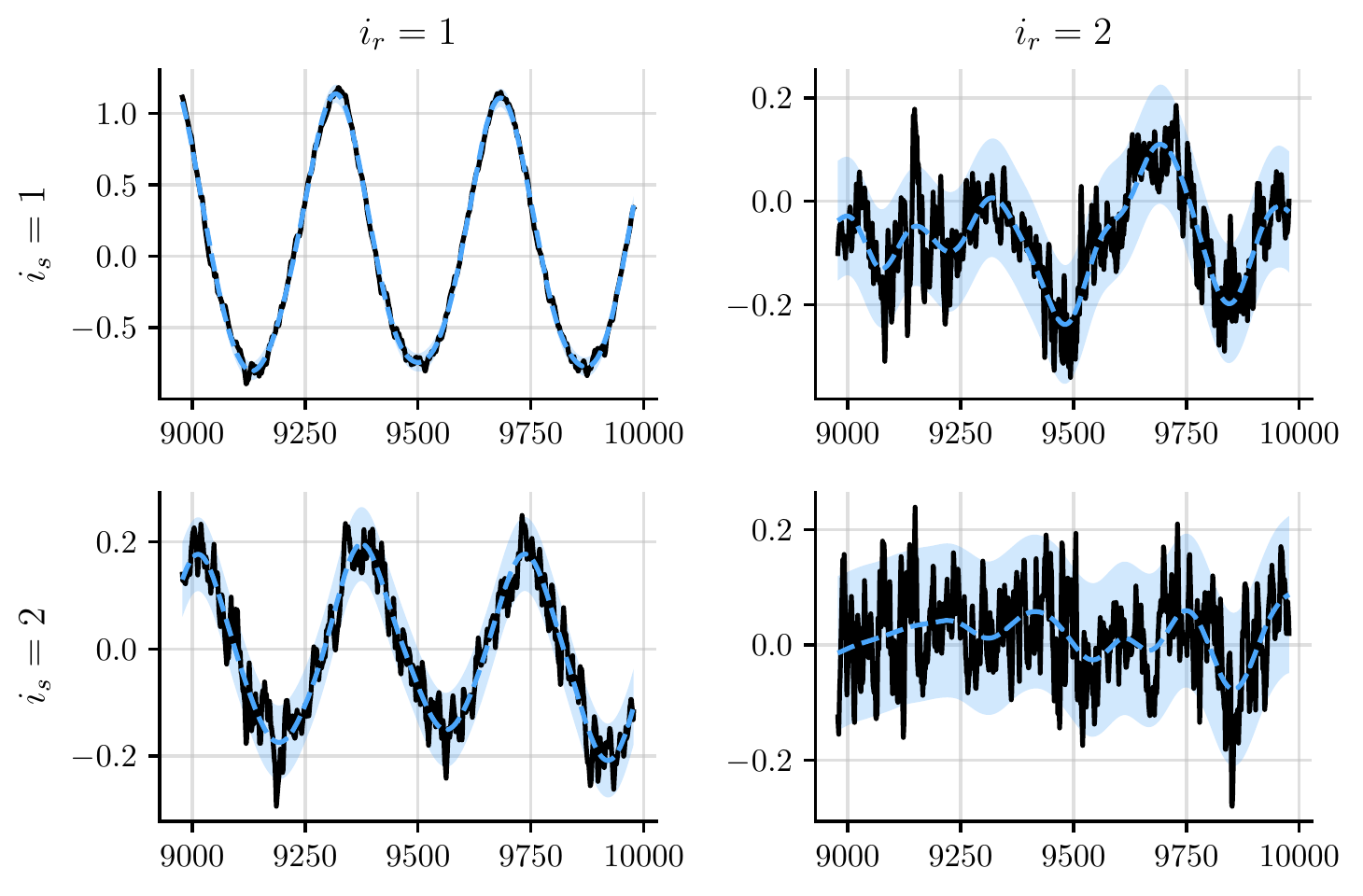}
        \caption{Latent processes}
        \label{fig:latents}        
    \end{subfigure} 
    \caption{
        Results of the large-scale climate simulator experiment, showing
        (a) the empirical correlations and learned correlations ($K_s$) between the simulators,
        (b) a dendrogram deriving from hierarhically clustering the simulators based on the learned correlations where the colours indicate discovered groups,
        and (c) predictions for the latent processes for the first two eigenvectors of the covariance matrix between simulators $i_s=1,2$ and the first two eigenvectors of the spatial covariance $i_r=1,2$, for the last 1000 days.
        Predictions are shown in blue, denoting the mean and central $95\%$ credible region.
    }
    \label{fig:climate}
\end{figure*}

\subsection{Large-Scale Climate model Calibration}
\label{exp:climate}
In this final experiment, we scale to large data in a climate modelling task over Europe.
We use the OILMM to find relationships between 28 climate simulators\footref{footnote:regridding} \citep[see][for background]{taylor2012overview} by letting $H = H_s \kron H_r$ (see \cref{app:tensor}), where $H_s$ are the first $m_s = 5$ eigenvectors of a $28\times 28$ covariance matrix $K_s$ between the simulators, and $H_r$ are the first $m_r = 10$ eigenvectors of the kernel matrix over the points in space (\Matern--$5/2$ with a different length scale for latitude and longitude).
This means that the $m_rm_s = 50$ latent processes are indexed by two indices $i_s$ and $i_r$, one corresponding to the eigenvector of the simulator covariance and one to the eigenvector of the spatial covariance.
The kernels for the latent processes are \Matern--$5/2$.
We consider $n=10000$ time points for all 28 simulators, each with 247 outputs, giving a total of roughly 70 million data points.
For the independent problems, we use the variational inducing point method by \textcite{Titsias:2009:Variational_Learning}.

\cref{fig:correlations} shows that, as opposed to the empirical correlations, which ignore all temporal structure, the correlations learned by the OILMM exhibit a rich structure.
A clustering of these correlations in \cref{fig:dendrogram} reveals that the identified structure is meaningful, because structurally similar simulators are grouped near each other.
We conclude that the OILMM can be used to analyse large data in a simple and straightforward way, and is able to produce interpretable and meaningful results.
See \cref{app:climate} for further experimental details and analysis of the results.

\section{Discussion and Conclusion}
\label{sec:conclusion}

We investigated the use of a sufficient statistic of the data to accelerate inference in MOGPs with orthogonal bases.
In practice, the proposed methodology scales linearly with the number of latent processes $m$, allowing to scale to large $m$ without sacrificing significant expressivity nor requiring any approximations.
This is achieved by breaking down the high-dimensional prediction problem into independent single-output problems, whilst retaining exact inference.
As a consequence, the method is interpretable, extremely simple to implement, and trivially compatible with off-the-shelf single-output GP scaling techniques for handling large numbers of observations.
We tested the method in a variety of experiments,
demonstrating that it offers significant computational benefits without harming predictive performance.
Interesting future directions are the application to non-Gaussian models for the latent processes (see \cref{prop:general_sufficiency})
and targeting sub-linear time complexity by parallelisation \citep[see, \eg,][]{Sarkka:2019:Temporal_Parallelization_of_Bayesian_Filters}.

\section*{Acknowledgements}
WT acknowledges funding from DeepMind.
JSH is supported by EPSRC grant EP/T001569/1, and NERC grant NE/N018028/1.
AS acknowledges funding from the Academy of Finland (grant numbers 308640 and 324345).
RET is supported by Google, Amazon, ARM, Improbable and EPSRC grants EP/M0269571 and EP/L000776/1.

\bibliographystyle{icml2020}
\bibliography{bibliography}

\begin{thebibliography}{71}
\providecommand{\natexlab}[1]{#1}
\providecommand{\url}[1]{\texttt{#1}}
\expandafter\ifx\csname urlstyle\endcsname\relax
  \providecommand{\doi}[1]{doi: #1}\else
  \providecommand{\doi}{doi: \begingroup \urlstyle{rm}\Url}\fi

\bibitem[{\' A}lvarez \& Lawrence(2009){\' A}lvarez and
  Lawrence]{Alvarez:2009:Sparse_Convolved_Gaussian_Processes_for}
{\' A}lvarez, M. and Lawrence, N.~D.
\newblock Sparse convolved {Gaussian} processes for multi-output regression.
\newblock volume~21, pp.\  57--64. Curran Associates, Inc., 2009.

\bibitem[{\'A}lvarez et~al.(2009){\'A}lvarez, Luengo, and
  Lawrence]{Alvarez:2009:Latent_Force_Models}
{\'A}lvarez, M., Luengo, D., and Lawrence, N.
\newblock Latent force models.
\newblock 5:\penalty0 9--16, 2009.

\bibitem[{\'A}lvarez et~al.(2010){\'A}lvarez, Luengo, Titsias, and
  Lawrence]{Alvarez:2010:Efficient_Multioutput_Gaussian_Processes_Through}
{\'A}lvarez, M., Luengo, D., Titsias, M., and Lawrence, N.~D.
\newblock Efficient multioutput {G}aussian processes through variational
  inducing kernels.
\newblock In \emph{Proceedings of the Thirteenth International Conference on
  Artificial Intelligence and Statistics}, volume~9 of \emph{Proceedings of
  Machine Learning Research}, pp.\  25--32. PMLR, 2010.

\bibitem[{\' A}lvarez \& Lawrence(2011){\' A}lvarez and
  Lawrence]{Alvarez:2011:Computationally_Efficient_Convolved}
{\' A}lvarez, M.~A. and Lawrence, N.~D.
\newblock Computationally efficient convolved multiple output {Gaussian}
  processes.
\newblock \emph{Journal of Machine Learning Research}, 12:\penalty0 1459--1500,
  7 2011.

\bibitem[Bellouin et~al.(2011)Bellouin, Collins, Culverwell, Halloran,
  Hardiman, Hinton, Jones, McDonald, McLaren, O'Connor,
  et~al.]{bellouin2011hadgem2}
Bellouin, N., Collins, W., Culverwell, I., Halloran, P., Hardiman, S., Hinton,
  T., Jones, C., McDonald, R., McLaren, A., O'Connor, F., et~al.
\newblock The hadgem2 family of met office unified model climate
  configurations.
\newblock \emph{Geoscientific Model Development}, 4\penalty0 (3):\penalty0
  723--757, 2011.

\bibitem[Bezanson et~al.(2017)Bezanson, Edelman, Karpinski, and
  Shah]{bezanson2017julia}
Bezanson, J., Edelman, A., Karpinski, S., and Shah, V.~B.
\newblock Julia: A fresh approach to numerical computing.
\newblock \emph{SIAM review}, 59\penalty0 (1):\penalty0 65--98, 2017.

\bibitem[Bi et~al.(2013)Bi, Dix, Marsland, O’Farrell, Rashid, Uotila, Hirst,
  Kowalczyk, Golebiewski, Sullivan, et~al.]{bi2013access}
Bi, D., Dix, M., Marsland, S.~J., O’Farrell, S., Rashid, H., Uotila, P.,
  Hirst, A., Kowalczyk, E., Golebiewski, M., Sullivan, A., et~al.
\newblock The access coupled model: description, control climate and
  evaluation.
\newblock \emph{Aust. Meteorol. Oceanogr. J}, 63\penalty0 (1):\penalty0 41--64,
  2013.

\bibitem[Bonilla et~al.(2007)Bonilla, Agakov, and
  Williams]{Bonilla:2007:Kernel_Multi-Task_Learning_Using_Task-Specific}
Bonilla, E.~V., Agakov, F.~V., and Williams, C. K.~I.
\newblock Kernel multi-task learning using task-specific features.
\newblock In \emph{Proceedings of the Eleventh International Conference on
  Artificial Intelligence and Statistics}, volume~2 of \emph{Proceedings of
  Machine Learning Research}, pp.\  43--50. PMLR, 2007.

\bibitem[Bonilla et~al.(2008)Bonilla, Chai, and
  Williams]{Bonilla:2008:Multi-Task_Gaussian_Process}
Bonilla, E.~V., Chai, K.~M., and Williams, C. K.~I.
\newblock Multi-task {Gaussian} process prediction.
\newblock In \emph{Advances in Neural Information Processing Systems},
  volume~20, pp.\  153--160. MIT Press, 2008.

\bibitem[Boyle \& Frean(2005)Boyle and
  Frean]{Boyle:2005:Dependent_Gaussian_Processes}
Boyle, P. and Frean, M.
\newblock Dependent {Gaussian} processes.
\newblock In Saul, L.~K., Weiss, Y., and Bottou, L. (eds.), \emph{Advances in
  Neural Information Processing Systems {17}}, pp.\  217--224. MIT Press, 2005.

\bibitem[Brochu et~al.(2010)Brochu, Cora, and
  de~Freitas]{Brochu:2010:A_Tutorial_on_Bayesian_Optimization}
Brochu, E., Cora, V.~M., and de~Freitas, N.
\newblock {A} tutorial on {Bayesian} optimization of expensive cost functions,
  with application to active user modeling and hierarchical reinforcement
  learning.
\newblock \emph{arXiv preprint arXiv:1012.2599}, 12 2010.

\bibitem[Bruinsma(2016)]{Bruinsma:2016:GGPCM}
Bruinsma, W.~P.
\newblock The generalised {Gaussian} convolution process model.
\newblock Master's thesis, Department of Engineering, University of Cambridge,
  2016.

\bibitem[Bui et~al.(2016)Bui, Hern{\'a}ndez-Lobato, Hernandez-Lobato, Li, and
  Turner]{Bui:2016:Deep_Gaussian_Processes_for_Regression}
Bui, T., Hern{\'a}ndez-Lobato, D., Hernandez-Lobato, J., Li, Y., and Turner, R.
\newblock Deep {G}aussian processes for regression using approximate
  expectation propagation.
\newblock pp.\  1472--1481, 2016.

\bibitem[Bui et~al.(2017)Bui, Yan, and
  Turner]{Bui:2016:A_Unifying_Framework_for_Gaussian}
Bui, T.~D., Yan, J., and Turner, R.~E.
\newblock A unifying framework for {G}aussian process pseudo-point
  approximations using power expectation propagation.
\newblock \emph{Journal of Machine Learning Research}, 18\penalty0
  (104):\penalty0 1--72, 2017.

\bibitem[Candela \& Rasmussen(2005)Candela and
  Rasmussen]{Quinonero:2005:Unifying_View}
Candela, J.~Q. and Rasmussen, C.~E.
\newblock {A} unifying view of sparse approximate {Gaussian} process
  regression.
\newblock \emph{Journal of Machine Learning Research}, 6:\penalty0 1939--1959,
  12 2005.

\bibitem[Casella \& Berger(2001)Casella and
  Berger]{Casella:2001:Statistical_Inference}
Casella, G. and Berger, R.
\newblock \emph{Statistical Inference}.
\newblock Duxbury Resource Center, 6 2001.

\bibitem[{Chen} \& {Revels}(2016){Chen} and {Revels}]{BenchmarkTools.jl-2016}
{Chen}, J. and {Revels}, J.
\newblock {Robust benchmarking in noisy environments}.
\newblock \emph{arXiv e-prints}, art. arXiv:1608.04295, Aug 2016.

\bibitem[Cheng \& Boots(2017)Cheng and
  Boots]{Cheng:2017:Variational_Inference_for_Gaussian_Process}
Cheng, C.-A. and Boots, B.
\newblock Variational inference for {G}aussian process models with linear
  complexity.
\newblock In \emph{Advances in Neural Information Processing Systems}, pp.\
  5184--5194. Curran Associates, Inc., 2017.

\bibitem[Condit(1998)]{Condit:2005}
Condit, R.
\newblock \emph{Tropical Forest Census Plots}.
\newblock Springer-Verlag and R. G. Landes Company, Berlin, Germany, and
  Georgetown, Texas, 1998.

\bibitem[Dahl \& Bonilla(2019)Dahl and
  Bonilla]{Dahl:2018:Grouped_Gaussian_Processes_for_Solar}
Dahl, A. and Bonilla, E.~V.
\newblock Grouped {Gaussian} processes for solar power prediction.
\newblock \emph{Machine Learning}, 108\penalty0 (8-9):\penalty0 1287--1306,
  2019.

\bibitem[Damianou(2015)]{Damianou:2015:Deep_Gaussian_Processes_and_Variational}
Damianou, A.
\newblock \emph{Deep {Gaussian} Processes and Variational Propagation of
  Uncertainty}.
\newblock PhD thesis, Department of Neuroscience, University of Sheffield,
  2015.

\bibitem[Dee et~al.(2011)Dee, Uppala, Simmons, Berrisford, Poli, Kobayashi,
  Andrae, Balmaseda, Balsamo, Bauer, et~al.]{dee2011era}
Dee, D.~P., Uppala, S., Simmons, A., Berrisford, P., Poli, P., Kobayashi, S.,
  Andrae, U., Balmaseda, M., Balsamo, G., Bauer, d.~P., et~al.
\newblock The era-interim reanalysis: Configuration and performance of the data
  assimilation system.
\newblock \emph{Quarterly Journal of the royal meteorological society},
  137\penalty0 (656):\penalty0 553--597, 2011.

\bibitem[Deisenroth \& Rasmussen(2011)Deisenroth and
  Rasmussen]{Deisenroth:2011:PILCO_A_Model-Based_and_Data-Efficient}
Deisenroth, M.~P. and Rasmussen, C.~E.
\newblock {PILCO:} {A} model-based and data-efficient approach to policy
  search.
\newblock volume~28, pp.\  465--472. Omnipress, 2011.

\bibitem[Dezfouli et~al.(2017)Dezfouli, Bonilla, and
  Nock]{Dezfouli:2017:Semi-Parametric_Network_Structure_Discovery_Models}
Dezfouli, A., Bonilla, E.~V., and Nock, R.
\newblock Semi-parametric network structure discovery models.
\newblock \emph{arXiv preprint arXiv:1702.08530}, 2 2017.

\bibitem[Duvenaud(2014)]{Duvenaud:2014:Automatic_Construction}
Duvenaud, D.
\newblock \emph{Automatic Model Construction With {Gaussian} Processes}.
\newblock PhD thesis, Computational and Biological Learning Laboratory,
  University of Cambridge, 2014.

\bibitem[Goovaerts(1997)]{Goovaerts:1997:Geostatistics_for_Natural_Resources_Evaluation}
Goovaerts, P.
\newblock \emph{Geostatistics for Natural Resources Evaluation}.
\newblock Oxford University Press, 1 edition, 1997.

\bibitem[Hartikainen \& S{\"a}rkk{\"a}(2010)Hartikainen and
  S{\"a}rkk{\"a}]{Hartikainen:2010:Kalman_filtering_and_smoothing_solutions}
Hartikainen, J. and S{\"a}rkk{\"a}, S.
\newblock Kalman filtering and smoothing solutions to temporal {G}aussian
  process regression models.
\newblock In \emph{Proceedings of the IEEE International Workshop on Machine
  Learning for Signal Processing (MLSP)}, pp.\  379--384, 2010.

\bibitem[Hartikainen et~al.(2011)Hartikainen, Riihim{\"a}ki, and
  S{\"a}rkk{\"a}]{hartikainen2011sparse}
Hartikainen, J., Riihim{\"a}ki, J., and S{\"a}rkk{\"a}, S.
\newblock Sparse spatio-temporal gaussian processes with general likelihoods.
\newblock In \emph{International Conference on Artificial Neural Networks},
  pp.\  193--200. Springer, 2011.

\bibitem[Hastings(1970)]{hastings1970monte}
Hastings, W.~K.
\newblock Monte carlo sampling methods using markov chains and their
  applications.
\newblock 1970.

\bibitem[Hennig et~al.(2015)Hennig, Osborne, and
  Girolami]{Hennig:2015:Probabilistic_Numerics_and_Uncertainty_in}
Hennig, P., Osborne, M.~A., and Girolami, M.
\newblock Probabilistic numerics and uncertainty in computations.
\newblock \emph{Proceedings of the Royal Society of London {A:} Mathematical,
  Physical and Engineering Sciences}, 471\penalty0 (2179), 2015.

\bibitem[Hensman et~al.(2013)Hensman, Fusi, and
  Lawrence]{Hensman:2013:Gaussian_Processes_for_Big_Data}
Hensman, J., Fusi, N., and Lawrence, N.~D.
\newblock Gaussian processes for big data.
\newblock In \emph{Proceedings of the 29th Conference on Uncertainty in
  Artificial Intelligence}, pp.\  282--290, 2013.

\bibitem[Hensman et~al.(2018)Hensman, Durrande, and
  Solin]{Hensman:2018:Variational_Fourier_Features_for_Gaussian}
Hensman, J., Durrande, N., and Solin, A.
\newblock Variational fourier features for {Gaussian} processes.
\newblock \emph{Journal of Machine Learning Research}, 18\penalty0
  (151):\penalty0 1--52, 2018.

\bibitem[Higdon et~al.(2008)Higdon, Gattiker, Williams, and
  Rightley]{Higdon:2008:Computer_Model_Calibration_Using_High-Dimensional}
Higdon, D., Gattiker, J., Williams, B., and Rightley, M.
\newblock Computer model calibration using high-dimensional output.
\newblock \emph{Journal of the American Statistical Association}, 103\penalty0
  (482):\penalty0 570--583, 2008.
\newblock ISSN 01621459.
\newblock URL \url{http://www.jstor.org/stable/27640080}.

\bibitem[Hubbell et~al.(1999)Hubbell, Foster, O'Brien, Harms, Condit, Wechsler,
  Wright, and De~Lao]{Hubbell:1999}
Hubbell, S., Foster, R., O'Brien, S., Harms, K., Condit, R., Wechsler, B.,
  Wright, S., and De~Lao, S.
\newblock Light-gap disturbances, recruitment limitation, and tree diversity in
  a neotropical forest.
\newblock \emph{Science}, 283\penalty0 (5401):\penalty0 554--557, 1999.

\bibitem[Hubbell et~al.(2005)Hubbell, Condit, and Foster]{rainforest}
Hubbell, S., Condit, R., and Foster, R.
\newblock Barro colorado forest census plot data.
\newblock {URL}: \url{https://ctfs.arnarb.harvard.edu/webatlas/datasets/bci},
  2005.

\bibitem[Kaiser et~al.(2018)Kaiser, Otte, Runkler, and
  Ek]{Kaiser:2017:Bayesian_Alignments_of_Warped_Multi-Output}
Kaiser, M., Otte, C., Runkler, T., and Ek, C.~H.
\newblock Bayesian alignments of warped multi-output {G}aussian processes.
\newblock In \emph{Advances in Neural Information Processing Systems}, pp.\
  6995--7004, 2018.

\bibitem[Kingma \& Welling(2013)Kingma and
  Welling]{Kingma:2013:Auto-Encoding_VB}
Kingma, D.~P. and Welling, M.
\newblock Auto-encoding variational {Bayes}.
\newblock \emph{arXiv preprint arXiv:1312.6114}, 12 2013.

\bibitem[L{\' a}zaro-Gredilla et~al.(2010)L{\' a}zaro-Gredilla, Candela,
  Rasmussen, and
  Figueiras-Vidal]{Lazaro-Gredilla:2010:Sparse_Spectrum_Gaussian_Process_Regression}
L{\' a}zaro-Gredilla, M., Candela, J.~Q., Rasmussen, C.~E., and
  Figueiras-Vidal, A.~R.
\newblock Sparse spectrum {Gaussian} process regression.
\newblock \emph{Journal of Machine Learning Research}, 11:\penalty0 1865--1881,
  2010.

\bibitem[MacKay(2002)]{MacKay:2002:Information_Theory_Learning}
MacKay, D. J.~C.
\newblock \emph{Information Theory, Inference {\&} Learning Algorithms}.
\newblock Cambridge University Press, 2002.

\bibitem[Matheron(1969)]{Matheron:1969:Le_Krigeage_Universel}
Matheron, G.
\newblock Le krigeage universel.
\newblock In \emph{Cahiers du Centre de morphologie math{\' e}matique de
  Fontainebleau}, volume~1. {\' E}cole nationale sup{\' e}rieure des mines de
  Paris, 1969.

\bibitem[Minka(2000)]{Minka:2000:Quadrature_GP}
Minka, T.
\newblock Deriving quadrature rules from {Gaussian} processes.
\newblock Technical report, 2000.

\bibitem[MISO(2019)]{MISO}
MISO.
\newblock Historical annual real{-}time {LMP}s, 2019.
\newblock URL
  \url{https://www.misoenergy.org/markets-and-operations/real-time--market-data/market-reports/#nt=%2FMarketReportType%3AHistorical%20LMP%2FMarketReportName%3AHistorical%20Annual%20Real-Time%20LMPs%20(zip)&t=10&p=0&s=MarketReportPublished&sd=desc}.

\bibitem[M{\o}ller et~al.(1998)M{\o}ller, Syversveen, and
  Waagepetersen]{Moller+Syversveen+Waagepetersen:1998}
M{\o}ller, J., Syversveen, A.~R., and Waagepetersen, R.~P.
\newblock Log {G}aussian {C}ox processes.
\newblock \emph{Scandinavian Journal of Statistics}, 25\penalty0 (3):\penalty0
  451--482, 1998.

\bibitem[Murray \& Adams(2010)Murray and Adams]{murray2010slice}
Murray, I. and Adams, R.~P.
\newblock Slice sampling covariance hyperparameters of latent {G}aussian
  models.
\newblock In \emph{Advances in Neural Information Processing Systems},
  volume~23, pp.\  1732--1740. Curran Associates, Inc., 2010.

\bibitem[Murray et~al.(2010)Murray, Adams, and
  MacKay]{Murray:2010:Elliptical_Slice_Sampling}
Murray, I., Adams, R., and MacKay, D.
\newblock Elliptical slice sampling.
\newblock In \emph{Proceedings of the Thirteenth International Conference on
  Artificial Intelligence and Statistics}, volume~9 of \emph{Proceedings of
  Machine Learning Research}, pp.\  541--548. PMLR, 13--15 May 2010.

\bibitem[Nguyen \& Bonilla(2014)Nguyen and
  Bonilla]{Nguyen:2014:Collaborative_Multi-Output}
Nguyen, T.~V. and Bonilla, E.~V.
\newblock Collaborative multi-output {Gaussian} processes.
\newblock In \emph{Conference on Uncertainty in Artificial Intelligence},
  volume~30, 2014.

\bibitem[Nickisch et~al.(2018)Nickisch, Solin, and
  Grigorievskiy]{nickisch2018state}
Nickisch, H., Solin, A., and Grigorievskiy, A.
\newblock State space gaussian processes with non-gaussian likelihood.
\newblock \emph{arXiv preprint arXiv:1802.04846}, 2018.

\bibitem[Nocedal \& Wright(2006)Nocedal and
  Wright]{Nocedal:2006:Numerical_Optimisation}
Nocedal, J. and Wright, S.~J.
\newblock \emph{Numerical Optimization}.
\newblock Springer, 2 edition, 2006.

\bibitem[Osborne et~al.(2008)Osborne, Roberts, Rogers, Ramchurn, and
  Jennings]{Osborne:2008:Towards_Real-Time_Information_Processing_of}
Osborne, M.~A., Roberts, S.~J., Rogers, A., Ramchurn, S.~D., and Jennings,
  N.~R.
\newblock Towards real-time information processing of sensor network data using
  computationally efficient multi-output {Gaussian} processes.
\newblock In \emph{Proceedings of the 7th International Conference on
  Information Processing in Sensor Networks}, {IPSN} {'08}, pp.\  109--120.
  IEEE Computer Society, 2008.

\bibitem[Parra \& Tobar(2017)Parra and
  Tobar]{Parra:2017:Spectral_Mixture_Kernels_for_Multi-Output}
Parra, G. and Tobar, F.
\newblock Spectral mixture kernels for multi-output {G}aussian processes.
\newblock In \emph{Advances in Neural Information Processing Systems}, pp.\
  6681--6690. Curran Associates, Inc., 2017.

\bibitem[Rasmussen \& Williams(2006)Rasmussen and
  Williams]{Rasmussen:2006:Gaussian_Processes}
Rasmussen, C.~E. and Williams, C. K.~I.
\newblock \emph{{Gaussian} Processes for Machine Learning}.
\newblock MIT Press, 2006.

\bibitem[Requeima et~al.(2019)Requeima, Tebbutt, Bruinsma, and
  Turner]{Requeima:2019:The_Gaussian_Process_Autoregressive_Regression}
Requeima, J., Tebbutt, W., Bruinsma, W., and Turner, R.~E.
\newblock The {Gaussian} process autoregressive regression model {(GPAR)}.
\newblock In Chaudhuri, K. and Sugiyama, M. (eds.), \emph{Proceedings of
  Machine Learning Research}, volume~89 of \emph{Proceedings of Machine
  Learning Research}, pp.\  1860--1869. PMLR, 4 2019.

\bibitem[Roweis \& Ghahramani(1999)Roweis and
  Ghahramani]{Roweis:1999:A_Unifying_Review_of_Linear}
Roweis, S. and Ghahramani, Z.
\newblock {A} unifying review of linear {Gaussian} models.
\newblock \emph{Neural Computation}, 11\penalty0 (2):\penalty0 305--345, 1999.
\newblock \doi{10.1162/089976699300016674}.
\newblock URL \url{https://doi.org/10.1162/089976699300016674}.

\bibitem[Saat{\c{c}}i(2012)]{saatcci2012scalable}
Saat{\c{c}}i, Y.
\newblock \emph{Scalable Inference for Structured {G}aussian Process Models}.
\newblock PhD thesis, University of Cambridge, Cambridge, UK, 2012.

\bibitem[S{\" a}rkk{\" a} \& Garc{\' i}a-Fern{\' a}ndez(2019)S{\" a}rkk{\" a}
  and Garc{\' i}a-Fern{\'
  a}ndez]{Sarkka:2019:Temporal_Parallelization_of_Bayesian_Filters}
S{\" a}rkk{\" a}, S. and Garc{\' i}a-Fern{\' a}ndez, {\' A}.~F.
\newblock Temporal parallelization of {Bayesian} filters and smoothers.
\newblock \emph{arXiv preprint arXiv:1905.13002}, 5 2019.

\bibitem[S{\"a}rkk{\"a} \& Solin(2019)S{\"a}rkk{\"a} and
  Solin]{Sarkka:2019:Applied_Stochastic_Differential_Equations}
S{\"a}rkk{\"a}, S. and Solin, A.
\newblock \emph{Applied Stochastic Differential Equations}.
\newblock Institute of Mathematical Statistics Textbooks. Cambridge University
  Press, 2019.

\bibitem[S\"arkk\"a et~al.(2013)S\"arkk\"a, Solin, and
  Hartikainen]{Sarkka:2013:Spatiotemporal_Learning_via}
S\"arkk\"a, S., Solin, A., and Hartikainen, J.
\newblock Spatiotemporal learning via infinite-dimensional {B}ayesian filtering
  and smoothing.
\newblock \emph{IEEE Signal Processing Magazine}, 30\penalty0 (4):\penalty0
  51--61, 2013.

\bibitem[Solin et~al.(2018)Solin, Hensman, and
  Turner]{Solin:2018:Infinite-Horizon_Gaussian_Processes}
Solin, A., Hensman, J., and Turner, R.~E.
\newblock Infinite-horizon {Gaussian} processes.
\newblock In \emph{Advances in Neural Information Processing Systems},
  volume~31. Curran Associates, Inc., 2018.

\bibitem[Taylor et~al.(2012)Taylor, Stouffer, and Meehl]{taylor2012overview}
Taylor, K.~E., Stouffer, R.~J., and Meehl, G.~A.
\newblock An overview of {CMIP5} and the experiment design.
\newblock \emph{Bulletin of the American Meteorological Society}, 93\penalty0
  (4):\penalty0 485--498, 2012.

\bibitem[Teh \& Seeger(2005)Teh and
  Seeger]{Teh:2005:Semiparametric_Latent_Factor}
Teh, Y.~W. and Seeger, M.
\newblock Semiparametric latent factor models.
\newblock In \emph{International Workshop on Artificial Intelligence and
  Statistics}, volume~10, 2005.

\bibitem[Tipping \& Bishop(1999)Tipping and
  Bishop]{Tipping:1999:Probabilistic_Principal_Component_Analysis}
Tipping, M.~E. and Bishop, C.~M.
\newblock Probabilistic principal component analysis.
\newblock \emph{Journal of the Royal Statistical Society, Series {B}},
  61\penalty0 (3):\penalty0 611--622, 1999.

\bibitem[Titsias(2009)]{Titsias:2009:Variational_Learning}
Titsias, M.
\newblock Variational learning of inducing variables in sparse {G}aussian
  processes.
\newblock In \emph{Proceedings of the Twelth International Conference on
  Artificial Intelligence and Statistics}, volume~5 of \emph{Proceedings of
  Machine Learning Research}, pp.\  567--574. PMLR, 2009.

\bibitem[Tokdar \& Ghosh(2007)Tokdar and Ghosh]{Tokdar+Ghosh:2007}
Tokdar, S.~T. and Ghosh, J.~K.
\newblock Posterior consistency of logistic {G}aussian process priors in
  density estimation.
\newblock \emph{Journal of Statistical Planning and Inference}, 137\penalty0
  (1):\penalty0 34--42, 2007.

\bibitem[Ulrich et~al.(2015)Ulrich, Carlson, Dzirasa, and
  Carin]{Ulrich:2015:Cross_Spectrum}
Ulrich, K., Carlson, D.~E., Dzirasa, K., and Carin, L.
\newblock {GP} kernels for cross-spectrum analysis, 2015.

\bibitem[Wackernagel(2003)]{Wackernagel:2003:Multivariate_Geostatistics}
Wackernagel, H.
\newblock \emph{Multivariate Geostatistics}.
\newblock Springer-Verlag Berlin Heidelberg, 3 edition, 2003.

\bibitem[Wilson \& Nickisch(2015)Wilson and
  Nickisch]{Wilson:2015:Kernel_Interpolation_for_Scalable_Structured}
Wilson, A. and Nickisch, H.
\newblock Kernel interpolation for scalable structured {G}aussian processes
  ({KISS-GP}).
\newblock 37:\penalty0 1775--1784, 2015.

\bibitem[Wilson et~al.(2012)Wilson, Knowles, and
  Ghahramani]{Wilson:2012:GP_Regression_Networks}
Wilson, A.~G., Knowles, D.~A., and Ghahramani, Z.
\newblock {Gaussian} process regression networks.
\newblock In \emph{International Conference on Machine Learning}, volume~29.
  Omnipress, 2012.

\bibitem[Wu et~al.(2014)Wu, Song, Li, Wang, Zhang, Xin, Zhang, Zhang, Li, Wu,
  et~al.]{wu2014overview}
Wu, T., Song, L., Li, W., Wang, Z., Zhang, H., Xin, X., Zhang, Y., Zhang, L.,
  Li, J., Wu, F., et~al.
\newblock An overview of bcc climate system model development and application
  for climate change studies.
\newblock \emph{Journal of Meteorological Research}, 28\penalty0 (1):\penalty0
  34--56, 2014.

\bibitem[Yu et~al.(2009)Yu, Cunningham, Santhanam, Ryu, Shenoy, and
  Sahani]{Yu:2009:Gaussian-Process_Factor_Analysis_for_Low-Dimensional}
Yu, B.~M., Cunningham, J.~P., Santhanam, G., Ryu, S.~I., Shenoy, K.~V., and
  Sahani, M.
\newblock {Gaussian-Process} factor analysis for low-dimensional single-trial
  analysis of neural population activity.
\newblock In \emph{Advances in Neural Information Processing Systems},
  volume~21, pp.\  1881--1888. Curran Associates, Inc., 2009.

\bibitem[Zhang et~al.(1995)Zhang, Begleiter, Porjesz, Wang, and
  Litke]{Zhang:1995:Event_Related_Potentials_During_Object}
Zhang, X., Begleiter, H., Porjesz, B., Wang, W., and Litke, A.
\newblock Event related potentials during object recognition tasks.
\newblock \emph{Brain Research Bulletin}, 38\penalty0 (6):\penalty0 531--538,
  1995.

\bibitem[Zhe et~al.(2019)Zhe, Xing, and
  Kirby]{Zhe:2019:Scalable_High-Order_Gaussian_Process_Regression}
Zhe, S., Xing, W., and Kirby, R.~M.
\newblock Scalable high-order {Gaussian} process regression.
\newblock In Chaudhuri, K. and Sugiyama, M. (eds.), \emph{Proceedings of
  Machine Learning Research}, volume~89 of \emph{Proceedings of Machine
  Learning Research}, pp.\  2611--2620. PMLR, 4 2019.

\end{thebibliography}

%%%
%   APPENDIX (to be cropped off before submission)
%%%

\appendix

\newcommand{\nipstitle}[1]{{%
    \phantomsection\hsize\textwidth\linewidth\hsize%
    \vskip 0.1in%
    \toptitlebar%
    \begin{minipage}{\textwidth}%
        \centering{\Large\bf #1\par}%
    \end{minipage}%
    \bottomtitlebar%
    \addcontentsline{toc}{section}{#1}%
}}

\clearpage
\normalsize\onecolumn

% Reset section numbering.
\setcounter{section}{0}

\nipstitle{
    {Supplementary Material:} \\
    Scalable Exact Inference in Multi-Output Gaussian Processes
}
\pagestyle{empty}

\section*{Table of Contents}
\vspace*{5pt}
\newcommand{\extraspace}{2pt}
\begin{tabularx}{\linewidth}{@{\hspace{-1pt}}lL@{\hspace{2pt}}r}
    \ref{app:implementation} & How to Implement the OILMM \dotfill & \pageref{app:implementation} \\[\extraspace]
    \ref{app:mmh} & Unifying Presentation of Multi-Output Gaussian Processes \dotfill & \pageref{app:mmh} \\[\extraspace]
    \ref{app:complexities} & Runtime and Memory Complexities \dotfill & \pageref{app:complexities} \\[\extraspace]
    \ref{app:mle} & Maximum Likelihood Estimate \dotfill & \pageref{app:mle} \\[\extraspace]
    \ref{app:sufficiency} & Sufficient Statistic \dotfill & \pageref{app:sufficiency} \\[\extraspace]
    \ref{app:proof} & Proof of Prop.\ \ref{prop:general_sufficiency} \dotfill & \pageref{app:proof} \\[\extraspace]
    \ref{app:interpretation-likelihood} & Interpretation of the Likelihood \dotfill & \pageref{app:interpretation-likelihood} \\[\extraspace]
    \ref{app:tensor} & Tensor Product Basis \dotfill & \pageref{app:tensor} \\[\extraspace]
    \ref{app:cost_u} & Cost of Parametrising the Basis \dotfill & \pageref{app:cost_u} \\[\extraspace]
    \ref{app:diagonal-noise} & Characterisation of Diagonal Projected Noise \dotfill & \pageref{app:diagonal-noise} \\[\extraspace]
    \ref{app:kl} & Kullback--Leibler Divergence Between and ILMM and OILMM \dotfill & \pageref{app:kl} \\[\extraspace]
    \ref{app:oilmm-proj-and-noise} & OILMM: Projection and Projected Noise \dotfill & \pageref{app:oilmm-proj-and-noise} \\[\extraspace]
    \ref{app:oilmm-likelihood} & OILMM: Likelihood \dotfill & \pageref{app:oilmm-likelihood} \\[\extraspace]
    \ref{app:oilmm-missing-data} & OILMM: Missing Data \dotfill & \pageref{app:oilmm-missing-data} \\[\extraspace]
    \ref{app:oilmm-het-obs-noise} & OILMM: Heterogenous Observation Noise \dotfill & \pageref{app:oilmm-het-obs-noise} \\[\extraspace]
    \ref{app:scaling} & Computational Scaling Experiment (Sec.\ \ref{exp:scaling}) Additional Details \dotfill & \pageref{app:scaling} \\[\extraspace]
    \ref{app:point-process} & Point Process Experiment (Sec.\ \ref{exp:rainforest}) Additional Details and Analysis \dotfill & \pageref{app:point-process} \\[\extraspace]
    \ref{app:temp} & Temperature Extrapolation Experiment (Sec.\ \ref{exp:temp_extrapolation}) Additional Results \dotfill & \pageref{app:temp} \\[\extraspace]
    \ref{app:climate} & Large-Scale Climate Model Calibration Experiment (Sec.\ \ref{exp:climate}) Additional Details and Analysis \dotfill & \pageref{app:climate}
\end{tabularx}

\section*{Notation}
\vspace*{5pt}
\begin{tabularx}{\linewidth}{@{}lL@{}}
    $\lra{\vardot, \vardot}$ & Euclidean inner product \\
    $\norm{\vardot}$ & Euclidean norm \\
    $\norm{\vardot}\ss{op}$ & Operator norm \\
    $\norm{\vardot}_\infty$ & Supremum norm \\
    $\norm{\vardot}_F$ & Frobenius norm \\[.5em]
    $S^\perp$ & Orthogonal complement of $S$  \\[.5em]
    $I_n$ & $n \times n$ identity matrix \\
    $A > 0$ & $A$ is strictly-positive definite \\
    $|A|$ & Determinant of $A$ \\
    $A^\dagger$ & Moore--Penrose pseudo-inverse of $A$ \\
    $\col(A)$ & Column space of $A$ \\
    $A \otimes B$ & Kronecker product of $A$ and $B$ \\[.5em]
    $\Normal(x\cond \mu, \Sigma)$ & Density of the multivariate normal distribution with mean $\mu$ and covariance $\Sigma$ at $x$
\end{tabularx}

\section*{Assumptions}
Throughout the appendix, we assume that the columns of $H$ are linearly independent and that $\Sigma > 0$.
As a consequence, $H^\T \Sigma^{-1} H > 0$.
\vspace*{-1cm}
 
\clearpage

\section{How to Implement the OILMM}
\label{app:implementation}

\subsection{Parameters}
The parameters of the OILMM are as follows:

\begin{tabular}{lll}
    Symbol & Type & Description \\ \midrule
    $U$
        & Truncated orthogonal $p \times m$ matrix 
        & Orthogonal part of the basis $H = U S^{\frac12}$  \\
    $S$
        & Positive, diagonal $m \times m$ matrix
        & Diagonal part of the basis $H = U S^{\frac12}$ \\
    $\sigma^2$
        & Positive scalar
        & Part of the observation noise \\
    $D$
        & Positive, diagonal $m \times m$ matrix
        & Part of the observation noise deriving from the latent processes \\
    $(\th_i)_{i=1}^m$
        & Hyperparameters
        & Hyperparameters for the latent processes, \eg\ kernel parameters
\end{tabular}

\subsection{Inference}
\label{app:subsec:inference}
Inference in the OILMM proceeds in three steps.
Let $Y \in \R^{p \times n}$ be a matrix where the columns correspond to observations.

\paragraph{Projection step.}
In the projection step, we project the data to generate ``observations for the latent processes''.
We denote these observations by $Y\ss{proj} \in \R^{m \times n}$, where again the columns corresponds to observations.
We also construct the ``projected noise'', which is the observation noise under which the latent processes perform their observations.
\begin{enumerate}[topsep=0pt,itemsep=2pt]
    \item
        Construct the projection:
        \[
            T = S^{-\frac12} U^\T \in \R^{m \times p}.
        \]
    \item
        Project the observations:
        \[
            Y\ss{proj} = T Y \in \R^{m \times n}.
        \]
    \item
        Construct the projected noise:
        \[
            \Sigma_{T} = \sigma^2 S^{-1} + D \in \R^{m \times m}\ss{diag}.
        \]
        This is a diagonal matrix.
\end{enumerate}
The $i$\textsuperscript{th} row of $Y\ss{proj}$, which we denote by $y^{(i)}\ss{proj} \in \R^{n}$, corresponds to observations for latent process $i$.

\paragraph{Projection step (missing data).}
In the case of missing data, certain elements of $Y$ are missing.
Partition the columns (time stamps) of $Y$ into blocks $Y^{(1)} \in \R^{p \times n_1 },\ldots, Y^{(k)} \in \R^{p \times n_k}$ where $n_1 + \ldots + n_k = n$.
These blocks should be chosen such that, for every block $Y^{(i)}$, the observations for an output are either all missing or all available, i.e.\ every row of $Y^{(i)}$ is either entirely missing or entirely available.
Then consider the blocks $Y^{(1)} \in \R^{p \times n_1},\ldots, Y^{(k)} \in \R^{p \times n_k}$ separately by repeatedly performing inference.

For every block---we henceforth suppress the dependence on the block index---denote by $Y\obs \in \R^{p \times n}$ be the rows of the data matrix corresponding to observed outputs.
Similarly, let $U\obs \in \R^{p \times m}$ be the rows of $U$ corresponding to observed outputs.
\begin{enumerate}[topsep=0pt,itemsep=2pt]
    \item
        Construct the projection:
        \[
            T = S^{-\frac12} (U\obs^\T U\obs)^{-1}U\obs^\T \in \R^{m \times p}.
        \]
    \item
        Project the observations:
        \[
            Y\ss{proj} = T Y\obs \in \R^{m \times n}.
        \]
    \item
        Construct the projected noise:
        \[
            \Sigma_{T} = \sigma^{2} S^{-\frac12} d[({U\obs}^\T U\obs)^{-1}] S^{-\frac12} + D \in \R^{m\times m}\ss{diag}
        \]
        where $d[A]$ sets the off-diagonal elements of $A$ to zero.
        This is a diagonal matrix.
\end{enumerate}

\paragraph{Latent process inference step.}
In this step, we perform inference on the latent processes.

\begin{enumerate}[topsep=0pt,itemsep=2pt]
    \item
        For $i = 1,\ldots,m$, do the following:
        \begin{enumerate}[topsep=0pt,itemsep=2pt,leftmargin=6em]
            \item[Conditioning:]
                Condition latent process $i$ on data $y^{(i)}\ss{proj} \in \R^n$ where the observation noise is $(\Sigma_{T})_{ii}$.
                The latent process is just an independent GP, and any GP package can be used to do this step.
                Moreover, any single-output scaling technique can be used here, such as the variational inducing point approximation by \citet{Titsias:2009:Variational_Learning}.
            \item[Prediction:]
                Make predictions with the posterior of latent process $i$.
                Again, any GP package can be used to do this step.
                Denote the predictive means by $\mu^{(i)} \in \R^n$ and the predictive marginal variances by $\nu^{(i)} \in \R^n$.
        \end{enumerate}
    \item
        Collect the predictive means and marginal variances of the latent processes into matrices $\mu$ and $\nu$:
        \[
            \mu = \begin{bmatrix}
                (\mu^{(1)})^\T \\ \vdots \\ (\mu^{(m)})^\T
            \end{bmatrix} \in \R^{m \times n}, \quad
            \nu = \begin{bmatrix}
                (\nu^{(1)})^\T \\ \vdots \\ (\nu^{(m)})^\T
            \end{bmatrix} \in \R^{m \times n}.
        \]
\end{enumerate}

\paragraph{Reconstruction step.}
In the reconstruction step, we construct the predictions of the OILMM from the predictions of the latent processes.
\begin{enumerate}[topsep=0pt,itemsep=2pt]
    \item
        Construct the basis: $H = U S^{\frac12} \in \R^{p \times m}$.
    \item
        Construct the predictive mean of the OILMM:
        \[
            \text{predictive mean} = H \mu \in \R^{p \times n}.
        \]
    \item
        Construct the predictive marginal variances of the OILMM:
        \[
            \text{predictive marginal variances} = (H \circ H) \nu \in \R^{p \times n}
        \]
        where $\circ$ denotes the Hadamard product.
\end{enumerate}

\subsection{Posterior Sampling}
\label{app:subsec:sampling}
Instead of computing posterior means and marginal variances, you might want to generate posterior samples.

\paragraph{Projection step.}
See \cref{app:subsec:inference}.

\paragraph{Latent process sampling step.}
\begin{enumerate}[topsep=0pt,itemsep=2pt]
    \item
        For $i = 1,\ldots,m$, do the following:
        \begin{enumerate}[topsep=0pt,itemsep=2pt,leftmargin=6em]
            \item[Conditioning:]
                Condition latent process $i$ on data $y^{(i)}\ss{proj} \in \R^n$ where the observation noise is $(\Sigma_{T})_{ii}$.
                The latent process is just an independent GP, and any GP package can be used to do this step.
                Moreover, any single-output scaling technique can be used here, such as the variational inducing point approximation by \citet{Titsias:2009:Variational_Learning}.
            \item[Sampling:]
                Sample from the posterior of latent process $i$.
                Again, any GP package can be used to do this step.
                Denote the sample by $\hat x^{(i)} \in \R^n$.
        \end{enumerate}
    \item
        Collect the samples into a matrix:
        \[
            \hat X = \begin{bmatrix}
                (x^{(1)})^\T \\ \vdots \\ (x^{(m)})^\T
            \end{bmatrix} \in \R^{m \times n}.
        \]
\end{enumerate}

\paragraph{Reconstruction step.}
\begin{enumerate}[topsep=0pt,itemsep=2pt]
    \item
        Construct the basis: $H = U S^{\frac12} \in \R^{p \times m}$.
    \item
        Construct the posterior sample for the OILMM:
        \[
            \text{posterior sample} = H \hat X \in \R^{p \times n}.
        \]
\end{enumerate}

\subsection{Computation of the Log-Marginal Likelihood}
\label{app:subsec:likelihood}

\paragraph{Projection step.}
See \cref{app:subsec:inference}.

\paragraph{Latent process marginal likelihood calculation.}
\begin{enumerate}[topsep=0pt,itemsep=2pt]
    \item
        For $i = 1,\ldots,m$, do the following:
        \begin{enumerate}[topsep=0pt,itemsep=2pt,leftmargin=9em]
            \item[Marginal likelihood:]
                Compute the log-probability of data $y^{(i)}\ss{proj} \in \R^n$ under latent process $i$ where the observation noise is $(\Sigma_{T})_{ii}$.
                Denote the resulting log-probability by $\text{LML}_i$.
                The latent process is just an independent GP, and any GP package can be used to do this step.
                Moreover, any single-output scaling technique can be used here, such as the variational inducing point approximation by \citet{Titsias:2009:Variational_Learning}.
        \end{enumerate}
\end{enumerate}

\paragraph{Reconstruction step.}
\begin{enumerate}[topsep=0pt,itemsep=2pt]
    \item
        Construct the ``regularisation term'':
        \[
            \text{regulariser}
            =
                - \frac{n}{2} \log |S|
                - \frac{n (p-m)}{2} \log 2 \pi \sigma^2
                - \frac{1}{2\sigma^2} (
                    \norm{Y}_F^2
                    -\norm{U^\T Y}_F^2
                )
        \]
        where $\norm{\vardot}_F$ denotes the Frobenius norm.
    \item
        Construct the log-probability of the data $Y$ under the OILMM:
        \[
            \log p(Y) = \text{regulariser} + \sum_{i=1}^m \text{LML}_i.
        \]
\end{enumerate}

\paragraph{Reconstruction step (missing data).}
In the case of missing data, (1) is slightly different:
\begin{enumerate}[topsep=0pt,itemsep=2pt]
    \item
        Construct the ``regularisation term'':
        \[
            \text{regulariser}
            =
                - \frac{n}{2} \log |S|
                - \frac{n}{2} \log |U\obs^\T U\obs|
                - \frac{n (p-m)}{2} \log 2 \pi \sigma^2
                - \frac{1}{2\sigma^2} (
                    \norm{Y\obs}_F^2
                    - \norm{\operatorname{chol}(U\obs^\T U\obs)^{-1}U\obs^\T Y\obs}_F^2
                )
        \]
        where $\norm{\vardot}_F$ denotes the Frobenius norm and $\operatorname{chol}(\vardot)$ the Cholesky decomposition.
        In this case, recall that $n$ corresponds to the number of time points in the current block and to $p$ to the number of observed outputs in the current block.
\end{enumerate}

\begin{figure*}[ph!]
    \footnotesize
    \hfill
    \begin{tikzpicture}[
        xscale=0.85,
        yscale=0.85,
        latent node/.append style={
            align=center,
            inner sep=0pt,
            text width=0.75cm
        }
    ]
        \node [latent node] (xt) at (0, 0) {$x_n$};
        \node [latent node] (ft) at ($(xt) + (0, 1.5)$) {$f_n$};
        \node [latent node] (xt2) at ($(xt) + (1, 0)$) {$x_{n+1}$};
        \node [latent node] (ft2) at ($(xt2) + (0, 1.5)$) {$f_{n+1}$};
        \node [latent node, dashed] (xtpre) at ($(xt) - (1, 0)$) {};
        \node [latent node, dashed] (xtpost) at ($(xt2) + (1, 0)$) {};
        \node [latent node, dashed] (ftpre) at ($(xtpre) + (0, 1.5)$) {};
        \node [latent node, dashed] (ftpost) at ($(xtpost) + (0, 1.5)$) {};
        \draw [arrow, ->] (xt) -- (ft) node [pos=.45, fill=white] {$H$};
        \draw [arrow, ->] (xt2) -- (ft2) node [pos=.45, fill=white] {$H$};
        \draw [arrow, ->] (xt) -- (xt2);
        \draw [arrow, ->, dashed] (xtpre) -- (xt);
        \draw [arrow, ->, dashed] (xt2) -- (xtpost);
        \node [anchor=south, rotate=90, draw=black, thick] at (-1.5, 0.9) {instantaneous mixing};
        \node [anchor=south, draw=black, thick] at (0.5, 2.3) {time invariant};
        \node [anchor=south] (desc) at (0.5, 1.85) {$\vphantom{\text{y}} H\in \R^{p \times m}$ is a fixed matrix};
        \node [plate,
               fit=(xtpre) (desc) (ftpost),
               inner sep=5pt,
               inner ysep=3pt,
               yshift=-2pt,
               label={[anchor=south west]north west:ILMM}] () {};
    \end{tikzpicture}
    \hfill
    \begin{tikzpicture}[
        xscale=0.85,
        yscale=0.85,
        latent node/.append style={
            align=center,
            inner sep=0pt,
            text width=0.75cm
        }
    ]
        \node [latent node] (xt) at (0, 0) {$x_n$};
        \node [latent node] (ft) at ($(xt) + (0, 1.5)$) {$f_n$};
        \node [latent node] (xt2) at ($(xt) + (1, 0)$) {$x_{n+1}$};
        \node [latent node] (ft2) at ($(xt2) + (0, 1.5)$) {$f_{n+1}$};
        \node [latent node, dashed] (xtpre) at ($(xt) - (1, 0)$) {};
        \node [latent node, dashed] (xtpost) at ($(xt2) + (1, 0)$) {};
        \node [latent node, dashed] (ftpre) at ($(xtpre) + (0, 1.5)$) {};
        \node [latent node, dashed] (ftpost) at ($(xtpost) + (0, 1.5)$) {};
        \draw [arrow, ->] (xt) -- (ft) node [pos=.45, fill=white, fill opacity=.85] {$H_n$};
        \draw [arrow, ->] (xt2) -- (ft2) node [pos=.45, fill=white, fill opacity=.85] {$H_{n+1}$};
        \draw [arrow, ->] (xt) -- (xt2);
        \draw [arrow, ->, dashed] (xtpre) -- (xt);
        \draw [arrow, ->, dashed] (xt2) -- (xtpost);
        \node [anchor=south, draw=black, thick] at (0.5, 2.3) {time-varying};
        \node [anchor=south] at (0.5, 1.85) {$H\colon \Tc \to \R^{p \times m}$ is a time-varying matrix};
        \node [plate,
               fit=(xtpre) (desc) (ftpost),
               inner sep=5pt,
               inner ysep=3pt,
               yshift=-2pt,
               draw=none] () {};
    \end{tikzpicture}
    \hfill\strut
    \\[0em]
    \strut\hfill
    \begin{tikzpicture}[
        xscale=0.85,
        yscale=0.85,
        latent node/.append style={
            align=center,
            inner sep=0pt,
            text width=0.75cm
        }
    ]
        \node [latent node] (xt) at (0, 0) {$x_n$};
        \node [latent node] (ft) at ($(xt) + (0, 1.5)$) {$f_n$};
        \node [latent node] (xt2) at ($(xt) + (1, 0)$) {$x_{n+1}$};
        \node [latent node] (ft2) at ($(xt2) + (0, 1.5)$) {$f_{n+1}$};
        \node [latent node, dashed] (xtpre) at ($(xt) - (1, 0)$) {};
        \node [latent node, dashed] (xtpost) at ($(xt2) + (1, 0)$) {};
        \node [latent node, dashed] (ftpre) at ($(xtpre) + (0, 1.5)$) {};
        \node [latent node, dashed] (ftpost) at ($(xtpost) + (0, 1.5)$) {};
        \draw [arrow, ->, dashed] (xtpre) -- (xt);
        \draw [arrow, ->, dashed] (xt2) -- (xtpost);
        \draw [arrow, ->, dashed] (xt) -- (ftpre);
        \draw [arrow, ->, dashed] (xt) -- (ftpost);
        \draw [arrow, ->, dashed] (xt2) -- (ftpre);
        \draw [arrow, ->, dashed] (xt2) -- (ftpost);
        \draw [arrow, ->] (xt) -- (ft) node [pos=0.45, fill=white, fill opacity=.85, inner sep=1pt] {$H_0$};
        \draw [arrow, ->] (xt) -- (ft2) node [pos=0.2, fill=white, fill opacity=.85, inner sep=1pt] {$H_{-1}$};
        \draw [arrow, ->] (xt2) -- (ft) node [pos=0.7, fill=white, fill opacity=.85, inner sep=1pt] {$H_{1}$};
        \draw [arrow, ->] (xt2) -- (ft2) node [pos=0.45, fill=white, fill opacity=.85, inner sep=1pt] {$H_0$};
        \draw [arrow, ->] (xt) -- (xt2);
        \node [anchor=south, rotate=90, draw=black, thick] at (-1.5, 0.9) {convolutional mixing};
        \node [anchor=south] at (0.5, 1.85) {$\vphantom{\text{y}} H\colon \Tc \to \R^{p \times m}$ is a time-invariant filter};
        \node [plate,
              fit=(xtpre) (desc) (ftpost),
              inner sep=5pt,
              inner ysep=3pt,
              yshift=-2pt,
              draw=none] () {};
    \end{tikzpicture}
    \hfill
    \begin{tikzpicture}[
        xscale=0.85,
        yscale=0.85,
        latent node/.append style={
            align=center,
            inner sep=0pt,
            text width=0.75cm
        }
    ]
        \node [latent node] (xt) at (0, 0) {$x_n$};
        \node [latent node] (ft) at ($(xt) + (0, 1.5)$) {$f_n$};
        \node [latent node] (xt2) at ($(xt) + (1, 0)$) {$x_{n+1}$};
        \node [latent node] (ft2) at ($(xt2) + (0, 1.5)$) {$f_{n+1}$};
        \node [latent node, dashed] (xtpre) at ($(xt) - (1, 0)$) {};
        \node [latent node, dashed] (xtpost) at ($(xt2) + (1, 0)$) {};
        \node [latent node, dashed] (ftpre) at ($(xtpre) + (0, 1.5)$) {};
        \node [latent node, dashed] (ftpost) at ($(xtpost) + (0, 1.5)$) {};
        \draw [arrow, ->, dashed] (xtpre) -- (xt);
        \draw [arrow, ->, dashed] (xt2) -- (xtpost);
        \draw [arrow, ->, dashed] (xt) -- (ftpre);
        \draw [arrow, ->, dashed] (xt) -- (ftpost);
        \draw [arrow, ->, dashed] (xt2) -- (ftpre);
        \draw [arrow, ->, dashed] (xt2) -- (ftpost);
        \draw [arrow, ->] (xt) -- (ft) node [pos=0.45, fill=white, fill opacity=.85, inner sep=1pt] {$H_n^n$};
        \draw [arrow, ->] (xt) -- (ft2) node [pos=0.2, fill=white, fill opacity=.85, inner sep=1pt] {$H_n^{n+1}$};
        \draw [arrow, ->] (xt2) -- (ft) node [pos=0.7, fill=white, fill opacity=.85, inner sep=1pt] {$H_{n+1}^n$};
        \draw [arrow, ->] (xt2) -- (ft2) node [pos=0.45, fill=white, fill opacity=.85, inner sep=1pt] {$H_{n+1}^{n+1}$};
        \draw [arrow, ->] (xt) -- (xt2);
        \node [anchor=south] at (0.5, 1.85) {$H\colon \Tc\times\Tc \to \R^{p \times m}$ is a time-varying filter};
        \node [plate,
               fit=(xtpre) (desc) (ftpost),
               inner sep=5pt,
               inner ysep=3pt,
               yshift=-2pt,
               draw=none] () {};
    \end{tikzpicture}
    \hfill\strut\\
    \vspace{-2em}
    \caption[Graphical models illustrating the difference between time-invariant/time-varying and instantaneous/convolutional MOGPs]{
        Graphical models illustrating the difference between time-invariant/time-varying and instantaneous/convolutional multi-output GP models, for data sampled at real-valued times $t_1,t_2,\ldots$ (sampling period $\Delta t$).
        Abbreviations used:
        $x_n=x(t_n)$,
        $f_n=f(t_n)$,
        $H_n=H(n\Delta t)$, and
        $H_n^m=H(t_m, t_n)$.
        For simplicity, the dynamics of $x$ are depicted as a Markov chain;
        since $x$ is modelled with a GP, $x_n$ actually depends on $x_{n'}$ for all $n' \le n$.
    }
    \label{fig:mixing_model_hierarchy_graphical_model}
\end{figure*}
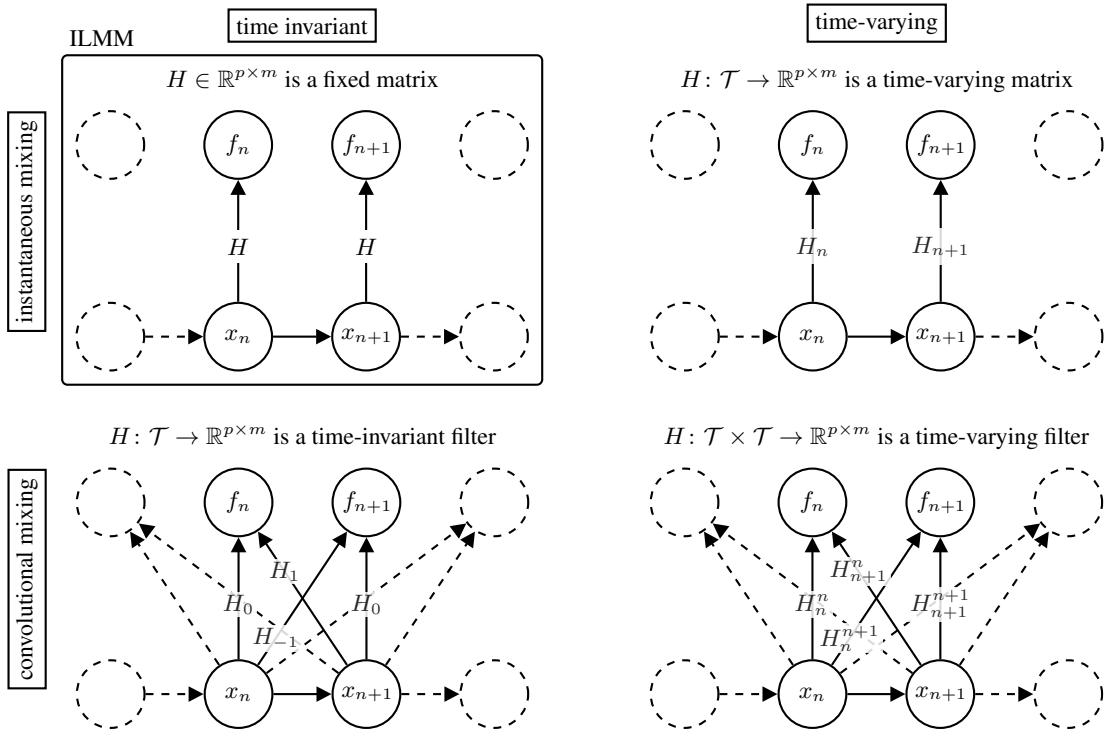

\begin{figure*}[ph!]
    \centering
    \begin{tikzpicture}[
        scale=2.0,
        labelbg/.style={
               fill=white,
               inner sep=0pt,
               opacity=0.7,
               text opacity=1
        }
    ]
        \draw [thick, draw=black!30]
            (0, 0, -1) coordinate (backsw)
            -- ++(0.8, 0, 0) coordinate (backse);
        \draw [thick]
            (backse)
            -- ++(0, 0.8, 0) coordinate (backne)
            -- ++(-0.8, 0, 0) coordinate (backnw);
        \draw [thick, draw=black!30]
            (backnw) -- (backsw);
        \draw [thick]
            (0, 0, 0) coordinate (frontsw)
                -- ++(1, 0, 0) coordinate (frontse)
                -- ++(0, 1, 0) coordinate (frontne)
                -- ++(-1, 0, 0) coordinate (frontnw)
                -- cycle;
        \draw [thick, draw=black!30]
            (frontsw) -- (backsw);
        \draw [thick]
            (frontse) -- (backse)
            (frontnw) -- (backnw)
            (frontne) -- (backne);
        \draw [arrow]
            ($(frontsw) + (0, -0.07, 0)$)
            -- ($(frontse) + (0, -0.07, 0)$)
            node [pos=0.475,
                  fill=white,
                  inner sep=.5pt,
                  ] {\small convolutional};
        \draw [arrow]
            ($(frontsw) + (-0.07, 0.07, 0)$)
            -- ($(frontnw) + (-0.07, -0.07, 0)$)
            node [pos=0.475,
                  rotate=90,
                  fill=white,
                  inner sep=.5pt,
                  ] {\small time varying};
        \draw [arrow]
            ($(frontnw) + (-0.1, 0, -0.17)$)
            -- ($(backnw) + (-0.1, 0.025, 0)$)
            node [pos=0.475,
                  rotate=22,
                  align=center,
                  anchor=south
                  ] {\small prior \\[-.35em] on $H$};
        % Labels:
        \node [anchor=east]
            (labelfrontsw) at ($(frontsw) + (205:0.25)$)
            {$f(t)=H x(t)$};
        \node [anchor=east]
            (labelfrontnw) at ($(frontnw) + (190:0.25)$)
            {$f(t)=H(t) x(t)$};
        \node [anchor=west, labelbg]
            (labelfrontne) at ($(frontne) + (-10:0.25)$)
            {$\displaystyle f(t)=\int H(t, \tau) x(\tau)\isd{\tau}$};
        \node [anchor=west]
            (labelfrontse) at ($(frontse) + (-25:0.25)$)
            {$\displaystyle f(t)=\int H(t - \tau) x(\tau)\isd{\tau}$};
        \draw [dashed] (labelfrontsw.east) -- (frontsw);
        \draw [dashed] (labelfrontse.west) -- (frontse);
        \draw [dashed] (labelfrontnw.east) -- (frontnw);
        \draw [dashed] (labelfrontne.west) -- (frontne);
        % References:
        \node [anchor=south west,
               align=left,
               labelbg,
               xshift=3pt,
               yshift=4pt] at (frontsw) {[1--3, 5--7, 9, \\\phantom{[}13, 14, 19--21]}; % ILMM
        \node [anchor=south west,
               yshift=3pt,
               xshift=4pt,
               labelbg] at (frontse) {[4, 8, 10, 11]}; % Conv
        \node [anchor=south west] at (backnw) {[12, 17]};  % ILMM + TV + prior
        \node [anchor=south west, color=black] at (backse) {[15]};  % Conv + prior
        \node [anchor=south west] at (backsw) {[16, 18]}; % ILMM + prior
        % Forms of H:
        \node [anchor=north, align=center, text width=\linewidth] (Hs) at (0.5, -0.35, 0) {
            \small
            \begin{tabular}{llll}
                \toprule
                & Form of $H$ & Form of $K$ & Mixing  \\ \midrule
                {}[1, 5, 6, 16] % ICM
                & $H$
                & $k(t,t') I$
                & Instantaneous \\
                {}[2] % LMC
                & $\big[ H_1 \; \cdots \; H_q \big]$
                & $\diag(k_1(t,t') I, \ldots, k_q(t,t') I)$
                & Instantaneous \\
                {}[3, 7, 9, 13, 20, 21] % SLFM
                & $H$
                & $\diag(k_1(t,t'), \ldots, k_q(t,t'))$
                & Instantaneous \\
                {}[4, 10, 11, 15] % Dependent GPs
                & $H(t - t')$
                & $\diag(\delta(t-t'), \ldots, \delta(t-t'))$
                & Convolutional \\
                {}[8] % LFM
                & Green's function
                & $\diag(k_1(t,t'), \ldots, k_q(t,t'))$
                & Convolutional \\
                {}[12, 17] % GPRN
                & $H(t)$
                & $\diag(k_1(t,t'), \ldots, k_q(t,t'))$
                & Instantaneous \\
                {}[14] % Collaborative GPs
                & $\big[ H  \; I \big]$
                & $\diag(k_1(t,t'), \ldots, k_{q+p}(t,t'))$
                & Instantaneous \\
                {}[18] % GPAR
                & Lower triangular
                & $\diag(k_1(t,t'), \ldots, k_{q}(t,t'))$
                & Instantaneous \\
                {}[19] % HOGPR
                & $H_1 \otimes \cdots \otimes H_q$ 
                & $k(t,t') I$
                & Instantaneous \\
                \bottomrule
            \end{tabular}
        };
        % Citations:
        \node [anchor=north, align=left, text width=.8\linewidth]
            at ($(Hs.south) + (-90:0.1)$)
            {\setstretch{1.0}\footnotesize%
            \begin{tabularx}{\linewidth}{lL}
            {}[1] & Intrinstic Coregionalisation Model \parencite{Goovaerts:1997:Geostatistics_for_Natural_Resources_Evaluation} \\
            {}[2] & Linear Model of Coregionalisation \parencite{Goovaerts:1997:Geostatistics_for_Natural_Resources_Evaluation} \\
            {}[3] & Semiparametric Latent Factor Model \parencite{Teh:2005:Semiparametric_Latent_Factor} \\
            {}[4] & Dependent Gaussian Processes \parencite{Boyle:2005:Dependent_Gaussian_Processes} \\
            {}[5] & Multi-Task Gaussian Processes \parencite{Bonilla:2008:Multi-Task_Gaussian_Process} \\
            {}[6] & \textcite{Osborne:2008:Towards_Real-Time_Information_Processing_of} \\
            {}[7] & \textcite{Higdon:2008:Computer_Model_Calibration_Using_High-Dimensional} \\
            {}[8] & Latent Force Models \parencite{Alvarez:2009:Latent_Force_Models} \\
            {}[9] & Gaussian Process Factor Analysis \parencite{Yu:2009:Gaussian-Process_Factor_Analysis_for_Low-Dimensional} \\
            {}[10] & Multi-Output {Gaussian} Processes Through Variational Inducing Kernels \parencite{Alvarez:2010:Efficient_Multioutput_Gaussian_Processes_Through} \\
            {}[11] & Convolved Multiple Output Gaussian Processes \parencite{Alvarez:2011:Computationally_Efficient_Convolved} \\
            {}[12] & Gaussian Process Regression Network \parencite{Wilson:2012:GP_Regression_Networks} \\
            {}[13] & Spatio--Temporal Bayesian Filtering and Smoothing \parencite{Sarkka:2013:Spatiotemporal_Learning_via} \\
            {}[14] & Collaborative Multi-Output Gaussian Processes \parencite{Nguyen:2014:Collaborative_Multi-Output} \\
            {}[15] & Generalised Gaussian Process Convolution Model \parencite{Bruinsma:2016:GGPCM} \\
            {}[16] & Semi-Parametric Network Structure Discovery Models \parencite{Dezfouli:2017:Semi-Parametric_Network_Structure_Discovery_Models} \\
            {}[17] & Grouped Gaussian Processes \parencite{Dahl:2018:Grouped_Gaussian_Processes_for_Solar} \\
            {}[18] & The Gaussian Process Autoregressive Regression Model \parencite{Requeima:2019:The_Gaussian_Process_Autoregressive_Regression} \\
            {}[19] & High-Order Gaussian Process Regression \parencite{Zhe:2019:Scalable_High-Order_Gaussian_Process_Regression} \\
            {}[20] & Instantaneous Linear Mixing Model (\cref{mod:ILMM}) \\
            {}[21] & Orthogonal Instantaneous Linear Mixing Model (\cref{mod:OILMM})
            \end{tabularx}
            \par
        };
    \end{tikzpicture}
    \caption[The Mixing Model Hierarchy]{
        The Mixing Model Hierarchy, which organises MOGPs from the machine learning and geostatistics literature according to their distinctive modelling assumptions
    }
    \label{fig:mixing_model_hierarchy}
\end{figure*}

\section{Unifying Presentation of Multi-Output Gaussian Processes}
\label{app:mmh}

Our attempt at a unifying presentation of MOGP models starts from setting up what we call the {\em Mixing Model Hierarchy} (MMH).
At the bottom of the Mixing Model Hierarchy stands the Instantaneous Linear Mixing Model (ILMM, \cref{mod:ILMM} in \cref{sec:ILMM}), which is a simple, but general class of MOGP models typically characterised by low-rank covariance structure.

The graphical model of the ILMM is illustrated in the top-left corner of \cref{fig:mixing_model_hierarchy_graphical_model}, which highlights two restrictions of the ILMM compared to a general MOGP:
{\em (i)}~the \textit{instantaneous spatial covariance} of $f$, $\E[f(t) f^\T(t)] = H H^\T$, does not vary with time, because neither $H$ nor $K(t, t) = I_m$ vary with time; and
{\em (ii)}~the noise-free observation $f(t)$ is a function of $x(t')$ for $t'=t$ only, meaning that, for example, $f$ cannot be $x$ with a delayed or a smoothed version of $x$.
We hence call the ILMM a \emph{time-invariant} (due to \emph{(i)}) and \emph{instantaneous} (due to \emph{(ii)})  MOGP.

The ILMM can be generalised in three ways.
First, the mixing matrix $H$ may vary with time.
Then $H \in \R^{p \times m}$ becomes a matrix-valued function $H\colon \Tc \to \R^{p \times m}$, and the mixing mechanism becomes
\[
    f(t)\cond H, x = H(t) x(t).
\]
We call such MOGP models \emph{time-varying} (see \cref{fig:mixing_model_hierarchy_graphical_model}, top right).
Second, $f(t)$ may depend on $x(t')$ for all $t' \in \Tc$.
Then the mixing matrix $H \in \R^{p \times m}$ becomes a matrix-valued time-invariant filter $H\colon \Tc \to \R^{p \times m}$, and the mixing mechanism becomes
\[
    f(t)\cond H, x = \int H(t - \tau) x(\tau) \isd{\tau}.
\]
We call such MOGP models \emph{convolutional} (see \cref{fig:mixing_model_hierarchy_graphical_model}, bottom left).
Finally, $f(t)$ may depend on $x(t')$ for all $t' \in \Tc$ \emph{and} this relationship may vary with time.
Then the mixing matrix $H \in \R^{p \times m}$ becomes a matrix-valued time-varying filter $H\colon \Tc\times\Tc \to \R^{p \times m}$, and the mixing mechanism becomes
\[
    f(t)\cond H, x = \int H(t, \tau) x(\tau) \isd{\tau}.
\]
We call such MOGP models \emph{time-varying} and \emph{convolutional} (see \cref{fig:mixing_model_hierarchy_graphical_model}, bottom right).
The graphical models corresponding to these generalisations of the ILMM are depicted in \cref{fig:mixing_model_hierarchy_graphical_model}.

The ILMM can be extended in one other way, which is to include a prior distribution on $H$.
This extension and the two previously proposed generalisations together form the \emph{Mixing Model Hierarchy} (MMH), which is depicted in \cref{fig:mixing_model_hierarchy}.
The MMH organises multi-output Gaussian process models according to their distinctive modelling assumptions.
\cref{fig:mixing_model_hierarchy} shows how sixteen MOGP models from the machine learning and geostatistics literature can be recovered as special cases of the various generalisations of the ILMM.

Not all multi-output Gaussian process models are covered by the MMH, however.
For example, Deep GPs \parencite{Damianou:2015:Deep_Gaussian_Processes_and_Variational} and variations thereon \parencite{Kaiser:2017:Bayesian_Alignments_of_Warped_Multi-Output} are excluded because they transform the latent processes \textit{nonlinearly} to generate the observations.

\section{Runtime and Memory Complexities}
\label{app:complexities}

For the ILMM and OILMM, \cref{tab:complexities_no_projection} gives an overview of the runtime and memory complexities associated to learning and inference, and \cref{tab:complexities_projection} gives an overview of the runtime and memory complexities associated to projecting the data and reconstructing the predictions.

\begin{table}[t]
    \small
    \caption{
        Complexities of learning and inference in the ILMM and OILMM, ignoring the projection.
        In the table,
        $n$ is the number of time points;
        $p$ is the number of outputs;
        $m$ is the number of latent processes;
        $r$ is the number of inducing points, typically $r \ll n$;
        and $d$ is the state dimensionality, typically $d \ll n, m$.
    }
    \label{tab:complexities_no_projection}
    \vspace{1em}
    \centering
    \begin{tabular}{p{10cm}p{1.5cm}p{1.5cm}}
        \toprule
        Model & Runtime & Memory \\ \midrule
        General MOGP & $\O(n^3p^3)$ & $\O(n^2 p^2)$ \\
        ILMM (\cref{mod:ILMM}) & $\O(n^3 m^3)$ & $\O(n^2 m^2)$ \\
        OILMM (\cref{mod:OILMM}) & $\O(n^3 m)$ & $\O(n^2 m)$ \\
        OILMM (\cref{mod:OILMM}) + \citet{Titsias:2009:Variational_Learning} & $\O(n m r^2)$ & $\O(n m r)$ \\
        OILMM (\cref{mod:OILMM}) + \citet{Hartikainen:2010:Kalman_filtering_and_smoothing_solutions} & $\O(n m d^3)$ & $\O(n m d^2)$ \\[.5em]
        \multicolumn{3}{l}{
            \textsc{Application to separable spatio--temporal GPs (\cref{sec:spatio-temporal})}
        } \\[.2em]
        OILMM (\cref{mod:OILMM}) & $\O(n^3 p)$ & $\O(n^2 p)$ \\
        OILMM (\cref{mod:OILMM}) + \citet{Titsias:2009:Variational_Learning} & $\O(n p r^2)$ & $\O(n p r)$ \\
        OILMM (\cref{mod:OILMM}) + \citet{Hartikainen:2010:Kalman_filtering_and_smoothing_solutions} & $\O(n p d^3)$ & $\O(n p d^2)$ \\
        Kronecker product factorisation \citep[][Ch.~5]{saatcci2012scalable} & $\O(n^3 + p^3)$ & $\O(n^2 + p^2)$ \\
        \bottomrule
    \end{tabular}
\end{table}

\begin{table}[t]
    \small
    \caption{
        Complexities of projecting the data and reconstructing the predictions in the ILMM and OILMM.
        In the table,
        $n$ is the number of time points;
        $p$ is the number of outputs;
        and $m$ is the number of latent processes.
    }
    \label{tab:complexities_projection}
    \vspace{1em}
    \centering
    \begin{tabular}{p{10cm}p{1.5cm}p{1.5cm}}
        \toprule
        Action & Runtime & Memory \\ \midrule
        Storing data & $-$ & $\O(np)$ \\
        Construction of projection $T$ & $\O(m^2p)$ & $\O(mp)$ \\
        Projection & $\O(nmp)$ & $\O(np)$ \\
        Construction of predictive marginal statistics & $\O(nmp)$ & $\O(np)$ \\[.5em]
        \multicolumn{2}{l}{
            \textsc{Application to separable spatio--temporal GPs (\cref{sec:spatio-temporal})}
        } \\[.2em]
        Construction of projection $T$ & $\O(p^3)$ & $\O(p^2)$ \\
        Projection & $\O(np^2)$ & $\O(np)$ \\
        Construction of predictive marginal statistics & $\O(np^2)$ & $\O(np)$ \\
        \bottomrule
    \end{tabular}
\end{table}

\section{Maximum Likelihood Estimate}
\label{app:mle}

\begin{proposition}
    \label{prop:mle}
    Denote $p(y\cond x) = \Normal(y\cond Hx, \Sigma)$,
    and let $T$ be the $m\times p$ matrix $(H^\T \Sigma^{-1} H)^{-1} H^\T \Sigma^{-1}$.
    Then
    \[
        Ty = \argmax_{x}\, p(y\cond x)
    \]
    and $Ty$ is an unbiased estimate of $x$: $\E[Ty\cond x] = x$.
\end{proposition}
\begin{proof}
    Note that
    \[
        \log p(y\cond x) \simeq -\tfrac12 (y - Hx)^\T \Sigma^{-1}(y - Hx)
    \]
    Using invertibility of $H^\T \Sigma^{-1} H$, an elementary calculation then shows that the unique maximum with respect to $x$ is given by
    \[
        x = (H^\T \Sigma^{-1} H)^{-1} H^\T \Sigma^{-1} y = T y.
    \]
    To show that $Ty$ is an unbiased estimate of $x$, we use that $\E[y\cond x]=Hx$:
    \[
        \E[Ty\cond x] = T H x = (H^\T \Sigma^{-1} H)^{-1} (H^\T \Sigma^{-1} H) x = x. \qedhere
    \]
\end{proof}

\section{Sufficient Statistic}
\label{app:sufficiency}

To prove sufficiency of $Ty$ , we need the property of $T$ that it ``preserves the signal-to-noise ratio''.
This is characterised in the following lemma.

\begin{lemma} \label{lemma:signal-to-noise_ratio}
    \[
        \frac
            {\Normal(y\cond Hx, \Sigma)}
            {\Normal(y\cond 0, \Sigma)}
        = \frac
            {\Normal(Ty\cond x, (H^\T\Sigma^{-1}H)^{-1})}
            {\Normal(Ty\cond 0,(H^\T\Sigma^{-1}H)^{-1})}.
    \]
\end{lemma}
\begin{proof}
    It is simple to check the equality by direct verification.
    We show, however, how the equality may be derived.
    To begin with, we have
    \[
        (y - Hx)^\T \Sigma^{-1}(y - Hx)
        = y^\T \Sigma^{-1}y - 2 x^\T H^\T \Sigma^{-1} y + x^\T H^\T \Sigma^{-1} H x.
    \]
    Here
    \[
        H^\T \Sigma^{-1} y
        = (H^\T \Sigma^{-1} H)(H^\T \Sigma^{-1} H)^{-1}H^\T \Sigma^{-1}y
        = (H^\T \Sigma^{-1} H)T y,
    \]
    so
    \[
        (y - Hx)^\T \Sigma^{-1}(y - Hx)
        = y^\T \Sigma^{-1} y - 2 x^\T (H^\T \Sigma^{-1} H)T y + x^\T (H^\T \Sigma^{-1} H) x.
    \]
    Adding and subtracting $yT^\T (H^\T \Sigma^{-1}H) T y$, we find
    \[
        (y - Hx)^\T \Sigma^{-1}(y - Hx)
        = y^\T \Sigma^{-1}y - yT^\T (H^\T \Sigma^{-1}H) T y + (x - T y)^\T (H^\T \Sigma^{-1} H)(x - Ty).
    \]
    Hence, rearranging,
    \[
        (y - Hx)^\T \Sigma^{-1}(y - Hx) - y^\T \Sigma^{-1}y
        = (x - T y)^\T (H^\T \Sigma^{-1} H)(x - Ty) - yT^\T (H^\T \Sigma^{-1}H) T y,
    \]
    which yields the result.
\end{proof}

\begin{proposition} \label{prop:sufficiency}
    The MLE $Ty$ of $x$ is a minimal sufficient statistic for $x$. 
\end{proposition}

\begin{proof}[Proof of \cref{prop:sufficiency}]
    By a general characterisation of minimal sufficient statistics \citep[see, \eg, Th.~6.2.13 in][]{Casella:2001:Statistical_Inference}, $Ty$ is a minimal sufficient statistic for $x$ if and only if it is true that $p(y_1 \cond x) / p(y_2 \cond x)$ is constant as a function of $x$ if and only if $Ty_1 = Ty_2$.
    Indeed, by \cref{lemma:signal-to-noise_ratio},
    \[
        \log \frac{p(y_1 \cond x)}{p(y_2 \cond x)}
        = (T y_1 - T y_2)^\T (H^\T \Sigma^{-1} H)^{-1} x
        + \text{const.}
    \]
    which, by invertibility of $H^\T \Sigma^{-1} H$, does not depend on $x$ if and only if $T y_1 = T y_2$.
\end{proof}

\section{Proof of Prop.\ \ref{prop:general_sufficiency}}
\label{app:proof}

\begin{proof}[Proof of \cref{prop:general_sufficiency}]
    By \cref{prop:sufficiency},
    \[
        p(f \cond Y)
        = \int p(f \cond x) p(x \cond Y) \isd x
        = \int p(f \cond x) p(x \cond TY) \isd x
        = p(f \cond TY)
    \]
    where $TY$ are observations for the process $Ty$.
    Since 
    \[
        y \cond x \sim \GP(Hx, \delta[t - t']\Sigma),
    \]
    the process $Ty$ has distribution
    \[
        Ty \cond x \sim \GP(THx, \delta[t - t']T \Sigma T^\T).
    \]
    By explicit calculation, we find that
    \[
        TH = (H^\T \Sigma^{-1} H)^{-1}H^\T \Sigma^{-1} H = I
    \]
    and
    \[
        T \Sigma T^\T = (H^\T \Sigma^{-1} H)^{-1}H^\T \Sigma^{-1} \Sigma \Sigma^{-1} H (H^\T \Sigma^{-1} H)^{-1} = (H^\T \Sigma^{-1} H)^{-1}.
    \]
    Thus
    \[
        Ty \cond x \sim \GP(x, \delta[t - t'] \Sigma_T)
        \quad\text{where}\quad
        \Sigma_T = (H^\T \Sigma^{-1} H)^{-1}.
    \]
    Moreover, using \cref{lemma:signal-to-noise_ratio},
    the probability of the data $Y$ is given by
    \[
        p(Y)
        = \int \prod_{i=1}^n \Normal(y_i\cond Hx, \Sigma) p(x) \isd x
        = \sbrac*{
            \frac{\Normal(y_i\cond 0, \Sigma)}{\Normal(y_i\cond 0, \Sigma_T)}
        } \int \prod_{i=1}^n \Normal(Ty_i\cond x, \Sigma_T) p(x) \isd x. \qedhere
    \]
\end{proof}

\section{Interpretation of the Likelihood}
\label{app:interpretation-likelihood}
\begin{proposition} \label{prop:regularisation-term}
    The regularisation terms in like likelihood in \cref{prop:general_sufficiency} can be written as
    \[
        \log \frac{\Normal(y\cond 0, \Sigma)}{\Normal(T y\cond 0, \Sigma_T)}
        = 
        -\frac12 (p - m)\log 2\pi
        - \overbracket{
            \frac12\log \frac{|\Sigma|}{|\Sigma_T|}
        }^{\mathclap{
            \text{noise ``lost by projection''}
        }}
        - \underbracket{
            \frac12\norm{(I_p - HT)y}_{\Sigma}^2,
            \vphantom{\frac{|\Sigma|}{|\Sigma_T|}}
        }_{\mathclap{
            \text{data ``lost by projection''}
        }}
    \]
    where $\norm{\vardot}_\Sigma$ denotes the norm induced by the weighted inner product $\lra{\vardot, \vardot}_\Sigma = \lra{\Sigma^{-1}\vardot, \vardot}$.
\end{proposition}
\begin{proof}
    The first two terms come directly from the multivariate Gaussian densities.
    We show how the third term may be obtained.
    Rearrange
    \[
        \lra{y, T^\T\Sigma_T^{-1} Ty}
        = \lra{\Sigma^{-\frac12} y, (\Sigma^{\frac12}T^\T\Sigma_T^{-1} T \Sigma^{\frac12}) \Sigma^{-\frac12}y}
        = \lra{\Sigma^{-\frac12} y, P \Sigma^{-\frac12} y}
    \]
    where
    \[
        P
        = \Sigma^{\frac12} (T^\T \Sigma_T^{-1} T) \Sigma^{\frac12}
        = \Sigma^{-\frac12} H (H^\T \Sigma^{-1} H)^{-1} H^\T \Sigma^{-\frac12}
        = \Sigma^{-\frac12} H T \Sigma^{\frac12}
    \]
    which is the orthogonal projection onto $\col(\Sigma^{-\frac12} H)$.
    Recall that an orthogonal projection $P$ is defined by $P^2 = P$ and $P^\T = P$.
    Then
    \begin{align*}
        \lra{y, \Sigma^{-1} y} - \lra{y, T^\T\Sigma_T^{-1} Ty}
        &= \lra{\Sigma^{-\frac12} y, (I_p - P) \Sigma^{-\frac12} y} \\
        &= \lra{\Sigma^{-\frac12} y, (I_p - P)^2\Sigma^{-\frac12} y} \\
        &= \lra{(I_p - P)^\T \Sigma^{-\frac12} y, (I_p - P)\Sigma^{-\frac12} y} \\
        &= \norm{(I_p - P) \Sigma^{-\frac12} y}^2,
    \end{align*}
    where we note that $(I_p - P)^2 = I_p - P$ and that $I_p - P$ is symmetric.
    (In fact, $P^\perp = I_p - P$ is the orthogonal projection onto $\col(\Sigma^{-\frac12} H)^\perp$.)
    To conclude, see that
    \[
        \norm{(I_p - P) \Sigma^{-\frac12} y}^2
        = \norm{\Sigma^{-\frac12}(I_p - \Sigma^{\frac12}P \Sigma^{-\frac12}) y}^2
        = \norm{(I_p - HT) y}_\Sigma^2. \qedhere
    \]
\end{proof}

We note that $HT$ is a projection, but not necessarily an \emph{orthogonal} projection.

\section{Tensor Product Basis}
\label{app:tensor}

If the observations can be naturally represented as multi-index arrays in $\R^{p_1\times\cdots \times p_q}$,
where the total number of outputs is $p = \prod_{i=1}^q p_i$,
to obtain a reduction in parameters of $H$, a natural choice is to correspondingly decompose $H = H_1 \otimes \cdots \otimes H_q$ where $\otimes$ is the Kronecker product and $H_i$ a $p_i \times m_i$ matrix.
The latent processes are then naturally seen as a $\R^{m_1\times\cdots\times m_q}$-valued process,
where their total number is $m = \prod_{i=1}^q m_i$.
In this parametrisation of the ILMM, \cref{prop:tensor} shows that the projection and projected noise also become the Kronecker products:
$T = T_1 \otimes \cdots \otimes T_q$ and $\Sigma_T = \Sigma_{T_1} \otimes \cdots \otimes \Sigma_{T_q}$.
Using the vectorisation trick, $TY$ can be computed efficiently without the need to explicitly construct $T$.

\begin{proposition} \label{prop:tensor}
    Let $H$ be a basis that is a tensor product of other bases and let the observation noise $\Sigma$ factorise similarly:
    \[
        H = H_1 \otimes \cdots \otimes H_q
        \quad\text{and}\quad
        \Sigma = \Sigma_1 \otimes \cdots \otimes \Sigma_q.
    \]
    Then the projection is the tensor product of the projections and the projected noise is the tensor product of the projected noises:
    \[
        T = T_1 \otimes \cdots \otimes T_q
        \quad\text{and}\quad
        \Sigma_T = \Sigma_{T_1} \otimes \cdots \otimes \Sigma_{T_q}
    \]
    where
    $T_i = (H_i^\T \Sigma_i^{-1} H_i)^{-1} H_i^\T \Sigma_i^{-1}$
    and
    $\Sigma_i = (H_i^\T \Sigma_i^{-1} H_i)^{-1}$.
\end{proposition}
\begin{proof}
    Follows directly from the compatibility of the Kronecker product with matrix multiplication, transposition, and inversion.
\end{proof}

\section{Cost of Parametrising the Basis}
\label{app:cost_u}

\begin{figure}[t]
    \centering
    \includegraphics[width=\linewidth]{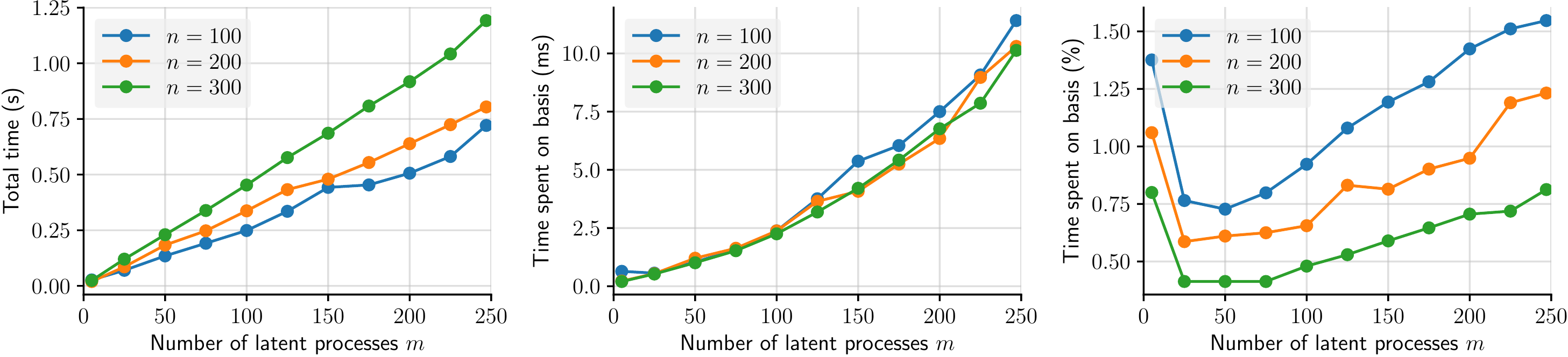}\vspace{-0.5em}
    \caption{
        Comparison of the time it takes to construct the basis $H$ to the total time of a log-marginal likelihood computation for a range of numbers of data points $n$ and numbers of latent processes $m$.
        The data used is from the temperature extrapolation experiment (\cref{exp:temp_extrapolation}).
    }
    \label{fig:timing_h}
\end{figure}

For the OILMM, the only computation that does not scale linearly with the number of latent processes $m$ is the parametrisation of the orthogonal part $U$ of the basis $H$, which takes $\O(m^2 p)$ time.
We argue that this cost is dominated by the cost $\O(n^3 m + n m p)$ of computing the log-marginal likelihood of the projected data:
\begin{enumerate}[topsep=1pt,itemsep=2pt]
    \item[\em (i)] typically $m \le n, p$;
    \item[\em (ii)] the cost of computing the log-marginal likelihood of the projected data scales with $n$, and often $n \gg m, p$; and
    \item[\em (iii)] assuming that $p$ is not much bigger than $n$, computing the log-marginal likelihood of the projected data costs at least $\O(n)$ more, so the cost of parametrising the basis $H$ should become insignificant as $n$ grows.
\end{enumerate}

We compare the time it takes to construct the basis $H$ to the total time of a log-marginal likelihood computation for a range of numbers of data points $n$ and numbers of latent processes $m$.
We use the data from the temperature extrapolation experiment (\cref{exp:temp_extrapolation}).
The results are depicted in \cref{fig:timing_h}.
Observe that, even in the worst case when $m = p = 247$, parametrising the basis $H$ takes no more than $1.5\%$ of the total time at $n = 100$ data points and no more than $0.8\%$ of the total time at $n = 300$ data points.
This cost is negligible, even in this worst case.

\section{Characterisation of Diagonal Projected Noise}
\label{app:diagonal-noise}

\cref{prop:decoupling} says that the projected noise is diagonal if and only if $H$ is of the form $H = \Sigma^{\frac{1}{2}} U S^{\frac{1}{2}}$ with $U$ a matrix with orthonormal columns and $S > 0$ diagonal.
This condition is awkward, as it couples $H$ and $\Sigma$.
Fortunately, \cref{prop:decoupling} also shows that we may drop $H$'s dependency on $\Sigma$ if and only if every column of $U$ is an eigenvector of $\Sigma$.

\begin{proposition} \label{prop:decoupling}
    The projected noise $\Sigma_T$ is diagonal if and only if $H$ is of the form $H = \Sigma^{\frac{1}{2}} U S^{\frac{1}{2}}$ with $U$ a matrix with orthonormal columns and $S > 0$ diagonal.
    Suppose that this is the case, and fix such a $U$.
    Then $H$ is of the form $H = U D^{\frac{1}{2}}$ with $D > 0$ diagonal if and only if every column of $U$ is an eigenvector of $\Sigma$.
\end{proposition}
\begin{proof}
    The projected noise is diagonal if and only if $H^\T \Sigma^{-1} H = S$ for some $S > 0$ diagonal.
    This condition is equivalent to
    \[
        S^{-\frac{1}{2}}H^\T \Sigma^{-\frac{1}{2}} \Sigma^{-\frac{1}{2}} H S^{-\frac{1}{2}} = I_m,
    \]
    which, in turn, holds if and only if $\Sigma^{-\frac{1}{2}} H S^{-\frac{1}{2}} = U$ is a matrix with orthonormal columns.
    Thus, the projected noise is diagonal if and only if $H$ is of the form
    $
        H = \Sigma^{\frac{1}{2}} U S^{\frac{1}{2}}
    $
    with $U$ a matrix with orthonormal columns and $S > 0$ diagonal.
    
    For the second statement, note that every column of $U$ is an eigenvector of $\Sigma$ if and only if it is an eigenvector of $\Sigma^{\frac12}$.
    Suppose that $H$ is of the form $H = U D^{\frac12}$ with $D > 0$ diagonal.
    Then
    \[
        \Sigma^{\frac12}U
        = \Sigma^{\frac12} U S^{\frac12} S^{-\frac12}
        = H S^{-\frac12}
        = U D^{\frac12} S^{-\frac12},
    \]
    so every column of $U$ is an eigenvector of $\Sigma^{\frac12}$.
    Conversely, suppose that every column of $U$ is an eigenvector of $\Sigma^{\frac12}$ with eigenvalues stacked into a diagonal matrix $D > 0$.
    Then
    \[
        H
        = \Sigma^{\frac12} U S^{\frac12}
        = U D S^{\frac12},
    \]
    which is of the desired form.
\end{proof}

\section{Kullback--Leibler Divergence Between an ILMM and OILMM}
\label{app:kl}
\begin{proposition} \label{prop:kl}
    Consider two ILMMs with equal $\Sigma = \sigma^2 I_p$, equal $K(t, t')$, but different bases $H$ and $\hat H$.
    Let $t_1,\ldots,t_n \in \Tc$ and denote $x_i = x(t_i)$ and $y_i = y(t_i)$.
    It then holds that
    \[
        \KL(p(y_{1:n}, x_{1:n}) \divsep \hat p(y_{1:n}, x_{1:n}))
        = \KL(\hat p(y_{1:n}, x_{1:n}) \divsep p(y_{1:n}, x_{1:n}))
        = n\frac{1}{2\sigma^2}\norm{H - \hat H}_F^2
    \]
    and
    \[
        \inf_{\hat H\,:\,\text{OILMM}}
        \KL(p(y_{1:n}, x_{1:n}) \divsep \hat p(y_{1:n}, x_{1:n}))
        \le n\frac{\E[\norm{f(t)}^2] }{\sigma^2} \max_i\, (1 - V_{ii})
        \le n\frac{\E[\norm{f(t)}^2] }{2\sigma^2} \norm{I_m - V}_F^2
    \]
    where $\hat H$ ranges over matrices of the form $U S^{\frac12}$ with $U$ a matrix with orthonormal columns and $S^{\frac12} > 0$ diagonal, $V$ is the orthogonal matrix collecting the right singular vectors of $H$, and $\E[\norm{f(t)}^2]$ denotes the variance of the observations under the first ILMM before adding noise.
\end{proposition}
\begin{proof}
    Start out by expanding the Kullback--Leibler divergence and noting that $p(x_{1:n}) = \hat p(x_{1:n})$:
    \begin{align*}
        \KL(p(y_{1:n}, x_{1:n}) \divsep \hat p(y_{1:n}, x_{1:n}))
        &= -\E_{p(y_{1:n}, x_{1:n})}
            \log\frac
                {\hat p(y_{1:n} \cond x_{1:n}) \cancel{\hat p(x_{1:n})}}
                {p(y_{1:n}\cond x_{1:n}) \cancel{p(x_{1:n})}} \\
        &= -\sum_{i=1}^n \E_{p(y_i, x_i)}\sbrac{
            \log \hat p(y_i \cond x_i)
            - \log p(y_i \cond x_i)
        } \\
        &= -\sum_{i=1}^n \E_{p(y_i, x_i)}\sbrac{
            \log \Normal(y_i\cond\hat H x_i, \sigma^2 I_p)
            - \log \Normal(y_i\cond H x_i, \sigma^2 I_p)
        }.
    \end{align*}
    Here
    \[
        \E_{p(y_i, x_i)}[\log \Normal(y_i\cond\hat H x_i, \sigma^2 I_p)]
        =
            -\frac p2 \log 2 \pi \sigma^2
            -\frac{1}{2\sigma^2} \E_{p(y_i, x_i)}[\norm{y_i - \hat H x_i}^2]
    \]
    where
    \begin{align*}
        \E_{p(y_i, x_i)}[\norm{y_i - \hat H x_i}^2]
        &= \E_{p(y_i, x_i)}\tr[
            y_i y_i^\T
            - 2 y_i x_i^\T \hat H^\T
            + x_i x_i^\T \hat H \hat H^\T
        ] \\
        &= \tr[
            HH^\T + \sigma^2 I
            - 2 H \hat H^\T
            + \hat H \hat H^\T
        ] \\
        &= p \sigma^2 + \tr[
            (H - \hat H)
            (H - \hat H)^\T
        ] \\
        &= p \sigma^2 + \norm{H - \hat H}^2_F.
    \end{align*}
    Therefore,
    \[
        \KL(p(y_{1:n}, x_{1:n}) \divsep \hat p(y_{1:n}, x_{1:n}))
        =
            n\frac{1}{2\sigma^2}\norm{H - \hat H}_F^2.
    \]
    Let $H = U S V^\T$ be the SVD of $H$ where $U$ is a truncated  orthogonal matrix with the same shape as $H$, $S > 0$ is a square diagonal matrix, and $V$ is a square orthogonal matrix.
    Note that $U^\T U = I_m$, but $U U^\T \neq I_p$.
    Then, choosing $\hat H = U S$,
    \[
        \inf_{\hat H\,:\,\text{OILMM}}
        \KL(p(y_{1:n}, x_{1:n}) \divsep \hat p(y_{1:n}, x_{1:n}))
        \le n\frac{1}{2\sigma^2}\norm{U(SV^\T - S)}_F^2
        = n\frac{1}{2\sigma^2}\norm{SV^\T - S}_F^2
    \]
    since $\norm{UA}_F^2 = \tr[A^\T U^\T U A] = \tr[A^\T A] = \norm{A}_F^2$.
    We now further simplify:
    \[
        \norm{SV^\T - S}_F^2
        = \tr[
            (SV^\T - S)(SV^\T - S)^\T
        ]
        = \tr[
            SV^\T V S - S V^\T S - S V S + S S
        ]
        = 2 \tr[
            S S - S V S
        ].
    \]
    Hence, by definition of the trace and the fact that $S$ is diagonal,
    \[
        \norm{SV^\T - S}_F^2
        = 2\sum_{i=1}^m
            S_{ii}^2(1 - V_{ii})
        \le 2\parens*{
            \sum_{i=1}^m S_{ii}^2
        } \max_i\, (1 - V_{ii})
        = 2\E[\norm{f}^2]  \max_i\, (1 - V_{ii}),
    \]
    since 
    \[
        \E[\norm{f(t)}^2] = \E\tr[f(t)f^\T(t)] = \tr[HH^\T] = \tr[S^2].
    \]
    Therefore,
    \[
        \norm{SV^\T - S}_F^2
        \le 2\E[\norm{f}^2]  \max_i\, (1 - V_{ii})
        \le 2\E[\norm{f}^2] \sum_{i=1}^m (1 - V_{ii})
        = \E[\norm{f}^2]\norm{I_m - V}_F^2,
    \]
    where the equality follows from a similar calculation:
    \[
        \norm{I_m - V}_F^2
        =\tr[I_m - V^\T - V + V^\T V]
        =2\tr[I_m - V]. \qedhere
    \]
\end{proof}

\section{OILMM: Projection and Projected Noise}
\label{app:oilmm-proj-and-noise}

\begin{proposition}
    \label{prop:form}
    Consider the OILMM (\cref{mod:OILMM}).
    Then the projection and projected noise are given by
    \[
        T = S^{-\frac12}U^\T
        \quad\text{and}\quad
        \Sigma_T = \sigma^{2}S^{-1} + D.
    \]
\end{proposition}
\begin{proof}
    To begin with, note that
    \[
        y \sim \GP(
            \underbracket{
                H K(t, t') H^\T + \delta[t - t'](\sigma^2 I_p + H D H^\T)
            }_{
                \mathclap{
                    \displaystyle
                    =H (K(t, t') + \delta[t - t'] D) H^\T + \delta[t - t'] \sigma^2 I_p,
                }
            }
        ),
    \]
    so we can assume that $D = 0$ by ``absorbing it into $K(t, t')$''.
    We then find that
    \[
        H^\T \Sigma^{-1} H = \sigma^{-2} S,
    \]
    so
    \[
        \Sigma_T
        = T \Sigma T^\T
        = (H^\T \Sigma^{-1} H)^{-1}
        = \sigma^2 S^{-1}.
    \]
    Moreover, then
    \[
        T
        = (H^\T \Sigma^{-1} H)^{-1} H^\T \Sigma^{-1} \\
        = (\sigma^2 S^{-1})(\sigma^{-2} S^{\frac12} U^\T)
        = S^{-\frac12} U^\T.
    \]
    Finally, ``pull $D$ back out of $K(t,t')$'', which we note is equivalent to adding it to $\Sigma_T$ by \cref{prop:general_sufficiency}.
\end{proof}

\section{OILMM: Likelihood}
\label{app:oilmm-likelihood}

\begin{proposition} \label{prop:likelihood}
    Consider the OILMM (\cref{mod:OILMM}).
    Let $Y$ be an $p \times n$ matrix of observations for $y$.
    Then
    \[
        \log p(Y)
        =
            - \frac{n}{2} \log |S|
            - \frac{n (p-m)}{2} \log 2 \pi \sigma^2
            - \frac{1}{2\sigma^2} \norm{(I_p-UU^\T) Y}_F^2
            + \sum_{i=1}^m \log \Normal((TY)_{i:}\cond 0, K_i + (\sigma^2/S_{ii} + D_{ii}) I_n)
    \]
    where $\norm{\vardot}_F$ denotes the Frobenius norm and $K_i$ is the $n\times n$ kernel matrix for the $i$\textsuperscript{th} latent process $x_i$.
\end{proposition}
\begin{proof}
    By \cref{prop:general_sufficiency} and \cref{prop:regularisation-term}, we have
    \[
        \log p(Y) = 
                -\frac{n (p - m)}{2}\log 2\pi
                - \frac{n}{2}\log \frac{|\Sigma|}{|\Sigma_T|}
                - \frac12 \sum_{i=1}^n \norm{(I_p - HT)y_i}_{\Sigma}^2
            + \log \int p(x) \prod_{i=1}^n \Normal(T y_i\cond x_i,\Sigma_T) \isd x.
    \]
    Using the same trick as in the proof of \cref{prop:form}, assume that $D = 0$ by ``absorbing it into $K(t, t')$''.
    We then simplify the terms one by one.
    First, we have that
    \[
        \log \frac{|\Sigma|}{|\Sigma_T|}
        = \log \frac{|\sigma^2 I_p|}{|\sigma^{2} S^{-1}|}
        = (p - m)\log \sigma^2 + \log |S|.
    \]
    Second, note that $I_p - HT  = I_p - UU^\T$, which we denote by $P_{U^\perp}$ and which is symmetric, so
    \[
        \norm{(I_p - HT)y_i}_{\Sigma}^2
        = \norm{P_{U^\perp}y_i}_{\Sigma}^2
        = \lra{P_{U^\perp} y_i, \Sigma^{-1}P_{U^\perp}y_i}
        = \sigma^{-2}\lra{P_{U^\perp} y_i, P_{U^\perp} y_i}
        = \sigma^{-2}\tr[P_{U^\perp} P_{U^\perp} y_i y_i^\T].
    \]
    Then sum over $i=1,\ldots,n$ to obtain
    \[
        \sum_{i=1}^n \norm{(I_p - HT)y_i}_{\Sigma}^2
        = \sigma^{-2}\tr[P_{U^\perp} P_{U^\perp} YY^\T]
        = \sigma^{-2} \norm{P_{U^\perp} Y}_F^2.
    \]
    Finally,
    \[
        \log \int p(x) \prod_{i=1}^n \Normal(T y_i\cond x_i,\Sigma_T) \isd x
        = \sum_{i=1}^m \log \Normal((TY)_{i:}\cond 0, K_i + (\sigma^2/S_{ii} + D_{ii}) I_n)
    \]
    follows from independence of the latent processes and remembering that we ``absorbed $D$ into $K(t, t')$''.
\end{proof}

Observe that
\[
    \norm{(I_p-UU^\T) Y}_F^2
    = \norm{Y}_F^2 - \norm{U^\T Y}_F^2,
\]
which is a computationally more efficient implementation.

\section{OILMM: Decomposition of the Mean Squared Error}
\label{app:mse}

\begin{proposition} \label{prop:mse}
    Let $H = U S^{\frac12}$ with $U$ a matrix with orthonormal columns and $S^{\frac12} > 0$ diagonal.
    Then
    \[
        \underbracket{\norm{y - H x}^2 \vphantom{\sum}}_{
            \text{MSE}
        }
        = \underbracket{\norm{P_{U^\perp} y}^2 \vphantom{\sum}}_{
            \mathclap{\substack{\text{data not} \\ \text{captured by basis}}}
        } + \sum_{i=1}^m \overbracket{S_{ii} \vphantom{\sum}}^{
            \mathclap{\substack{\text{variance of} \\ \text{$i$\textsuperscript{th} latent process}}}
        } \underbracket{((T y)_i - x_i)^2 \vphantom{\sum}}_{
            \mathclap{\substack{\text{MSE of} \\ \text{$i$\textsuperscript{th} latent process}}}
        }
    \]
    where $T = S^{-\frac12}U^\T$ and $P_{U^\perp}$ is the orthogonal projection onto the orthogonal complement of $\col(U)$.
\end{proposition}
\begin{proof}
    By expanding and using orthogonality of $U$,
    \begin{align*}
        \norm{y - H x}^2
        &= \norm{y}^2 - 2 \lra{y, U S^{\frac12} x} + \norm{U S^{\frac12} x}^2 \\
        &= \norm{y}^2 - \norm{U^\T y}^2 + \norm{U^\T  y}^2 - 2 \lra{U^\T y, S^{\frac12} x} + \norm{S^{\frac12} x}^2 \\
        &= \lra{y, (I_p - UU^\T)y} + \sum_{i=1}^m \parens{
            \lra{u_i, y}^2 - 2 \lra{u_i, y} S^{\frac12}_{ii} x_i + (S^{\frac12}_{ii} x_i)^2
        } \\
        &= \lra{y, (I_p - UU^\T)y} + \sum_{i=1}^m  S_{ii}\parens{
            S^{-1}_{ii}\lra{u_i, y}^2 - 2 S^{-\frac12}_{ii} \lra{u_i, y}  x_i + x_i^2
        } \\
        &= \lra{y, (I_p - UU^\T)y} + \sum_{i=1}^m S_{ii} ((T y)_i - x_i)^2,
    \end{align*}
    where $u_i$ is the $i$\textsuperscript{th} column of $U$.
    Note that $P_U = UU^\T$ is the orthogonal projection onto $\col(U)$, so $I - UU^\T = P_{U^\perp}$ is the orthogonal projection onto the orthogonal complement of $\col(U)$.
    Therefore,
    \[
        \lra{y, (I_p - UU^\T)y}
         = \lra{y, P_{U^\perp} y}
         = \lra{y, P^2_{U^\perp} y}
         = \lra{P^\T_{U^\perp} y, P_{U^\perp} y}
         = \lra{P_{U^\perp} y, P_{U^\perp} y}
         = \norm{P_{U^\perp} y}^2. \qedhere
    \]
\end{proof}

\section{OILMM: Missing Data}
\label{app:oilmm-missing-data}

For a matrix or vector $A$,
let $A\obs$ and $A\miss$ denote the rows of $A$ corresponding to respectively observed and missing values.

\begin{proposition}
    \label{prop:form-missing}
    Consider the OILMM (\cref{mod:OILMM}).
    For observed outputs $y\obs$, which are a subset of all outputs $y$, the projection and projected noise are given by
    \[
        T\obs = S^{-\frac12} U\obs^\dagger
        \quad\text{and}\quad
        \Sigma_{T\obs} = \sigma^{2}S^{-\frac12}(U\obs^\T U\obs)^{-1} S^{-\frac12} + D
    \]
    where $U\obs^\dagger$ is the pseudo-inverse of $U\obs$.
\end{proposition}
\begin{proof}
    Note that
    \[
        y\obs \sim \GP(
            H\obs K(t, t') H\obs^\T + \delta[t-t'](\sigma^2 I\obs + H\obs D H\obs^\T)
        ),
    \]
    so $y\obs$ is an ILMM with basis $H\obs$ and observation noise $\sigma^2 I + H\obs D H\obs^\T$.
    The proof proceeds like that of \cref{prop:form}, also using trick of assuming that $D = 0$ by ``absorbing it into $K(t, t')$''.
    To begin with, we have
    \[
        H^\T \Sigma^{-1} H = \sigma^{-2} S^{\frac12}U\obs^\T U\obs S^{\frac12},
    \]
    so
    \[
        \Sigma_T
        = T \Sigma T^\T
        = (H^\T \Sigma^{-1} H)^{-1}
        = \sigma^{2}S^{-\frac12}(U\obs^\T U\obs)^{-1} S^{-\frac12}.
    \]
    Moreover, then
    \[
        T
        = (H^\T \Sigma^{-1} H)^{-1} H^\T \Sigma^{-1} \\
        = (\sigma^{2}S^{-\frac12}(U\obs^\T U\obs)^{-1} S^{-\frac12})(\sigma^{-2} S^{\frac12} U\obs^\T)
        = S^{-\frac12} (U\obs^\T U\obs)^{-1} U\obs^\T
        = S^{-\frac12} U\obs^\dagger.
    \]
    Finally, ``pull $D$ back out of $K(t,t')$'', which, again, is equivalent to adding it to $\Sigma_T$.
\end{proof}

\begin{remark} \label{rem:using-form-missing}
    When using $T\obs$ and $\Sigma_{T\obs}$, in the likelihood computation in \cref{prop:likelihood}, from \cref{prop:regularisation-term}, it can be seen that two things change:
    for every time point with missing data,
    \begin{remlist}[topsep=1pt,itemsep=2pt]
        \item $H\obs T\obs = U\obs U\obs^\dagger$, so $U U^\T$ becomes $U\obs U\obs^\dagger$; and
        \item $\tfrac12\log |\Sigma_{T\obs}|$ gives an extra term $-\tfrac12\log |U\obs^\T U\obs|$.
    \end{remlist}
\end{remark}

\subsection{Diagonal Approximation of Projected Noise}
\label{app:missing-data-projection-noise-approximation}

For a matrix $A$, let $d[A]$ denote the diagonal matrix resulting from setting the off-diagonal entries of $A$ to zero.

\begin{proposition}
    \label{prop:diagonal-approximation}
    For $\Sigma_{T\obs}$ from \cref{prop:form-missing}, we have
    \[
        \frac
            {\norm{\Sigma_{T\obs} - d[\Sigma_{T\obs}]}\ss{op}}
            {\norm{d[\Sigma_{T\obs}]}\ss{op}}
        \le
            \frac{S\ss{max}}{S\ss{min}}
            \max_{y \in \col(H): \norm{y} = 1} \norm{y\miss}^2
    \]
    where $\norm{\vardot}\ss{op}$ denotes the operator norm,
    and $S\ss{min}$ and $S\ss{max}$ are the smallest and largest diagonal values of $S$.
\end{proposition}
\begin{proof}
    Let $e_i$ be the $i$\textsuperscript{th} unit vector.
    Denote $A = (U\obs^\T U\obs)^{-1}$,
    and let $\lambda\ss{min}$ and $\lambda\ss{max}$ be the minimum and maximum eigenvalue of $A$.
    To begin with,
    \[
        d[A]_{ii}
        = \lra{e_i, A e_i}
        \in [\lambda\ss{min}, \lambda\ss{max}].
    \]
    Let $x \in \R^m$ be such that $\norm{x}=1$.
    Then
    \[
        \lra{x, (A - d[A]) x}
        = \lra{x, A x} - \lra{x, d[A] x}
        \le \lambda\ss{max} - \lambda\ss{min}.
    \]
    Similarly,
    \[
        \lra{x, (A - d[A]) x} \ge -\lambda\ss{max} + \lambda\ss{min}.
    \]
    Therefore,
    \[
        |\lra{x, (A - d[A]) x}| \le \lambda\ss{max} - \lambda\ss{min},
    \]
    so
    \[
        \norm{A - d[A]}\ss{op} \le \lambda\ss{max} - \lambda\ss{min}
    \]
    Since $S$ is diagonal, we have
    \[
        \Sigma_{T\obs} - d[\Sigma_{T\obs}]
        = \sigma^2 S^{-\frac12}(A - d[A]) S^{-\frac12}.
    \]
    Using the derived bound on the operator norm and submultiplicativity of the operator norm, it follows that
    \[
        \norm{\Sigma_{T\obs} - d[\Sigma_{T\obs}]}\ss{op}
        \le \sigma^2 S^{-1}\ss{min} (\lambda\ss{max} - \lambda\ss{min}).
    \]
    Moreover,
    \[
        \norm{d[\Sigma_{T\obs}]}\ss{op}
        = \sigma^2
            \max_{i=1,\ldots,m} (S_{ii}^{-1}d[A]_{ii} + D_{ii})
        \ge
            \sigma^2
            \max_{i=1,\ldots,m} S_{ii}^{-1}d[A]_{ii}
        \ge
            \sigma^2
            S\ss{max}^{-1}
            \max_{i=1,\ldots,m} d[A]_{ii}
        \ge
            \sigma^2 S\ss{max}^{-1}\lambda\ss{max}.
    \]
    Therefore,
    \[
        \frac
            {\norm{\Sigma_{T\obs} - d[\Sigma_{T\obs}]}\ss{op}}
            {\norm{d[\Sigma_{T\obs}]}\ss{op}}
        \le
            \frac{S\ss{max}}{S\ss{min}}\parens*{1 - \frac{\lambda\ss{min}}{\lambda\ss{max}}}.
    \]
    By definition of $\lambda\ss{min}$ and $\lambda\ss{max}$ and orthogonality of $U$, we have that
    \[
        \frac{1}{\lambda\ss{min}}
        =
            \max_{x \in \R^m :\norm{x}=1} \norm{U\obs x}^2
        \le
            \max_{x \in \R^m :\norm{x}=1} \norm{U x}^2
        =
            1
        \quad\text{and}\quad
        \frac{1}{\lambda\ss{max}}
        =
            \min_{x \in \R^m :\norm{x}=1} \norm{U\obs x}^2.
    \]
    Substitute these results into the bound:
    \[
        \frac
            {\norm{\Sigma_{T\obs} - d[\Sigma_{T\obs}]}\ss{op}}
            {\norm{d[\Sigma_{T\obs}]}\ss{op}}
        \le
            \frac{S\ss{max}}{S\ss{min}}\parens*{1 - \min_{x \in \R^m :\norm{x}=1} \norm{U\obs x}^2}
        =
            \frac{S\ss{max}}{S\ss{min}}\max_{x \in \R^m :\norm{x}=1} \parens{1 - \norm{U\obs x}^2}.
    \]
    By orthogonality of $U$, for $x \in \R^m$ such that $\norm{x} = 1$, we have
    \[
        1 = \norm{x}^2 = \norm{U x}^2 = \norm{U\obs x}^2 + \norm{U\miss x}^2,
    \]
    so $1 - \norm{U\obs x}^2 = \norm{U\miss x}^2$.
    Therefore,
    \[
        \max_{x \in \R^m :\norm{x}=1} \parens{1 - \norm{U\obs x}^2}
        =
            \max_{x \in \R^m :\norm{x}=1} \norm{U\miss x}^2
        =
            \max_{x \in \R^m :\norm{x}=1} \norm{(U x)\miss}^2
        =
            \max_{y \in \col(U) :\norm{y}=1} \norm{y\miss}^2
    \]
    and we conclude by noting that $\col(U) = \col(H)$.
\end{proof}

\begin{corollary} \label{prop:bound-U-inequality}
    Suppose $\norm{U}_\infty^2 \le C / p$ for some $C \ge 1$, and that $s$ outputs are missing.
    Then
    \[
        \frac
            {\norm{\Sigma_{T\obs} - d[\Sigma_{T\obs}]}\ss{op}}
            {\norm{d[\Sigma_{T\obs}]}\ss{op}}
        \le
            C
            \frac{S\ss{max}}{S\ss{min}}
            \frac{m s}{p}.
    \]       
\end{corollary}
\begin{proof}
    Let $y \in \col(H)$ be such that $\norm{y} = 1$.
    Then $y = Ux$ for some $x \in \R^m$ such that $\norm{x} = 1$.
    Therefore,
    \[
        \norm{y\miss}^2
        = \sum_{i \in \text{missing}} (Ux)_i^2
        \le \sum_{i \in \text{missing}} \norm{U_{i:}}^2 \norm{x}^2
        = \sum_{i \in \text{missing}} \norm{U_{i:}}^2
        \le \frac{C m s}{p},
    \]
    so the result follows from the previous proposition.
\end{proof}

\subsection{Variational Approach}
\label{app:missing-data-variational-approximation}
Let $Y\ss{o}$ be the observed data.
Complement $Y\ss{o}$ with missing data $Y\ss{m}$ such that $Y = Y\ss{o} \cup Y\ss{m}$ is complete.
Then a way to deal with missing data is to use variational inference.
In particular, assume a Gaussian approximate posterior distribution $q(Y\ss{m})$ over $Y\ss{m}$, and maximise the evidence lower bound (ELBO) $\L$ using gradient-based optimisation:
\[
    \log p(Y\ss{o}) \ge \E_{q(Y\ss{m})}[\log p(Y)] + H[q(Y\ss{m})] = \L[q(Y\ss{m})],
\]
where the expectation can be approximated using the reparametrisation trick \parencite{Kingma:2013:Auto-Encoding_VB}, $\log p(Y)$ can be computed efficiently because $Y$ is complete, and $H[q(Y\ss{m})]$ denotes the entropy of $q(Y\ss{m})$. This approach provides a tractable solution when the missing data are not too numerous.

\section{OILMM: Heterogeneous Observation Noise}
\label{app:oilmm-het-obs-noise}
Although the specification of the observation noise $\Sigma = \sigma^2 I_p + H D H^\T$ in the OILMM does not allow for heterogeneous observation noise, it is possible to set $\Sigma = \diag(\sigma_1^2, \ldots, \sigma_p^2)$ and use \cref{prop:decoupling} to include $\Sigma$ in the parametrisation of $H$: $H = \Sigma^{\frac12} U S^{\frac12}$.
This parametrisation can be interpreted in two ways:
\begin{enumerate}[topsep=0pt,itemsep=2pt]
    \item[\em (i)]
        The model has a whitening transform built in.
        In the projection $T$, the (noise in the) data will first by whitened by $\Sigma^{-\frac12}$.
        Hence, this parametrisation can be used as a more principled substitute for the usual data normalisation where the outputs are divided by their empirical standard deviation prior to feeding them to the model.
    \item[\em (ii)]
        The basis is orthogonal with respect to a weighted Euclidean inner product:
        $
            \lra{h_i, h_j}_{\Sigma} = \sum_{k=1}^p h_{ik} h_{jk}/ \sigma_k^2 = 0
        $
        for $i \neq j$.
        Intuitively, this means that the basis is orthogonal in the usual sense after stretching the $i$\textsuperscript{th} dimension by $\sigma_i^{-1}$.
\end{enumerate}
Although this construction provides additional flexiblity, it does require that $D=0$ to avoid a circular dependency between $\Sigma$ and $H$.

\section{Computational Scaling Experiment (Sec.\ \ref{exp:scaling}) Additional Details}
\label{app:scaling}
Measurements were performed using a MacBook Pro with a 2.7 GHz Intel Core i7 processor and 16 GB RAM. Code was implemented in Julia 1.0 \parencite{bezanson2017julia} and memory and time were measured using the \code{@allocated} and the \code{@elapsed} macros, respectively, with the measurements averaged over 10 samples run serially. This means memory reported is the total memory allocated, not peak memory consumption.

\section{Point Process Experiment (Sec.\ \ref{exp:rainforest}) Additional Details and Analysis}
\label{app:point-process}

We consider a subset of the extensive rainforest data set credited to \citet{rainforest,Condit:2005,Hubbell:1999}. The data features a 1000$\,$m $\times$ 500$\,$m rainforest dynamics plot in Barro Colorado Island, Panama. In the survey area, the locations of all \emph{Trichilia tuberculata} (a tree species of the Mahogany family) have been measured (see \cref{fig:log-intensity}).

We tackle this spatial point pattern with a log-Gaussian Cox process model, which is an inhomogeneous Poisson process model for count data. The unknown intensity function $\Sigma(x)$ is modelled with a Gaussian process such that $f(x) = \log \Sigma(x)$. Locally-constant intensity in subregions are modelled by discretising the region into $np$ bins \parencite{Moller+Syversveen+Waagepetersen:1998}. This leads to a Poisson observation model for each bin. This model reaches posterior consistency in the limit of bin width going to zero \parencite{Tokdar+Ghosh:2007}. The accuracy thus improves with tighter binning. We use a separable Mat\'ern-$5/2$ GP prior over $f(x_1,x_2)$, and discretise the area into a $n \times p = 200{\times}100$ (each bin is \SI{5}{m} $\times$ \SI{5}{m}) grid with $np=20000$ grid bins in total, and treat the first dimension as time. The conditional probability of the complete binned data set given the latent GP is therefore 
\begin{equation}
    p(Y \mid f) \approx \prod_{i=1}^n \prod_{j=1}^p \mathrm{Poisson}(Y_{ij} \mid a e^{f(r_{ij})}), \nonumber
\end{equation}
where $r_{ij}$ is the coordinate of the $ij$\textsuperscript{th} bin, $Y_{ij}$ is the number of data points in the $ij$\textsuperscript{th} bin, $Y$ is the $n \times p$ matrix of counts, and $a$ is the area of each bin.
\begin{figure}[t]
    \centering\scriptsize
    \setlength{\figurewidth}{.4\columnwidth}
    \setlength{\figureheight}{.5\figurewidth}  
    \tikzset{every picture/.append style = {xscale = .5, yscale = .5}}  
    % Customize the figure
    \pgfplotsset{
        y tick label style={rotate=90},
        scale only axis,
        axis on top,
        clip=true
    }  
    \begin{subfigure}[t]{.48\columnwidth}
\begin{tikzpicture}

\begin{axis}[%
axis on top,
xmin=0,
xmax=1000,
xlabel={$x_1$ (meters)},
y dir=reverse,
ymin=0,
ymax=500,
ylabel={$x_2$ (meters)},
axis background/.style={fill=white},
legend style={legend cell align=left,align=left,draw=white!15!black},
width=\figurewidth,
height=\figureheight
]
\addplot [forget plot] graphics [xmin=-0.394321766561514,xmax=1000.39432176656,ymin=-0.394944707740916,ymax=500.394944707741] {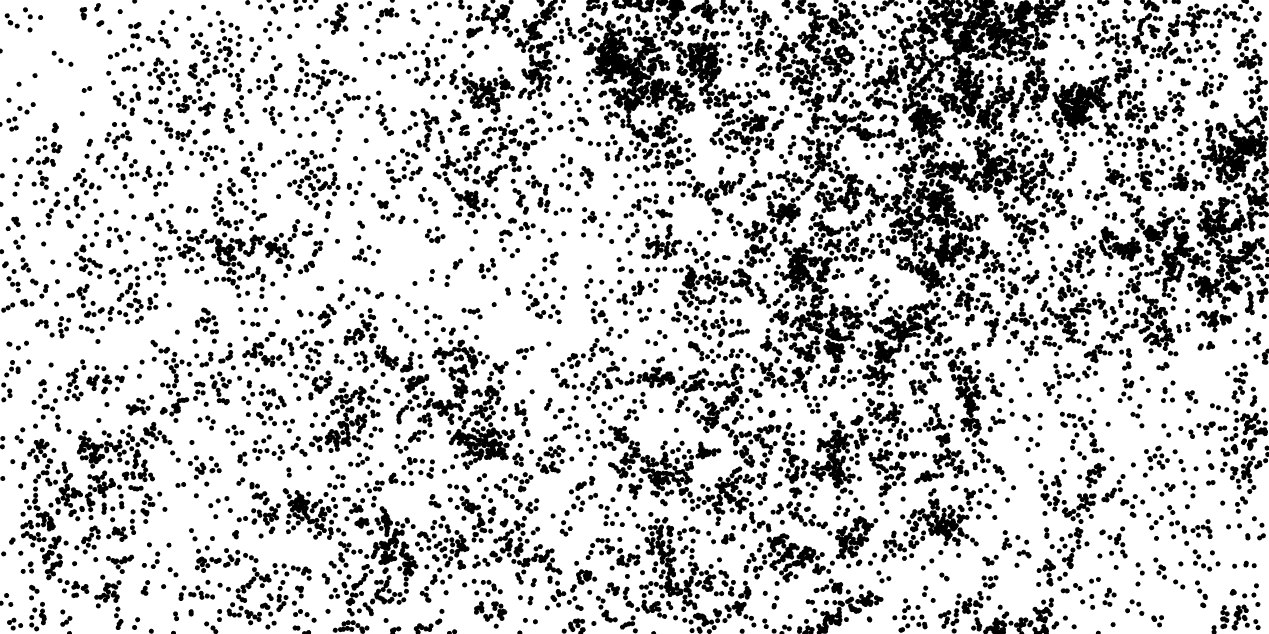};
\end{axis}
\end{tikzpicture}%
    \end{subfigure}
    \hspace*{\fill}
    \begin{subfigure}[t]{.48\columnwidth}
\begin{tikzpicture}
\begin{axis}[%
point meta min=-3,
point meta max=3,
axis on top,
xmin=0,
xmax=1000,
xlabel={$x_1$ (meters)},
y dir=reverse,
ymin=0,
ymax=500,
ylabel={$x_2$ (meters)},
axis background/.style={fill=white},
legend style={legend cell align=left,align=left,draw=white!15!black},
width=\figurewidth,
height=\figureheight
]
\addplot [forget plot] graphics [xmin=-5.05050505050505,xmax=1005.05050505051,ymin=-1.25628140703518,ymax=501.256281407035] {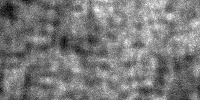};
\end{axis}
\end{tikzpicture}%
    \end{subfigure}  
    \caption{Observations of the rainforest tree locations (left), and posterior mean log-intensity for the log-Gaussian Cox process model (right) with a grid of $np=20000$ observation bins.}
    \label{fig:log-intensity}
\end{figure}    

We perform $10^5$ iterations of block Gibbs sampling, each of which comprises $10$ iterations Elliptical Slice Sampling \parencite{Murray:2010:Elliptical_Slice_Sampling,murray2010slice} for the Gaussian process given its hyperparameters, and a single iteration of Metropolis Hastings \parencite{hastings1970monte} with proposal distribution $\Normal(\theta, 0.05^2)$ for the log of the hyperparameters given the latent GP-distributed function. Each step of Elliptical Slice Sampling requires an additional sample from the GP prior at the current hyperparameter values, while each step of Metropolis Hastings requires a log marginal likelihood evaluation. As such approximately $10^6$ samples from the prior were drawn, and $10^5$ log marginal likelihood calculations undertaken.
The kernel is a product of two \Matern-$5/2$ kernels with a shared length scale. A single process variance is utilised, and a nugget term is added. The $\log$ of each of the three hyperparameters was given a $\Normal(0, 1)$ prior.
\cref{fig:log-joint} shows the log joint of the entire state after each iteration, while \cref{fig:traces} shows the progress of each hyperparameter per iteration.

\begin{figure}[t]
    \centering
    \begin{subfigure}[t]{.48\columnwidth}
        \includegraphics[width=\linewidth]{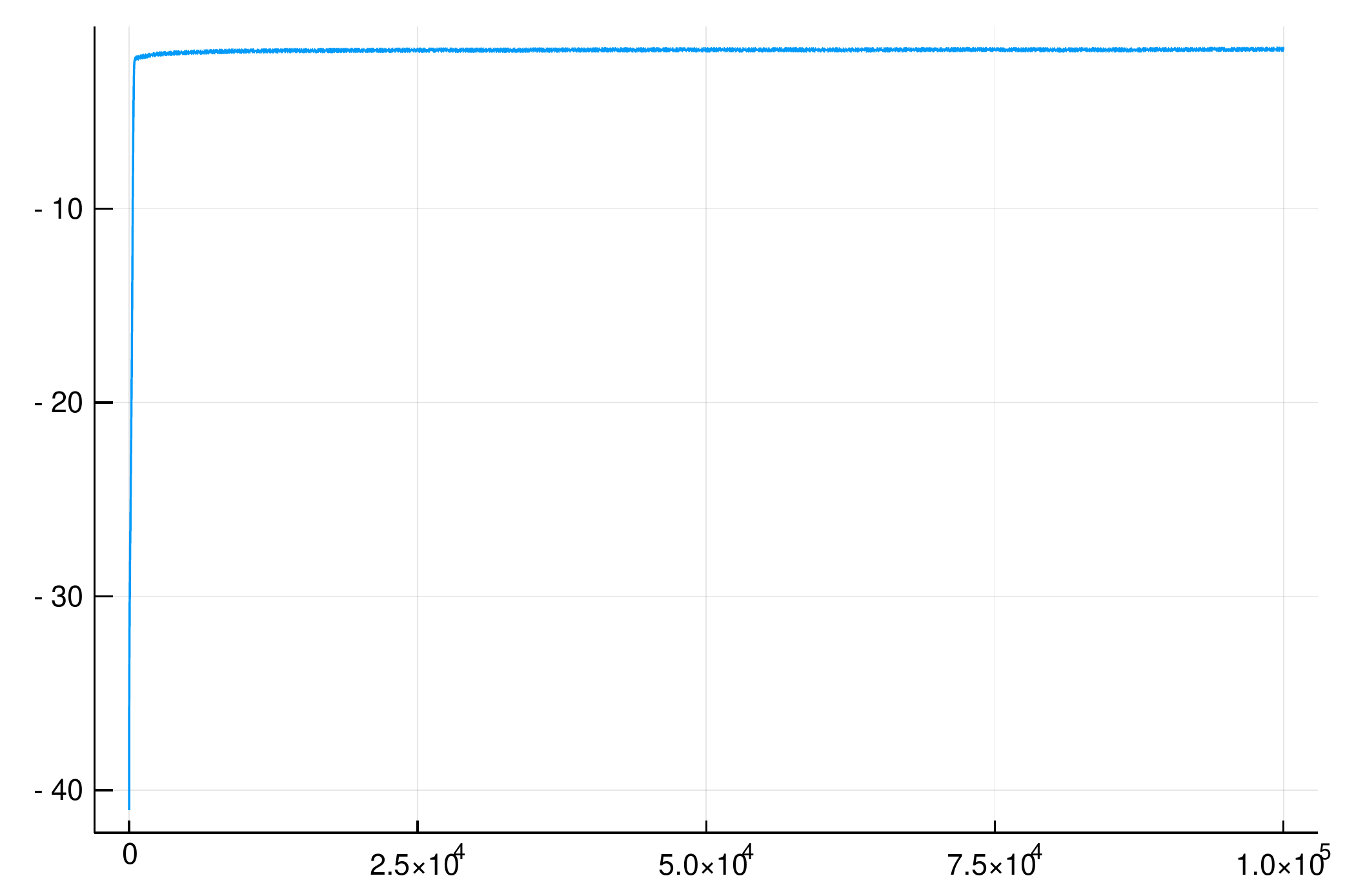}
        \caption{}
        \label{fig:log-joint}
    \end{subfigure}
    \hspace*{\fill}
    \begin{subfigure}[t]{.48\columnwidth}
        \includegraphics[width=\linewidth]{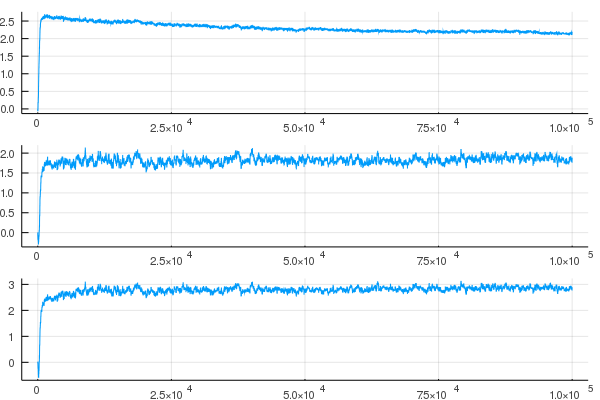}
        \caption{}
        \label{fig:traces} 
    \end{subfigure}  
    \caption{
        (a): Log--joint probability per iteration.
        (b):
            Hyperparameters per iteration.
            Shows the length scale, process variance, and nugget variance respectively.
    }
    \label{fig:my_label}
\end{figure}

The times in \cref{fig:timing-ratios} were obtained via \code{BenchmarkTools.jl} \parencite{BenchmarkTools.jl-2016}.
The implementation of the standard Kronecker product decomposition trick makes use of \code{Kronecker.jl}, and Julia's \parencite{bezanson2017julia} standard linear algebra libraries, which make use of OpenBLAS and LAPACK to efficiently perform matrix-matrix products and compute eigendecompositions. The implementation of the state-space GP additionally makes use of \code{StaticArrays.jl} for efficient stack-allocated matrices, and \code{Stheno.jl} for GP-related functionality.
Timing experiments were conducted on a single CPU core.

When computing the log marginal likelihood, the state-space implementation of the GP makes use of the infinite-horizon trick introduced to the GP literature by \textcite{Solin:2018:Infinite-Horizon_Gaussian_Processes}. However, this trick is only exploited here once the filtering covariance has converged, which is determined by the point at which the Frobenius norm of the difference between the filtering covariance at the $t$\textsuperscript{th} and $(t-1)$\textsuperscript{th} iterations drops below $10^{-12}$. This produces log marginal likelihood evaluations and samples from the prior that are exact for all practical purposes.

\subsection{ Performance versus Kronecker Trick }

\begin{table}[t]
    \small
    \caption{
        Description of the data points associated with the timing experiment from \cref{fig:timing-ratios}
    }
    \label{tab:kronecker-timing-statistics}
    \centering
    \vspace{1em}
    \begin{tabular}{@{}lllll@{}}
        \toprule
               & \multicolumn{2}{l}{LML}                    & \multicolumn{2}{l}{RNG}  \\
        $n$     & Kronecker      & OILMM            & Kronecker      & OILMM          \\ \midrule
        $2000$ & $2.45 \pm 0.0193$     & $0.403 \pm 0.00414$    & $2.45 \pm 0.0278$      & $0.478 \pm 0.00376$    \\
        $1000$ & $0.365 \pm 0.00256$   & $0.0712 \pm 0.000369$  & $0.364 \pm 0.00451$    & $0.0892 \pm 0.000435$  \\
        $200$  & $0.0111 \pm 0.000301$ & $0.00235 \pm 2.53 \!\times\! 10^{-5}$  & $0.0112 \pm 9.89 \!\times\! 10^{-5}$   & $0.00318 \pm 1.2 \!\times\! 10^{-5}$   \\
        $100$  & $0.00237 \pm 8.66 \!\times\! 10^{-6}$ & $0.000582 \pm 6.55 \!\times\! 10^{-7}$ & $0.00237 \pm 3.1 \!\times\! 10^{-5}$   & $0.000792 \pm 8.69 \!\times\! 10^{-7}$ \\
        $40$   & $0.00044 \pm 4.35 \!\times\! 10^{-7}$ & $0.000109 \pm 2.22 \!\times\! 10^{-7}$ & $0.000436 \pm 3.19 \!\times\! 10^{-7}$ & $0.000141 \pm 2.0 \!\times\! 10^{-7}$  \\
        $20$   & $9.15 \!\times\! 10^{-5} \pm 1.48 \!\times\! 10^{-7}$ & $2.38 \!\times\! 10^{-5} \pm 2.1 \!\times\! 10^{-7}$   & $9.06 \!\times\! 10^{-5} \pm 1.89 \!\times\! 10^{-7}$  & $3.13 \!\times\! 10^{-5} \pm 1.72 \!\times\! 10^{-7}$  \\
        $10$   & $1.84 \!\times\! 10^{-5} \pm 1.54 \!\times\! 10^{-7}$ & $9.87 \!\times\! 10^{-6} \pm 1.08 \!\times\! 10^{-7}$  & $1.84 \!\times\! 10^{-5} \pm 3.02 \!\times\! 10^{-7}$  & $1.15 \!\times\! 10^{-5} \pm 1.17 \!\times\! 10^{-7}$  \\ \bottomrule
    \end{tabular}
\end{table}

\cref{fig:timing-ratios} demonstrates that, for the particular approach taken to inference in the Poisson process and, importantly, the dimensions of the data, the Kronecker trick discussed by \textcite{saatcci2012scalable} takes slightly longer to compute log marginal likelihoods and generate samples than does the OILMM implemented in the manner described above. It would of course be unreasonable to assert that the OILMM dominates the Kronecker trick; rather, it seems appropriate to assert that they are competitive with each other in the regime considered.

This is perhaps surprising as the performance of the Kronecker trick is determined almost entirely by a couple of computationally intensive operations, the eigendecomposition and matrix-matrix multiplies. Carefully optimised implementations of these operations exist, and were used, to implement the Kronecker trick. Conversely, the OILMM implementation discussed above comprises many small operations. While our implementation benefits from \eg~the \code{StaticArrays.jl} library, which is suitable for operations on small matrices and vectors, it remains surprising that similar performance was found.
 
In general we anticipate the OILMM implemented in the described manner be significantly faster on data sets where $n$ is much larger than $p$, whilst the Kronecker trick will likely do better when $n$ is similar to $p$.

\section{Temperature Extrapolation Experiment (Sec.\ \ref{exp:temp_extrapolation}) Additional Results}
\label{app:temp}

\cref{fig:temp_graph} depicts the RMSE and PPLP achieved in the temperature extrapolation experiment (\cref{exp:temp_extrapolation}).

\begin{figure}[t] \small
    \centering
    \includegraphics[width=.45\linewidth]{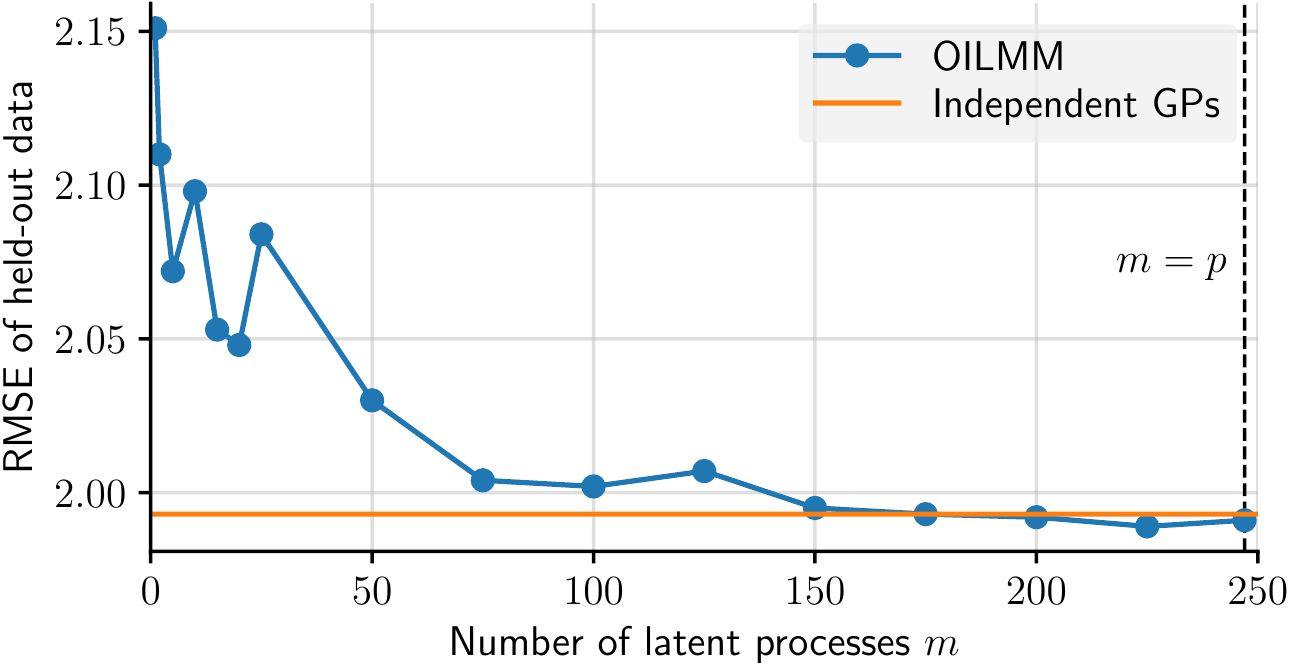}
    \includegraphics[width=.45\linewidth]{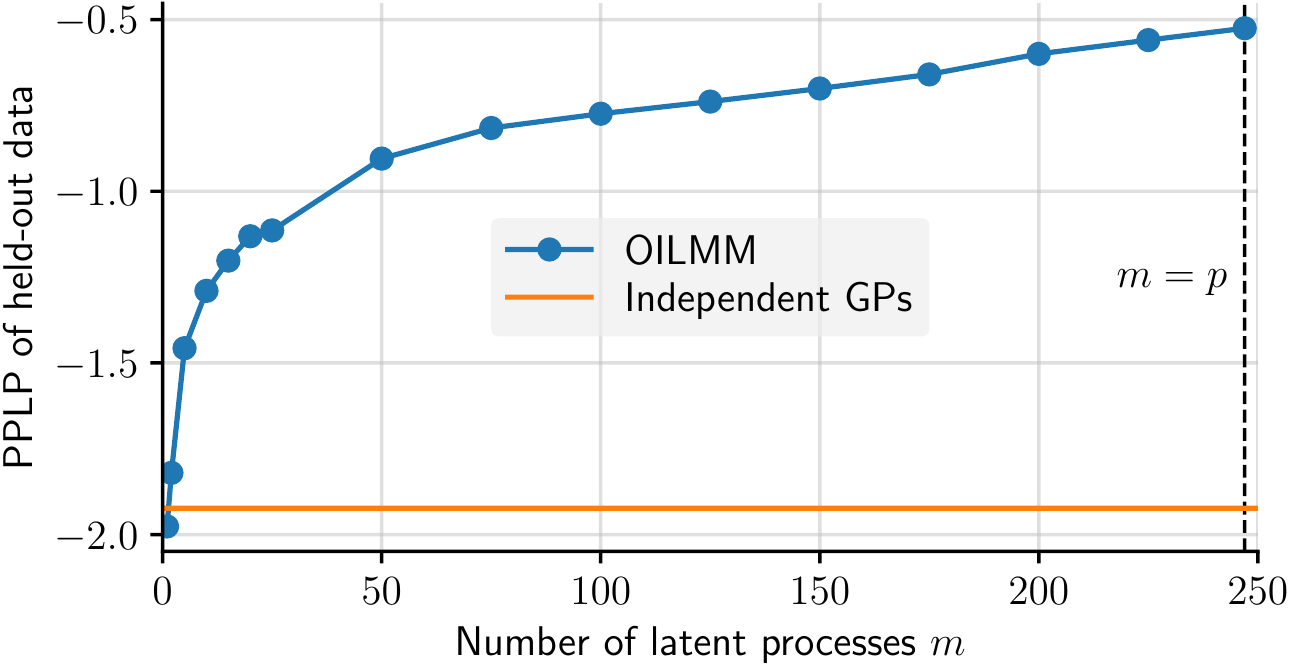}
    \vspace{-0.5em}
    \caption{RMSE and PPLP achieved in the temperature extrapolation experiment.}
    \label{fig:temp_graph}
\end{figure}

\section{Large-Scale Climate Model Calibration Experiment (Sec.\ \ref{exp:climate}) Additional Details and Analysis}
\label{app:climate}

We use the variational inducing point method by \textcite{Titsias:2009:Variational_Learning}, where the positions of the inducing points are initialised to one every two months.
All hyperparameters and the locations of the inducing points are optimised until convergence using \texttt{scipy}'s implementation of the L-BFGS-B algorithm \citep{Nocedal:2006:Numerical_Optimisation}, which takes about 4 hours on a MacBook Pro (2.7 GHz Intel Core i7 processor and 16 GB RAM).
The learned length scales were $23.3^\circ$ for latitude and $43.6^\circ$ for longitude.

\cref{fig:correlations} shows the empirical correlations and the correlations learned by the OILMM (derived from $K_s$).
In order to get insight into the learned correlations, we hierarchically cluster the models using farthest point linkage with $1-\abs{\text{corr.}}$ as the distance.
\cref{fig:dendrogram} shows the resulting dendrogram, in which models are grouped by their similarity.
For two models, the further to the right the branch connecting them is, the less similar the models are.

In \cref{fig:correlations,fig:dendrogram}, HadGEM2 is clearly singled out:
it is one of the simplest models, not including several processes that can be found in others, such as ocean \& sea-ice, terrestrial carbon cycle, stratosphere, and ocean biogeochemistry \parencite{bellouin2011hadgem2}.
Furthermore, if we inspect the names of the simulators in the groups in \cref{fig:dendrogram}, we observe that often simulators of the same family are grouped together.
We observe some interesting cases:
\begin{enumerate}[topsep=1pt,itemsep=2pt]
    \item[\em (i)]
        Although IPSL-CM5A-LR and IPSL-CM5A-MR are close, IPSL-CM5B-LR is grouped far apart.
        It turns out that IPSL-CM5A-LR and IPSL-CM5A-MR are different-resolution versions of the same model, while IPSL-CM5B-LR employs a different atmospheric model.\footnote{
            See \url{https://portal.enes.org/models/earthsystem-models/ipsl/ipslesm}.
        }
    \item[\em (ii)]
        ACCESS1.0 and ACCESS1.3 have a similar name, but differ greatly in their implementation: ACCESS1.0 is the basic model, while ACCESS1.3 is much more aspirational, including experimental atmospheric physics models and a particular land surface model \parencite{bi2013access}.
    \item[\em (iii)]
        The distance between BCC\_CSM1.1(m) and BCC\_CSM1.1 can be explained by the more realistic surface air temperature predictions obtained by the former \parencite{wu2014overview}, which is exactly the quantity we study.
\end{enumerate}

Finally, \cref{fig:latents} shows predictions for four latent processes ($i_s=1,2$ with $i_r=1,2$).
The first spatial eigenvector ($i_r=1$) is constant in space; combined with the strongest eigenvector of $K_s$ ($i_s=1$), we obtain a strong signal constituting seasonal temperature changes.

\end{document}